\documentclass[twoside,11pt]{article}
\usepackage{jmlr2e}

\usepackage[utf8]{inputenc} 
\usepackage[T1]{fontenc}    
\usepackage{graphicx}
\usepackage{subfigure}
\usepackage{color,xcolor}
\usepackage{amsmath,amssymb}
\usepackage{bm}
\usepackage{bbm}
\usepackage{mathrsfs}
\usepackage{mdframed}
\usepackage{hyperref}
\usepackage{paralist}
\usepackage{booktabs}
\usepackage{algorithmic,algorithm}
\usepackage{tikz} 
\usepackage{comment}
\usepackage{breqn} 
\usepackage[inline]{enumitem}
\usepackage{multicol}
\usepackage{mathtools}

\allowdisplaybreaks 
 
\hypersetup{ 
  colorlinks,
  linkcolor=blue!50!black, 
  citecolor=blue!50!black,
  urlcolor=blue!50!black
}

\newcommand{\krik}[1]{\noindent{\bf{\color{red}KRIK: #1}}}

\newcommand{\moto}[1]{\noindent{\bf{\color{blue}MOTO:  #1}}}
\newcommand{\puene}[1]{\noindent{\bf{\color{magenta}PUENE:  #1}}}

\newtheorem{assumption}{Assumption} 
   
\newcommand{\ie}{\emph{i.e.}} 
\newcommand{\eg}{\emph{e.g.}}

\newcommand{\inspace}{\ensuremath{\mathcal{X}}}   

\newcommand{\outspace}{\ensuremath{\mathcal{Y}}}  
\newcommand{\pp}[1]{\ensuremath{\mathbb{#1}}}     
\newcommand{\pspace}{\ensuremath{\mathscr{P}}}    

\newcommand{\hbspace}{\ensuremath{\mathscr{H}}}   

\newcommand{\hbspf}{\ensuremath{\mathscr{F}}}

\newcommand{\muv}{\ensuremath{\mu}}
\newcommand{\muh}{\ensuremath{\hat{\muv}}}
\newcommand{\rr}{\mathbb{R}} 		         
\newcommand{\ep}{\mathbb{E}}                     
\newcommand{\kmat}{\mathbf{K}}                   
\newcommand{\lmat}{\mathbf{L}}                   
\newcommand{\bvec}{\bm{\beta}}                   
\newcommand{\id}{\mathbf{I}}
\newcommand{\covx}{\ensuremath{\mathcal{C}_{\mathit{XX}}}}
 
\newcommand{\covxy}{\ensuremath{\mathcal{C}_{\mathit{XY}}}}
\newcommand{\covyx}{\ensuremath{\mathcal{C}_{\mathit{YX}}}}
\newcommand{\ecovx}{\ensuremath{\widehat{\mathcal{C}}_{\mathit{XX}}}}
\newcommand{\ecovy}{\ensuremath{\widehat{\mathcal{C}}_{\mathit{YY}}}}
\newcommand{\ecovxy}{\ensuremath{\widehat{\mathcal{C}}_{\mathit{XY}}}}
\newcommand{\ecovyx}{\ensuremath{\widehat{\mathcal{C}}_{\mathit{YX}}}}
\newcommand{\dd}{\, \mathrm{d}}

\newcommand{\ecoryx}{\ensuremath{\widehat{\mathcal{W}}_{\mathit{YX}}}}

\newcommand{\x}{\ensuremath{\mathbf{x}}}
\newcommand{\s}{\ensuremath{\mathbf{s}}}
\newcommand{\y}{\ensuremath{\mathbf{y}}}

\newcommand{\ve}{\varepsilon}

\newcommand{\bfa}{\ensuremath{\mathbf{a}}}

\newcommand{\Pz}{\pp{P}_{X_0}}
\newcommand{\Po}{\pp{P}_{X_1}}

\def\ci{\perp\!\!\!\perp}

\usepackage{lastpage}
\jmlrheading{22}{2021}{1-\pageref{LastPage}}{2/20}{6/21}{20-185}{Krikamol Muandet, Motonobu Kanagawa, Sorawit Saengkyongam, and Sanparith Marukatat}
\ShortHeadings{Counterfactual Mean Embeddings}{Muandet, Kanagawa, Saengkyongam, and Marukatat}

\firstpageno{1}

\begin{document}

\title{Counterfactual Mean Embeddings}

\author{\name Krikamol Muandet\thanks{A part of this work was done when KM was affiliated with the Department of Mathematics, Mahidol University, Thailand.} \email krikamol@tuebingen.mpg.de \\ 
\addr Max Planck Institute for Intelligent Systems \\
T\"ubingen, Germany
\AND 
\name Motonobu Kanagawa\thanks{A part of this work was done when MK was affiliated with the Institute of Statistical Mathematics, Japan, and with the University of T\"ubingen and  Max Planck Institute for Intelligent Systems, Germany.} \email motonobu.kanagawa@eurecom.fr \\ 
\addr Data Science Department, EURECOM \\
Sophia Antipolis, France
\AND 
\name Sorawit Saengkyongam \email ss@math.ku.dk \\ 
\addr {University of Copenhagen \\
Copenhagen, Denmark}
\AND
\name Sanparith Marukatat \email sanparith.marukatat@nectec.or.th \\
\addr National Electronics and Computer Technology Center \\ 
National Science and Technology Development Agency \\
Pathumthani, Thailand
}

\editor{}

\maketitle

\begin{abstract}%
   Counterfactual inference has become a ubiquitous tool in online advertisement, recommendation systems, medical diagnosis, and econometrics. Accurate modelling of outcome distributions associated with different interventions---known as  counterfactual distributions---is crucial for the success of these applications. In this work, we propose to model counterfactual distributions using a novel Hilbert space representation called counterfactual mean embedding (CME). The CME embeds the associated counterfactual distribution into a reproducing kernel Hilbert space (RKHS) endowed with a positive definite kernel, which allows us to perform causal inference over the entire landscape of the counterfactual distribution. Based on this representation, we propose a  distributional treatment effect (DTE) which can quantify the causal effect over entire outcome distributions. Our approach is nonparametric as the CME can be estimated under the unconfoundedness assumption from observational data without requiring any parametric assumption about the underlying distributions. We also establish a rate of convergence of the proposed estimator which depends on the smoothness of the conditional mean and the Radon-Nikodym derivative of the underlying marginal distributions. Furthermore, our framework allows for more complex outcomes such as images, sequences, and graphs. Our experimental results on synthetic data and off-policy evaluation tasks demonstrate the advantages of the proposed estimator.
\end{abstract}


\section{Introduction}

To make a rational decision, a decision maker must be able to anticipate the effects of a decision to the outcomes of interest, before committing to that decision.
For instance, before building a certain facility in a city, \eg, a dam, policymakers and citizens must seek to understand its environmental effects.
In medicine, a doctor has some prior knowledge about the effects a certain drug will have on a patient's health, before actually prescribing it.  
In business, a company needs to understand the effects of a certain strategy of advertisement to its revenue.
One approach to addressing these questions is {\em counterfactual inference}.

Counterfactual inference we consider in this work consists of the following three main ingredients. 
Suppose that there exists a hypothetical subject (\eg, a patient in medical treatment), and let $X$ be {\em covariates} representing the features of the subject (\eg, age, weight, medical record, etc.), $T$ be a {\em treatment indicator} representing the treatment assigned to the subject (a drug of interest or a placebo), and $Y$ be the {\em observed outcome} representing the post-treatment quantity of interest (\eg, whether the patient is recovered or not). 
Given certain realizations of these variables $\{(\x_i,t_i,\y_i)\}_{i=1}^n$, in which each index $i$ represents the identity of a subject,  an analyst wishes to know how the treatment affects the outcome.


This problem is called counterfactual since for each subject $i$, we only observe the outcome $\y_i$ resulting from the assigned treatment $t_i$ and can never observe the outcome (say $\y'_i$) that would have been realized under an alternative treatment $t_i' \neq t_i$. For example, if a patient receives an active treatment (\eg, a drug of interest), we can never observe the outcome from the same patient under a control treatment (\eg, a placebo).  This is known as the fundamental problem of causal inference \citep{Holland86:FPCI} and also as {\em bandit feedback} in the bandit literature \citep{Dudik11:DoublyRobust}.
One way to partially address this issue is a randomized experiment \citep{Fisher35:Rand}, in which treatments are randomly assigned to subjects. 
Although considered a gold standard,  in practice randomization  can be too expensive, time-consuming, or unethical. In most cases, therefore, analysis about treatment effects needs to be done on the basis of observational data $\{(\x_i,t_i,\y_i)\}_{i=1}^n$ in which the treatment assignment $t_i$ may depend on covariates $\x_i$ and possibly on some hidden confounders; this setting is commonly known as {\em observational studies} \citep{Rosenbaum02:OS,Rubin05:PO}.


A fundamental framework for observational studies is the {\em potential outcome framework} \citep{Neyman1923:Causal,Rubin74:Causal}. 
It provides a clear notation for {\em potential outcomes}, \ie, the outcomes that would have been observed under different treatments, and elucidates the conditions required for making a valid inference about treatment effects; see Section \ref{sec:potential-outcome}.
The framework has been studied extensively in statistics, and has a wide range of applications in biomedical and social sciences; see, \eg, \citet{Imbens15:Imbens}. 
Moreover, important applications of machine learning such as off-policy evaluation for online advertisement and recommendation systems can be reformulated under this framework \citep{Schnabel16:RecomTreat,KallusZ18:Continuous}. 
We argue, however, that there exist the following challenges:

\paragraph{Average treatment effects.}
Many of existing works focus on estimating the {\em average treatment effect} (ATE), which is the difference between the means of the outcome distributions; see Section \ref{sec:potential-outcome} for details.
However, the ATE does not inform changes in higher-order moments, even when they exist.
For instance, if a treatment of interest has an effect only in the {\em variance} of the distribution of outcomes, then the analysis of average treatment effects cannot capture such effects. 
Suppose that the treatment is whether to provide a certain drug, and the outcome is the blood pressure of a patient; just analyzing the average treatment effects may lead to an incorrect conclusion, if the drug increases/decreases the blood pressure of a patient whose blood pressure was already high/low.
This highlights the importance of analyzing the outcome distribution as a whole.
 
In this work, we focus on the \emph{distributional treatment effect} (DTE), which involves the entire outcome distributions.
This scenario often arises in several real-world socioeconomic applications; see, \eg, \citet{Rot10,Chernozhukov13:Counterfactual}.

\paragraph{Parametric models.}
Many of the classical approaches in causal inference make parametric assumptions about relationships between covariates $X$, treatment assignment $T$, and observed (or potential) outcomes.
However, if the imposed parametric assumption is incorrect, \ie, model misspecification, then the conclusion about treatment effects can be wrong or misleading.
To overcome this limitation, there is a recent surge in applying nonparametric machine learning models to causal inference problems, e.g., \citet{Shalit16:CounterfactualBound} and \citet{Alaa17:MTGP} among others.
This paper also contributes to this endeavour.

\paragraph{Overparameterized models.}
Deep learning has become the first choice in many applied fields due to its excellent empirical performance, and thus has also been applied to counterfactual inference, \eg, \citet{Johansson16:CounterfactualRep,Hartford17a:DeepIV}.
Unfortunately, such approaches based on deep learning lack theoretical guarantees, because arguably deep learning itself lacks an established theory as a learning method (at least until now).
This is problematic when consequential decisions are based on the analysis of treatment effects (\eg, political decisions and medical treatments).
Having better theoretical grounding, kernel methods have recently become popular tools for causal inference \citep{Alaa17:MTGP,Singh19:KIV,Muandet20:KCM,Muandet20:DualIV}.

\paragraph{Multivariate and structured outputs.}
Existing works often deal with outcomes that are discrete or real-valued.
However, depending on the application, outcome variables may be multivariate (possibly high-dimensional) or structured, such as images and graphs.
For example, in medical data analysis, outcomes may be fMRI data taken from a subject after receiving a certain treatment.
Thus, it is not straightforward to apply  existing approaches.

\paragraph{}
In this work, we propose a novel approach to counterfactual inference that addresses the above challenges, which we term {\em counterfactual mean embedding} (CME). 
Our approach is built on kernel mean embedding \citep{Berlinet04:RKHS,Smola07Hilbert,Muandet17:KME}, a framework for representing probability distributions as elements in a reproducing kernel Hilbert space (RKHS), so that each element representing a distribution maintains all of its information (cf. Section \ref{sec:background} and \ref{kme-cme}).
We define an element representing a counterfactual distribution, for which we propose a nonparametric estimator.
Notable advantages of the proposed approach are summarized as follows:
\begin{enumerate}
    \item 
    The proposed estimator can be computed based only on linear algebraic operations involving kernel matrices.
    Being a kernel method, it can be applied to not only standard domains (such as the Euclidean space), but also more complex and structured covariates and/or outcomes such as images, sequences, and graphs, by using  off-the-shelf kernels designed for such data \citep{Gartner03:Structured}; this widens possible applications of counterfactual inference in general (cf. Section \ref{sec:CME-estimators}).
    Thus our work offers more flexibility than the existing approaches by \citet{Rot10} and \citet{Chernozhukov13:Counterfactual}, who focused on estimating the cumulative distribution functions of counterfactual distributions by assuming real-valued outcomes. 
    \item 
    The proposed estimator can be used for computing a distance between the counterfactual and controlled distributions, thereby providing a way of quantifying the effect of a treatment to the distribution of outcomes; we define this distance as the maximum mean discrepancy (MMD) \citep{Borgwardt06:MMD,Gretton12:KTT} between the  counterfactual and controlled distributions.
    It also provides a way to sample points from a counterfactual distribution based on kernel herding \citep{CheWelSmo10}, a kernel-based deterministic sampling method (cf. Section \ref{sec:kte}).

    \item 
    The proposed estimator is nonparametric, and has theoretical guarantees.
    Specifically, we prove the consistency of the proposed estimator under a very mild condition (cf. Theorem \ref{theo:uinf_conv}), and derive its convergence rates under certain regularity assumptions involving kernels and underlying distributions (cf. Theorem \ref{theo:convergence-rate}).
    Both results hold without assuming any parametric assumption. 
   
\end{enumerate}


The rest of the paper is organized as follows.
After summarizing related work in Section \ref{sec:related_work},
we review in Section \ref{sec:preliminaries} the potential outcome framework as well as kernel mean embedding of distributions. 
Section \ref{sec:cme} introduces counterfactual learning and then provides a generalization of Hilbert space embedding to counterfactual distributions. 
This section also presents how we can quantify and estimate distributional treatment effects (DTEs) with our approach.
We subsequently provide the detailed convergence analysis in Section \ref{sec:theory-main}, followed by examples of the important applications in Section \ref{sec:applications} (sampling and testing) and Section \ref{sec:policy-evaluation} (off-policy evaluation).
Finally, we demonstrate the effectiveness of the proposed estimator on simulated data as well as real-world policy evaluation tasks in Section \ref{sec:experiments}.

\subsection{Related Work} 
\label{sec:related_work}

We summarize below related works on counterfactual inference.


\paragraph{Treatment effect estimation.}

Estimating treatment effects is one of the most fundamental tasks in counterfactual inference \citep{Rubin74:Causal,Shalit17:ITE}.
This task is hindered by the fact that one cannot observe all potential outcomes at the same time for each subject.
Moreover, the data is usually biased by a non-randomized treatment assignment. 
Modern approaches attempt to resolve these problems by usingo state-of-the-art ML algorithms. For example, \citet{Hill11:Nonparam} develops a nonparametric method for estimating the ITE based on Bayesian additive regression tree (BART). \citet{Athey16:Recursive} and \citet{Wager17:RandomForests} adapt tree-based methods to treatment effect estimation.
\citet{Shalit16:CounterfactualBound} and \citet{Johansson16:CounterfactualRep} formulate the problem as a domain adaptation problem and propose to balance the covatiates using representation learning. \citet{Hartford17a:DeepIV} develop a two-step regression method based on deep neural networks for instrumental variable regression. 
Adversarial training of neural networks for causal inference have also been considered in \citet{Yoon18:ganite}, for example.




\paragraph{Off-policy evaluation and learning from observational data.}

In many circumstances, evaluating and learning a policy by interacting directly with an environment may not be possible due to practical constraints (\eg, monetary costs, safety and ethics). 
As a result, several works have attempted to leverage historical data collected using a logging policy in off-policy evaluation and learning, \eg, \citet{Langford08:ES,Atan18;DeepTreat}. 
Most methods rely on importance weighting \citep{Langford08:ES,Bottou13:Counterfactual,Swaminathan15:CRM}.
\citet{Dudik11:DoublyRobust} uses a doubly robust estimator to reduce the variance of off-policy evaluation. 
\citet{Swaminathan15:CRM} presents a framework for policy learning called counterfactual risk minimization (CRM) based on empirical variance regularization. 
In this work, we also demonstrate the application of our estimator in off-policy evaluation. 

\paragraph{Causal inference with kernel mean embeddings.}

Hilbert space embedding of distributions has been applied extensively in causal inference. 
For instance, in causal discovery, \citet{Fukumizu2008,Zhang2011,Doran14:Perm} develop powerful kernel-based tests of conditional independence which allow for the recovery of causal graphs up to the Markov equivalence class. 
See \citet[Section 4.8]{Muandet17:KME} for a review of many other applications.
In treatment effect estimation, kernel methods have become a popular approach to covariate balancing between treatment and control groups \citep{Shalit16:CounterfactualBound,Johansson16:CounterfactualRep,Wong17:CovBal,Kallus17:Matching}.
Our work, on the contrary, focuses on characterizing the representation of counterfactual distribution of outcomes using the kernel mean embedding and provides nonparametric inference tools.

\section{Preliminaries}
\label{sec:preliminaries}

The counterfactual mean embedding relies on the potential outcome framework as well as the concepts of kernels, reproducing kernel Hilbert spaces (RKHSs), and kernel mean embedding of distributions.
We review these concepts in this section.


\subsection{Kernels and Reproducing Kernel Hilbert Spaces (RKHSs)}

We first review kernels and RKHSs, details of which can be found in, \eg, \citet{Scholkopf01:LKS}, \citet{Berlinet04:RKHS}, and \citet{Smola07Hilbert}.




Let $\inspace$ be a nonempty set. 
Let $\hbspace$ be a Hilbert space consisting of functions on $\inspace$ with $\left< \cdot,\cdot\right>_\hbspace$ and $\| \cdot \|_\hbspace$ being its inner-product and norm, respectively.
The Hilbert space $\hbspace$ is called a \emph{reproducing kernel Hilbert space} (RKHS), if there exists a symmetric function $k: \inspace \times \inspace \to \mathbb{R}$, called the {\em reproducing kernel} of $\hbspace$, satisfying the following properties: 
\begin{enumerate}
    \item For all $\x\in\inspace$, we have $k(\cdot, \x) \in \hbspace$. Here $k(\cdot, \x)$ is the function of the first argument with $\x$ being fixed, such that $\x' \mapsto k(\x',\x)$.
    \item For all $f\in\hbspace$ and $\x\in\inspace$, we have $f(\x) = \langle k(\cdot, \x), f\rangle_{\hbspace}$. This is called the {\em reproducing property} of $\hbspace$ (or of $k$).
\end{enumerate}
It is known that the linear span of functions $k(\cdot,\x)$, denoted by ${\rm span}(k(\cdot, \x) \mid \x \in \inspace)$, is dense in $\hbspace$, \ie,
$$
\hbspace = \overline{{\rm span}(k(\cdot, \x) \mid \x \in \inspace)},
$$
where the closure on the right hand side is taken with respect to the norm of $\hbspace$.
In other words, any $f \in \hbspace$ can be written as $f = \sum_{i=1}^\infty \alpha_i k(\cdot,\x_i)$ for some $(\alpha_i)_{i=1}^\infty \subset \mathbb{R}$ and $(\x_i)_{i=1}^\infty \subset \inspace$ such that $\| \sum_{i=1}^\infty \alpha_i k(\cdot,\x_i) \|_\hbspace^2 = \sum_{i,j=1}^\infty \alpha_i \alpha_j k(\x_i,\x_j) < \infty$.

Any RKHS is uniquely associated with its reproducing kernel $k$, which is {\em positive definite}: a symmetric function $k: \inspace \times \inspace \to \mathbb{R}$ is called positive definite, if for all $n\in\mathbb{N}$, $\alpha_1,\ldots,\alpha_n \in\mathbb{R}$, and all $\x_1,\ldots,\x_n\in\inspace$, we have $\sum_{i=1}^n\sum_{j=1}^n\alpha_i\alpha_j k(\x_i,\x_j) \geq 0$.
On the other hand, for {\em any} positive definite kernel $k: \inspace \times \inspace \to \mathbb{R}$, there exists an RKHS $\hbspace$ for which $k$ is the reproducing kernel \citep{aronszajn50reproducing}.
Therefore, by defining a positive definite kernel, one always implicitly defines its RKHS.

As indicated from the definition of positive definiteness, kernels can be defined on {\em any} nonempty set $\inspace$. 
Therefore, they have been defined not only for the real vector space $\mathbb{R}^d$, but also for non-standard domains such as those of images and graphs.
Popular kernels on $\inspace \subset \mathbb{R}^d$ include linear kernels $k(\x,\x')=\x^\top\x'$, polynomial kernels $k(\x,\x') = (\x^\top\x' + c)^p, c > 0, p\in\mathbb{N}_{+}$, Gaussian kernels $k(\x,\x')=\exp(-\|\x-\x'\|_2^2/2\sigma^2), \sigma >0$, and Laplace (or more generally Mat\'ern) kernels $k(\x,\x')=\exp(-\|\x-\x'\|_2/2\sigma^2), \sigma > 0$.  
More examples of positive definite kernels can be found in \citet{Genton02:Kernel} and \citet{HofSchSmo08}.

\subsection{Kernel Mean Embedding of Distributions}
\label{sec:background}
In this work, we use kernels and RKHSs to represent, compare, and estimate {\em probability distributions}.
This is enabled by the approach known as \emph{kernel mean embedding} of distributions \citep{Berlinet04:RKHS,Smola07Hilbert,Muandet17:KME}, which we review here.
In what follows, we assume that $\inspace$ is a measurable space with some sigma algebra $\mathcal{B}_\inspace$.

\begin{definition}[Kernel mean embedding (KME)]
     Let $\pspace$ be the set of all probability measures on a measurable space $(\inspace, \mathcal{B}_\inspace)$ and $k:\inspace\times\inspace\to\mathbb{R}$ be a measurable positive definite kernel with associated RKHS  $\hbspace$, such that $\sup_{\x\in\inspace} k(\x,\x) < \infty$. 
     Then, the kernel mean embedding (KME) of $\pp{P}\in\pspace$ is defined as the Bochner integral\footnote{See, \eg, \citet[Chapter 2]{Diestel-77} and \citet[Chapter 1]{Dinculeanu:2000} for the definition of Bochner integral.} of $k(\cdot, \x)$ with respect to $\pp{P}$:
    \begin{equation}\label{eq:kme}
        \mu\,:\, \pspace\rightarrow\hbspace, \quad \pp{P}\mapsto \muv_{\pp{P}}:= \int k(\cdot, \x)\dd\pp{P}(\x).
    \end{equation}
     The element $\muv_{\pp{P}}$ may be alternatively called the kernel mean of $\pp{P}$.
     For a random variable $X \sim \pp{P}$, the kernel mean may also be written as $\muv_{X}$.
\end{definition}

The kernel mean $\mu_{\pp{P}}$ serves as a representation of $\pp{P} \in \pspace$ in the RKHS $\hbspace$.
This is justified if $\hbspace$ is \emph{characteristic} \citep{Fukumizu04:DRS}: the RKHS $\hbspace$ (and the associated kernel $k$) is defined to be characteristic, if the mapping $\mu\,:\, \pspace\rightarrow\hbspace$ in \eqref{eq:kme} is injective.
In other words, $\hbspace$ is characteristic, if for any $\pp{P}, \pp{Q} \in \pspace$, we have $\muv_{\pp{P}} = \muv_{\pp{Q}}$ {\em if and only if}  $\pp{P} = \pp{Q}$.
That is, $\mu_{\pp{P}}$ is uniquely associated with $\pp{P} \in \pspace$, and thus $\mu_{\pp{P}}$ becomes a unique representation of $\pp{P}$ in $\hbspace$, maintaining all information about $\pp{P}$. 
Examples of characteristic kernels on $\inspace = \mathbb{R}^d$ include Gaussian, Mat\'ern and Laplace kernels \citep{Sriperumbudur10:Metrics}. 
On the other hand, linear and polynomial kernels are not characteristic, since their RKHSs are finite dimensional and only provide unique representations of distributions up to certain moments.

The kernel mean embedding \eqref{eq:kme} is the key ingredient of a well-known metric on probability measures called maximum mean discrepancy (MMD) \citep{Borgwardt06:MMD,Gretton12:KTT}. 
For two distributions $\pp{P}, \pp{Q} \in \pspace$, their MMD is given as the RKHS distance between the corresponding kernel means $\muv_{\pp{P}}$, $\muv_{\pp{Q}}$:
\begin{equation}\label{eq:mmd}
    \text{MMD}[\hbspace,\pp{P},\pp{Q}] := \|\muv_{\pp{P}} - \muv_{\pp{Q}}\|_{\hbspace} 
    = \sup_{f\in\hbspace,\|f\|_\hbspace \leq 1}\left| \int f(\x) \dd\pp{P}(\x) - \int f(\x) \dd\pp{Q}(\x) \right|,
\end{equation}
where the second identity follows from the reproducing property and $\hbspace$ being a vector space \cite[Lemma 4]{Gretton12:KTT}.
The right expression is the maximum discrepancy between the means of functions from the unit ball of the RKHS $\hbspace$, and is the original definition of MMD.
Being defined via the RKHS distance, MMD is a pseudo-metric on $\pspace$.
Moreover, if $\hbspace$ is characteristic,  $\text{MMD}[\hbspace,\pp{P},\pp{Q}] = 0$ holds if and only if $\pp{P} = \pp{Q}$, and thus MMD becomes a proper metric on probability measures.
See \citet{Sriperumbudur10:Metrics,SimSch18} for details and relationships to other popular metrics on probability measures.

Given an i.i.d.~(identically and independently distributed) sample $\x_1,\ldots,\x_n$ from $\pp{P}$, the kernel mean $\muv_{\pp{P}}$ can be estimated simply by the empirical average
\begin{equation}\label{eq:emp-kme}
    \muh_{\pp{P}} := \frac{1}{n}\sum_{i=1}^n k(\cdot, \x_i).
\end{equation}
The $\sqrt{n}$-consistency of \eqref{eq:emp-kme}, that is $\| \muv_{\pp{P}} - \muh_{\pp{P}} \|_\hbspace = O_p(n^{-1/2})$ as $n \to \infty$, has been established in \citet[Theorem 27]{Song08:Thesis} and also in \citet{Gretton12:KTT,Lopez-Paz15:Towards,TolstikhinSM17:Minimax}.
Importantly, this holds without any parametric assumption about the underlying distribution $\pp{P}$. 

Given another i.i.d.~sample $\x_1',\dots,\x_m'$ from $\pp{Q}$, and defining  $\muh_{\pp{Q}} := \frac{1}{m}\sum_{j=1}^m k(\cdot, \x_j')$ as an estimate of the kernel mean $\muv_{\pp{Q}}$, the (squared) MMD \eqref{eq:mmd} can be estimated as
\begin{align*}
\widehat{\text{MMD}}^2[\hbspace,\pp{P},\pp{Q}] &= \| \muh_{\pp{P}} - \muh_{\pp{Q}} \|^2_{\hbspace} \nonumber \\
&= \frac{1}{n^2} \sum_{i=1}^n \sum_{j=1}^n k(\x_i,\x_j) - \frac{2}{nm} \sum_{i=1}^n \sum_{j=1}^m k(\x_i,\x_j') + \frac{1}{m^2} \sum_{i=1}^m \sum_{j=1}^m k(\x_i',\x_j'), 
\end{align*}
where the right expression follows from the reproducing property \citep[Eq.~5]{Gretton12:KTT}.
Applying the triangle inequality, it follows that 
$
\left| \|\muv_{\pp{P}} - \muv_{\pp{Q}}\|_{\hbspace} - \| \muh_{\pp{P}} - \muh_{\pp{Q}} \|_\hbspace \right| \leq \| \muv_{\pp{P}} - \muh_{\pp{P}} \|_\hbspace + \|  \muv_{\pp{Q}} - \muh_{\pp{Q}} \|_\hbspace  = O_p(n^{-1/2}) + O_p(m^{-1/2})
$
as $n, m \to \infty$, implying the consistency of the above estimator of MMD with a parametric convergence rate.
This estimator only requires evaluations of the kernel, and therefore is easy to implement in practice.
We note that the above MMD estimator is biased, while being consistent; an unbiased estimator is also available for MMD \citep[Eq.~3]{Gretton12:KTT}.




\subsection{Kernel Mean Embedding of Conditional Distributions}
\label{kme-cme}

Finally, the notion of KME can be extended to conditional distributions \citep{Song10:KCOND,Grunewalder12:LGBPP,Song2013,Fukumizu13:KBR}. 
To describe this, let $(X,Y)$ be a random variable taking values in the product space $\inspace\times\outspace$, where $\inspace$ and $\outspace$ are measurable spaces. 
We define a measurable kernel $k$ on $\inspace$ and let $\hbspace$ be the associated RKHS.
Similarly, we define a measurable kernel $\ell$ on $\outspace$ and let $\hbspf$ be the associated RKHS.
Let $\pp{P}_{XY}$ be the joint distribution of $(X,Y)$, and  $\pp{P}_{Y|X=\x}$ be the conditional distribution of $Y$ given $X = \x$.

The KME of the conditional distribution $\pp{P}_{Y|X=\x}$ is then defined as the conditional expectation of $\ell(\cdot,\y)$ with respect to $\pp{P}_{Y|X=\x}$: 
\begin{equation} \label{eq:cond-kme}
    \muv_{Y|X=\x} := \int \ell (\cdot,\y) \dd\pp{P}_{Y|X = \x}  (\y) \in \hbspf \quad (\x \in \inspace).
\end{equation}
Again, if $\hbspf$ is characteristic, this kernel mean maintains all information about $\pp{P}_{Y|X=\x}$, thus being qualified as its representation.
It is instructive to note that $\muv_{Y|X=\x}$ is defined for each $\x \in \inspace$ individually.

Given an i.i.d.~sample $(\x_1,\y_1),\dots,(\x_n,\y_n)$ from the joint distribution $\pp{P}_{XY}$, the conditional mean embedding \eqref{eq:cond-kme} can be estimated as 
\begin{equation}
\hat{\muv}_{Y|X=\x} := \sum_{i=1}^n w_i(\x) \ell(\cdot,\y_i), \label{eq:cond-kme-est} \\
\end{equation}
\noindent where
\begin{eqnarray*}
 (w_1(\x),\dots,w_n(\x))^\top &:=& (\kmat + n \varepsilon \id)^{-1} {\bf k}(\x) \in \mathbb{R}^n, \\
 {\bf k}(\x) &:=& (k(\x,\x_1),\dots,k(\x,\x_n))^\top \in \mathbb{R}^n.
\end{eqnarray*}
Here, $\kmat \in \mathbb{R}^{n \times n}$ is the kernel matrix such that $\kmat_{i,j} = k(\x_i,\x_j)$, and $\varepsilon > 0$ is a regularization constant.
As pointed out by \citet{Grunewalder12:LGBPP}, this estimator can be interpreted as that of {\em function-valued kernel ridge regression}, where the task is to estimate the mapping $\x \mapsto \int \ell (\cdot,\y) \dd\pp{P}_{Y|X = \x}  (\y)$ from training data $(\x_1,\ell(\cdot,\y_1)), \dots, (\x_n,\ell(\cdot,\y_n)) \in \inspace \times \hbspf$.
In fact, the weights $w_1(\x),\dots,w_n(\x)$ in \eqref{eq:cond-kme-est} are identical to those of kernel ridge regression (or Gaussian process regression).
As such, the regularization constant $\varepsilon$ should decay to $0$ at an appropriate speed as $n \to \infty$, in order to ensure a good convergence rate of the estimator \eqref{eq:cond-kme-est}, see, \eg, \citet{CapDev07}. 

\section{Counterfactual Mean Embeddings}
\label{sec:cme}

In this section, we formulate our problem of estimating distributional treatment effects and describe our approach.
In Section \ref{sec:potential-outcome}, we review the potential outcome framework and, based on it, we define distributional treatment effects.
The key concepts here are counterfactual distributions on outcomes.
In Section \ref{sec:kmean-counterfactual-dist}, we describe our approach, \emph{counterfactual mean embeddings}, as the kernel mean embeddings of counterfactual distributions.
Section \ref{sec:more-on-assignment} provides details of the distributional effects of covariate distributions, which are essential for applications in off-policy evaluation. 
We then define their empirical estimators in Section \ref{sec:CME-estimators}.
Finally, we introduce the kernel treatment effect (KTE) as a way to evaluate the distributional treatment effect in Section \ref{sec:kte}.

\subsection{Potential Outcome Framework and Distributional Causal Effects}
\label{sec:potential-outcome}

We pose our problem based on the potential outcome framework, also known as the Neyman-Rubin causal model, which is a classic and widely used approach to estimating causal effects of treatments from observational data \citep{Neyman1923:Causal,Rubin74:Causal,Rubin05:PO}.     

We consider a hypothetical subject (\eg, a patient) in a population.
Let $X \in \inspace$ be a {\em covariate} random variable representing the subject's features (\eg, age, weight, blood pressure, etc.), where $\inspace$ is a measurable space.
Let $T \in \mathcal{T}$ a random variable that indicates the {\em treatment} assigned to the subject,  where $\mathcal{T}$ denotes the set of treatments of interest. We call $T$  {\em treatment indicator} or {\em treatment assignment}.  
In this work, we focus on binary treatments $\mathcal{T} := \{0,1\}$ for simplicity, but an extension to multiple treatments is straightforward. 
For instance, $T=1$ may represent that the subject is assigned an active treatment (e.g., a drug of interest), and $T = 0$ a control treatment (e.g., placebo). 

Let $Y_0^*, Y_1^* \in \mathcal{Y}$ be random variables representing {\em potential} outcomes, where $\mathcal{Y}$ is a measurable space.
That is, $Y_1^*$ represents the outcome of interest after the subject is exposed to treatment $1$, and $Y_0^*$ the outcome after the subject is exposed to treatment $0$.
For instance, $Y_1^*$ may be the blood pressure of the patient measured after the patient had the drug, and $Y^*_0$ be that after having nothing. 
The problem here, known as the fundamental problem of causal inference, is that one can only observe either $Y_1^*$ or $Y_0^*$, but not both.
For instance, if one gave the drug to the patient and measured the resulting blood pressure, it is no longer possible to measure the blood pressure of the same patient without the drug.
Thus, the {\em observed} outcome $Y \in \outspace$ can be defined as
$$
Y :=  \mathbbm{1}(T=0) Y^*_0 + \mathbbm1(T=1) Y^*_1,  
$$
where $\mathbbm{1}(T=j) := 1$ if $T=j$, and zero otherwise.
Note that in observational studies, the treatment assignment may not be completely random, \ie, $T$ depends on $Y_0^*$, $Y_1^*$ and $X$.

Assume that there are $N$ subjects, and that each subject $i = 1,\dots,N$ is associated with random variables $(\x_i, t_i, \y_{0i}^*, \y_{1i}^*)$ that are distributed as $(X, T, Y_0^*, Y_1^*)$ independently to the other subjects,\footnote{This independence assumption may be seen as a version of the Stable Unit Treatment Value Assumption (SUTVA), which requires that the potential outcomes of any subject $i$ are independent of the treatments $t_j$ assigned to the other subjects $j \neq i$.} \ie,
\begin{equation} \label{eq:sample-potential-notation}
(\x_i, t_i, \y_{0i}^*, \y_{1i}^*)_{i=1}^N  \sim (X, T, Y_0^*, Y_1^*),\quad {\rm i.i.d.} 
\end{equation}
Note that for each subject $i$, only one of $\y_{0i}^*$ or $\y_{1i}^*$ can be observed.
Thus, observational data given to the analyst are 
\begin{equation} \label{eq:observational-data}
(\x_i,t_i, \y_i )_{i=1}^N, \quad \y_i := \mathbbm{1}(t_i=0) \y^*_{0i} + \mathbbm1(t_i=1) \y^*_{1i},
\end{equation}
which are i.i.d.~with $(X,T,Y)$.
We write $n := \sum_{i=1}^N \mathbbm{1}(t_i = 0)$ the number of subjects receiving treatment $T=0$, and $m := \sum_{i=1}^N \mathbbm{1}(t_i = 1)$ that of treatment $T=1$.

We consider three kinds of distributional causal effect, as described below.
For ease of understanding, we also present the corresponding expressions based on the sample \eqref{eq:sample-potential-notation}.
Nevertheless, these sample expressions are also counterfactual quantities due to the fundamental problem of causal inference.

\subsubsection{Distributional Treatment Effect (DTE)}
\label{sec:DATE}
Let $\pp{P}_{Y_0^*}$ and $\pp{P}_{Y_1^*}$ be the distributions of the potential outcomes $Y_0^*$ and $Y_1^*$, respectively. 
Then we define the distributional treatment effect (DTE) as the difference between these two distributions:
    \begin{equation} \label{eq:dist-treatment-effect}
     \pp{P}_{Y_0^*} (\cdot) - \pp{P}_{Y_1^*} (\cdot).
    \end{equation}
    The corresponding sample expression is given by
    $$
      \frac{1}{N} \sum_{i=1}^N  \delta(\cdot - \y_{0i}^*) - \frac{1}{N} \sum_{i=1}^N  \delta(\cdot - \y_{1i}^*), 
    $$
    where $\delta$ is the Dirac distribution.
    As mentioned, this sample expression cannot be obtained from observational data \eqref{eq:observational-data}, since for each subject $i$, we only have either $\y_{0i}^*$ or $\y_{1i}^*$.
    
    The DTE \eqref{eq:dist-treatment-effect} can capture the treatment effects on the potential outcomes that may not be identified only by the average treatment effect (ATE) \citep{Imbens04:ATE}, the difference between the expectations of $Y^*_0$ and $Y^*_1$:
\begin{equation}\label{eq:ate}
    \text{ATE} := \ep[Y^*_0] - \ep[Y^*_1] 
\end{equation} 
    or its corresponding sample version
    $$
    \text{ATE}_N := \frac{1}{N} \sum_{i=1}^N \y_{0i}^* - \frac{1}{N} \sum_{i=1}^N \y_{1i}^*. 
    $$
    For instance, even when the ATE is $0$, the higher order moments of $\pp{P}_{Y_0^*}$ and  $\pp{P}_{Y_1^*}$, such as their variances, may differ.   
    The DTE can capture such a difference, while the ATE cannot.
    
\subsubsection{Distributional Treatment Effects on the Treated}
\label{sec:effects_by_treatment}
This is defined as the difference in two conditional distributions as
\begin{equation} \label{eq:cond_effect_potential_outcomes}
    \pp{P}_{Y_1^* | T } (\cdot\,|\,t) - \pp{P}_{Y_0^* | T } (\cdot\,|\,t), \quad t \in \{0,1\}.
\end{equation}
For $t=1$, this can be understood as the distributional treatment effect for the treated, and the corresponding sample expression is given by
$$
\frac{1}{m} \sum_{i=1}^N \mathbbm{1}(t_i = 1) \delta(\cdot - \y_{1i}^*) - \frac{1}{m} \sum_{i=1}^N \mathbbm{1}(t_i = 1) \delta(\cdot - \y_{0i}^*),
$$
where the second term is counterfactual. 
The details of the conditional treatment effect \eqref{eq:cond_effect_potential_outcomes} can be found, for example, in \citet[p.2214]{Chernozhukov13:Counterfactual}.

\subsubsection{Distributional Effects of the Covariate Distributions}
\label{sec:effects_by_treatment_assignment}

This is defined as the difference between the conditional distribution of $Y_0^*$ given $T = 0$ and that of $Y_0^*$ given $T = 1$:
\begin{equation} \label{eq:effect_potential_outcomes}
    \pp{P}_{Y_0^* | T } (\cdot \,|\, 0) - \pp{P}_{Y_0^* | T } (\cdot \,|\, 1)
\end{equation}
where $\pp{P}_{Y_0^* | T }$ is the conditional distribution of $Y_0^*$ given $T$.
A similar definition can be given for $Y_1^*$.
The corresponding sample expression is given by
\begin{equation}\label{eq:selection-bias-emp}
\frac{1}{n} \sum_{i=1}^N \mathbbm{1}(t_i = 0) \delta(\cdot - \y_{0i}^*) - \frac{1}{m} \sum_{i=1}^N \mathbbm{1}(t_i = 1) \delta(\cdot - \y_{0i}^*). 
\end{equation}
Note that the second term in \eqref{eq:effect_potential_outcomes} and \eqref{eq:selection-bias-emp} are counterfactual in the sense that the potential outcome  $\y_{0i}^*$ of subject $i$ with $t_i = 1$ is not observable.












The above distributional differences capture the effects caused by the difference in the characteristics (\ie, {\em covariates}) of subjects exposed to different treatments, e.g., selection bias, rather than the effects caused by the treatment itself. 
For instance, let us assume that a drug was assigned to subjects whose blood pressures were already high ($t_i = 1$) and not assigned to subjects with low blood pressures ($t_i = 0$).
Then the counterfactual blood pressures $\y_{0i}^*$ of the subjects with $t_i = 1$, which would have been observed if they had not taken the drug, would be higher than those $\y_{0i}^*$ with $t_i = 0$.

This kind of ``selection bias'' can be captured in the above distributional difference, and this helps understand how the difference in {\em observed} outcome distributions arises. 
To explain this more precisely, however, we need the notation, definitions and assumptions introduced in the next subsection. 
Thus, we defer further explanations to Section \ref{sec:more-on-assignment}.
There, we also explain that this distributional difference is useful in studying counterfactual effects of a policy defined as a specification of a covariate distribution. In fact, this is how we formulate the problem of off-policy evaluation in Section  \ref{sec:policy-evaluation}.


\subsection{Counterfactual Distributions}
\label{sec:kmean-counterfactual-dist}

To deal with distributional treatment effects discussed in the previous subsection, we need to introduce the notion of counterfactual distributions \citep{Chernozhukov13:Counterfactual}.
We first summarize the notation defined above and introduce new ones, which we follow \citet[Appendix C]{Chernozhukov13:Counterfactual}.
\begin{definition} \label{def:random_variables}
Let $Y^*_0$ and $Y^*_1$ be random variables taking values in $\mathcal{Y}$, 
and $X$ and $T$ be random variables taking values in $\mathcal{X}$ and $\mathcal{T} = \{0,1\}$, respectively.
The random variables $Y$, $Y_t$ and $X_t$ ($t = 0,1$) are defined as
\begin{eqnarray*}
Y &:=&  \mathbbm{1}(T=0) Y^*_0 + \mathbbm1(T=1) Y^*_1, \\ 
Y_t &:=& Y \mid T = t \quad (t = 0,1), \\
X_t &:=& X \mid T = t \quad (t = 0,1). \label{eq:def_cov}
\end{eqnarray*}  
\end{definition}

In Definition \ref{def:random_variables}, $Y$ is the observed outcome variable.
Thus, $Y_t$ is $Y$ given that the treatment assignment is $T = t$ ($t=0,1$). 
By the definition of $Y$, this implies that $Y_t = Y_t^* \,|\, (T =t)$, that is, $Y_t$ is the potential outcome conditional on $T = t$.
Note that, since $Y_t^*$ and $T$ may be dependent, $Y_t^* \,|\, (T=t)$ may differ from $Y_t^*$ as a random variable.  
The variable $X_t$ is the covariate variable $X$ conditional on $T= t$.
The pair of variables $(X_t,Y_t)$ can thus be seen as observed random variables conditional on the treatment assignment $T = t$.

The following is a key assumption, which is needed in general for counterfactual inference with observational data.
\begin{assumption} 
  \label{asp:main-asp}
  \begin{enumerate}
  \item[(\textbf{A1})\label{asmp:cond-exogen}] {\bf Conditional exogeneity}: $Y^*_0,Y^*_1 \ci
    T \,|\, X$ almost surely for $X$.
  \item[(\textbf{A2})\label{asmp:support}] {\bf Support condition}: 
    $\inspace_0 = \inspace_1$, where $\inspace_j$ is the support of the distribution $\pp{P}_{X_j}$ of $X_j$ for $j = 0,1$.
  \end{enumerate}
\end{assumption}


\pgfdeclarelayer{background}
\pgfdeclarelayer{foreground}
\pgfsetlayers{background,main,foreground}
\begin{figure}[t!]
    \centering
    \begin{tikzpicture}[scale=0.45]
    \node[draw, circle, very thick] at (11,0) (tt) {$T$};
    \node[draw, circle, very thick] at (4.5,0) (y1) {$Y^*_1$};
    \node[draw, circle, very thick] at (0,0) (y0) {$Y^*_0$};
    \node[draw, circle, very thick] at (5.5,-5) (xx) {$X$};
    \draw[-, very thick] (y0) to (y1);
    \draw[-, very thick] (tt) to (xx);
    \draw[-, very thick] (xx) to (y0);
    \draw[-, very thick] (xx) to (y1);

    \node[] at (2.2,1.8) {\texttt{Potential Outcomes}};
    \node[] at (5.5,-6.5) {\texttt{Covariates}};
    \node[] at (13,1.8) {\texttt{Treatment Indicator}};
    
    \begin{pgfonlayer}{background}
        \path (y0.north west)+(-1.5,2) node (a) {};
        \path (y1.south east)+(1.5,-1) node (b) {};
        \path[very thick, rounded corners, draw=black, dashed]
            (a) rectangle (b);
    \end{pgfonlayer}
    
  \end{tikzpicture} 
  \caption{A graphical representation of the conditional exogeneity assumption. An edge between two random variables indicates that they are dependent. The conditional exogeneity assumption states that given the covariates $X$, the potential outcomes $Y_0^*, Y_1^*$ and the treatment assignment $T$ are conditionally independent. The assumption does not hold if there exists an edge between $Y_0^*, Y_1^*$ and $T$, which is the case, for instance, when there exists a hidden confounder $Z$ that is dependent to both $Y_0^*, Y_1^*$ and $T$.}
    \label{fig:cond_exogeneity}
\end{figure}

The conditional exogeneity (A1), also known as the {\em unconfoundedness} or {\em ignorability}, is a common assumption in observational studies to guarantee the identifiability of causal effects from observational data \citep{Rosenbaum83:Propensity,Imbens04:ATE,Rubin05:PO}.
It requires that there is no hidden confounder, say $Z$, that affects both the treatment assignment $T$ and potential outcomes $Y_0^*, Y_1^*$. 
In other words, the covariates $X$ include all important characteristics regarding the potential outcomes.
This assumption is described further in Figure \ref{fig:cond_exogeneity}, where the graphical model represents the conditional independence structure between the random variables. 
The support condition (A2) is needed to make the counterfactual distribution (introduced in  \eqref{eq:counterfactual-dist} below) well-defined, and is also made in \citet[Eq.~2.3]{Chernozhukov13:Counterfactual}.
It is analogous to the overlap assumption required for propensity score methods \cite[\eg][Assumption 2.2]{Imbens04:ATE}.

We now define counterfactual distributions.
Let $\pp{P}_{X_0}$ and $\pp{P}_{X_1}$ be the probability distributions of $X_0$ and $X_1$, respectively.
Denote by $\pp{P}_{Y\langle 0|0 \rangle}$ and $\pp{P}_{Y\langle 1|1 \rangle}$ the corresponding marginal distributions of outcomes defined by
\begin{eqnarray*}
\pp{P}_{Y\langle 0|0 \rangle}(\y) &:=&  \int \pp{P}_{Y_0|X_0}(\y|\x) \dd\pp{P}_{X_0}(\x) = \pp{P}_{Y_0}(\y) \\
\pp{P}_{Y\langle 1|1 \rangle} (\y) &:=&  \int \pp{P}_{Y_1|X_1}(\y|\x) \dd\pp{P}_{X_1}(\x) = \pp{P}_{Y_1}(\y)
\end{eqnarray*}
where $\pp{P}_{Y_0|X_0}(\y|\x)$ is the conditional distribution of $Y_0$ given $X_0$, and $\pp{P}_{Y_1|X_1}(\y|\x)$ is that of $Y_1$ given $X_1$.
Following \citet{Chernozhukov13:Counterfactual},  \emph{counterfactual
  distributions} are then defined as
\begin{eqnarray}
  \label{eq:counterfactual-dist}
  \pp{P}_{Y\langle 0|1 \rangle}(\y) &:=& \int \pp{P}_{Y_0|X_0}(\y|\x)\dd\pp{P}_{X_1}(\x), \\
    \pp{P}_{Y\langle 1|0 \rangle}(\y) &:=& \int \pp{P}_{Y_1|X_1}(\y|\x)\dd\pp{P}_{X_0}(\x), \label{eq:counterfactual-dis-10}
\end{eqnarray} 
which are well-defined as long as the support condition in Assumption \ref{asp:main-asp} is satisfied. 

The distributions introduced above are defined in terms of the observed random variables $(X_t,Y_t)_{t = 0,1}$. 
We now see how these distributions are related to the distributions on potential outcomes  that appear in distributional causal effects \eqref{eq:cond_effect_potential_outcomes} and \eqref{eq:effect_potential_outcomes}. 
First, as summarized in the following lemma,  $\pp{P}_{Y\left<0|0\right>}$ and $\pp{P}_{Y\left<1|1\right>}$ are nothing but $\pp{P}_{Y_0^* | T } (\y|0)$ and $\pp{P}_{Y_1^* | T } (\y|1)$, respectively. 
For completeness, we include the proof in Appendix \ref{sec:proof-lemma-observable-outcome-dist}.
\begin{lemma} \label{lemma:observable_outcome_dist}
We have $\pp{P}_{Y\left<0|0\right>}(\y)  = \pp{P}_{Y_0^* | T} (\y|0)$ and  $\pp{P}_{Y\left<1|1\right>}(\y)  = \pp{P}_{Y_1^* | T} (\y|1)$. 
\end{lemma}

On the other hand, the counterfactual distributions $\pp{P}_{Y\langle 0|1 \rangle}$ and $\pp{P}_{Y\langle 1|0 \rangle}$ are respectively equal to distributions $\pp{P}_{Y_0^* | T } (\y|1)$ and $\pp{P}_{Y_1^* | T } (\y|0)$ appearing in \eqref{eq:effect_potential_outcomes} and \eqref{eq:cond_effect_potential_outcomes}, provided that Assumption \ref{asp:main-asp} holds \cite[Lemma 2.1]{Chernozhukov13:Counterfactual}; we provide a proof for completeness in Appendix \ref{sec:proof-lemma-causal-interpret}.
\begin{lemma}[Causal interpretation]
  \label{lem:causal-interpret} 
 Suppose that Assumption \ref{asp:main-asp} is satisfied. 
 Then we have $\pp{P}_{Y\langle 0|1 \rangle}=\pp{P}_{Y_0^* | T = 1}$ and $\pp{P}_{Y\langle 1|0 \rangle}=\pp{P}_{Y_1^* | T = 0}$.
\end{lemma}
Lemma \ref{lem:causal-interpret} shows that the distributions $\pp{P}_{Y_0^* | T = 1}$ and $\pp{P}_{Y_1^* | T = 0}$, which play the key role in analyzing distributional treatment effects  \eqref{eq:cond_effect_potential_outcomes} \eqref{eq:effect_potential_outcomes}, can be obtained by estimating the corresponding counterfactual distributions $\pp{P}_{Y\langle 0|1 \rangle}$ and $\pp{P}_{Y\langle 1|0 \rangle}$ defined in terms of observed random variables $(X_t,Y_t)_{t = 0,1}$.
The key assumption in this regard is the conditional exogeneity in Assumption \ref{asp:main-asp}. 


\subsection{Further Explanation on the Distributional Effects of  Covariate Distributions}
\label{sec:more-on-assignment}

We are now in a position to provide further explanation on the distributional difference introduced in Section \ref{sec:effects_by_treatment_assignment}.   
To this end, let us assume that the conditional exogeneity in Assumption 1 is satisfied. 
Then, by Lemmas  \ref{lemma:observable_outcome_dist} and \ref{lem:causal-interpret}, the distributional difference in \eqref{eq:effect_potential_outcomes} can be written as
\begin{align}
 \pp{P}_{Y^*_0 | T}(\y\,|\,0)  -   \pp{P}_{Y^*_0 | T }(\y\,|\,1)  & =  \pp{P}_{Y\left<0| 0\right>}(\y) -  \pp{P}_{Y\left<0|1\right>} (\y) \nonumber \\
&= \int \pp{P}_{Y_0|X_0}(\y|\x)\dd\pp{P}_{X_0}(\x) - \int \pp{P}_{Y_0|X_0}(\y|\x)\dd\pp{P}_{X_1}(\x). \label{eq:dist_differ_counterfactual}
\end{align}
The rhs of \eqref{eq:dist_differ_counterfactual} shows that this distributional difference (if it exists) is due to the difference between the covariate distributions $\pp{P}_{X_0}$ and $\pp{P}_{X_1}$.
In what follows, we provide two distinct interpretations.
First, it quantifies a selection bias that affects the difference in observed outcome distributions. 
Second, it quantifies the causal effect for a policy implemented as a specification of a covariate distribution.

\subsubsection{Quantifying a Selection Bias}

We can decompose the difference in the {\em observed} outcome distributions $\pp{P}_{Y_0}$ and $\pp{P}_{Y_1}$ as
\begin{eqnarray*}
\pp{P}_{Y_0}(\y) - \pp{P}_{Y_1}(\y) &=& \pp{P}_{Y_0^*| T}(\y \,|\, 0) -  \pp{P}_{Y_1^*| T}(\y \,|\, 1) \\
&=& \underbrace{\pp{P}_{Y_0^*| T}(\y \,|\, 0) - \pp{P}_{Y_0^*| T}(\y \,|\, 1)}_{(A)}  +  \underbrace{\pp{P}_{Y_0^*| T}(\y \,|\, 1) -  \pp{P}_{Y_1^*| T}(\y \,|\, 1)}_{(B)} ,
\end{eqnarray*}
 where the first term $(A)$ is the distributional effect of covariate distributions \eqref{eq:dist_differ_counterfactual}, and the second term $(B)$ is the distributional treatment effect on the treated. 
 Thus, the difference in the observed outcome distributions can arise from $(A)$ and/or $(B)$, and the estimation of $(A)$ and $(B)$ is useful in studying the origin of the difference in observed outcome distributions. 
 For instance, if we find that $(A)$ is zero, the difference between the observed outcome distributions is originated from the distributional difference on the treated $(B)$. 
 On the other hand, if $(B)$ is zero, then the difference between the observed outcome distributions is due to $(A)$, i.e., by the selection bias, and is not due to the effects of the treatment.

Note that this difference is different from the difference between the {\em potential} outcome distributions $\pp{P}_{Y_0^*}$, $\pp{P}_{Y_1^*}$, which accounts for the effects of the treatments $0$ and $1$ and thus is of primary interest. The observed outcome distributions $\pp{P}_{Y_0}$, $\pp{P}_{Y_1}$ are biased approximations to the potential outcome distributions, if the treatment assignment is not randomized (i.e., if $X$ and $T$ are not independent).

\subsubsection{Policy as a Specification of a Covariate Distribution} \label{sec:policy-spec-cov-dist}

The distributional difference \eqref{eq:dist_differ_counterfactual} can also be used to quantify the effects of a policy that specifies a covariate distribution.
Recall that we introduced the random variables $X_0, X_1$ as the covariate random variable $X$ conditioned on $T = t$, $t \in \{0,1 \}$, i.e., $X_t := X \,|\, (T= t)$. 
We can instead {\em directly define} two random variables $X_0, X_1$ by specifying their probability distributions $\pp{P}_{X_0} = \pp{P}_{X|T}(\cdot\,|\,0)$ and $\pp{P}_{X_1} = \pp{P}_{X|T}(\cdot\,|\,1)$, respectively. 
In this case, the conditioning $T = t$ for $t \in \{0,1\}$ may be regarded as specifying the covariate distribution $\pp{P}_{X_t} = \pp{P}_{X|T}(\cdot\,|\,t)$ on the space of covariates $\mathcal{X}$.
This specification of the covariate distribution $\pp{P}_{X_t} = \pp{P}_{X|T}(\cdot\,|\,t)$ {\em itself} can be regarded as a certain {\em policy}.\footnote{Here we use the terminology ``policy'' instead of ``treatment'' not to confuse the two notions. In our paper, a ``treatment'' $t \in \{0,1\}$ specifies the corresponding {\em potential} outcome $Y_t^*$ and its distribution $\pp{P}_{Y_t^*}$; thus, the difference between $\pp{P}_{Y_0^*}$ and  $\pp{P}_{Y_1^*}$ characterizes the treatment effects. On the other hand, a ``policy'' here $t \in \{0,1 \}$ specifies the corresponding covariate random variable $X_t$ and  its distribution $\pp{P}_{X_t}.$}

For instance, \citet[Section 5.2]{Rot10} used this formulation to study the effects of smoking of a pregnant mother on the birth weight of the baby. 
There, the observed outcome $Y > 0$ is the birth weight of the baby, and covariates $X := (X^1, X^2, X^3, X^4) \in \mathbb{R}^4$ are relevant features of the mother: $X^1$ is the number of cigarettes per day, $X^2$ is the age, $X^3$ is the weight gain and $X^4$ is the marital status.  
The distribution $\pp{P}_{X_0}$ is the covariate distribution of available data of smoking mothers, while $\pp{P}_{X_1}$ is a transformation of $\pp{P}_{X_0}$ so that the number of cigarettes per day, $X_0^1$, is reduced to 75\%. 
Thus, $T=1$ or $\pp{P}_{X_1}$ may be regarded as a hypothetical policy that reduces the amount of cigarettes of smoking pregnant women. 
Then, $\pp{P}_{\left<0 | 1 \right>}(\y) = \int \pp{P}_{Y_0|X_0}(\y\,|\,\x) \dd\pp{P}_{X_1}(\x)$ is the counterfactual distribution of the birth weights of babies that would have been observed if the mothers had smoked 75 \% less amount of cigarettes than they actually did. 
The comparison to the observed outcome distribution $\pp{P}_{\left<0 | 0 \right>}(\y) = \int \pp{P}_{Y_0|X_0}(\y\,|\,\x) \dd\pp{P}_{X_0}(\x)$  then enables studying the effects of the amount of cigarettes on birth weights. 

Another important instance is the off-policy evaluation task, which will be discussed further in Section \ref{sec:policy-evaluation}.



\subsection{Kernel Mean Embeddings for Counterfactual Distributions}
\label{sec:CME-estimators}

We now define counterfactual mean embeddings.
  Let $\ell$ be a positive definite kernel on $\mathcal{Y}$ with RKHS $\hbspf$, and assume that the support condition in Assumption \ref{asp:main-asp} is satisfied.
  We then refer to the kernel mean embeddings of the counterfactual distributions \eqref{eq:counterfactual-dist} and \eqref{eq:counterfactual-dis-10} 
\begin{eqnarray}  
  \label{eq:def-cme} 
	\mu_{Y\langle 0|1 \rangle} &:=& \int \ell(\cdot,\y) \dd \pp{P}_{Y\langle 0|1 \rangle}(\y) \in  \hbspf, \\
		\mu_{Y\langle 1|0 \rangle} &:=& \int \ell(\cdot,\y) \dd \pp{P}_{Y\langle 1|0 \rangle}(\y) \in  \hbspf, \label{eq:def-cme-10}
\end{eqnarray} 
as {\em counterfactual mean embeddings (CME)}.
Lemma \ref{lem:causal-interpret} implies that, under Assumption \ref{asp:main-asp}, these CMEs are respectively identical to the kernel mean embeddings of $\pp{P}_{Y_0^* | T } (\y|1)$ and $\pp{P}_{Y_1^* | T } (\y|0)$ defined as
$$
 \mu_{Y_0^* | T = 1} := \int \ell(\cdot,\y) \dd \pp{P}_{Y_0^* | T } (\y|1), \quad  \mu_{Y_1^* | T = 0} := \int \ell(\cdot,\y) \dd \pp{P}_{Y_1^* | T } (\y|0).
$$
Therefore, by defining an empirical estimator of the CME \eqref{eq:def-cme}, one can hope to estimate the distributional treatment effects in \eqref{eq:cond_effect_potential_outcomes} and \eqref{eq:effect_potential_outcomes}, which will be done below.

\paragraph{Estimating counterfactual mean embeddings.}
\label{sec:empirical-cme}
In what follows, we introduce our estimator of the CME $\mu_{Y\langle 0|1 \rangle}$ defined in \eqref{eq:def-cme}; one can define an estimator of \eqref{eq:def-cme-10} in a similar manner. 
In practice, it is not possible to obtain a sample from $\pp{P}_{Y\langle 0|1 \rangle}$, and therefore the counterfactual mean embedding $\mu_{Y\langle 0|1 \rangle}$ cannot be estimated directly. 
Instead, we propose an estimator that uses samples from $\pp{P}_{X_0 Y_0}$ and $\pp{P}_{X_1}$ to estimate $\mu_{Y\langle 0|1 \rangle}$.
To this end, first note that $\mu_{Y\langle 0|1 \rangle}$  in \eqref{eq:def-cme} can be written in terms of the conditional mean embedding \eqref{eq:cond-kme} of $\pp{P}_{Y_0|X_0=\x}$:
$$
	\mu_{Y\langle 0|1 \rangle} =  \int \mu_{Y_0|X_0 = \x} \dd\pp{P}_{X_1}(\x) \in \hbspf,
$$
where 
$
\mu_{Y_0|X_0 = \x} := \int \ell(\cdot, \y) \dd\pp{P}_{Y_0|X_0=\x}(\y) \in \hbspf.
$
This formulation suggests that $\mu_{Y\langle 0|1 \rangle}$ can be estimated by i) constructing an estimator of the conditional mean embedding $\mu_{Y_0|X_0 = \x}$ and then ii) taking its average over $\pp{P}_{X_1}(\x)$.
This is how our estimator is derived below.

Suppose that we are given independent samples $(\x_1,\y_1),\ldots,(\x_n,\y_n)$ from $\pp{P}_{Y_0X_0}(\x,\y)$ and $\x'_1,\ldots,\x'_m$ from $\pp{P}_{X_1}(\x)$. 
For $\x \in \inspace$, let $\muh_{Y_0|X_0=\x}$ denote the estimate \eqref{eq:cond-kme-est} of the conditional mean embedding $\muv_{Y_0|X_0=\x}$ based on $(\x_1,\y_1),\ldots,(\x_n,\y_n)$.
Then, an empirical estimator of $\mu_{Y\langle 0|1 \rangle}$ is defined and expressed as 
  \begin{equation}
    \label{eq:empirical-cme}
     \hat{\mu}_{Y\langle 0|1 \rangle} := \frac{1}{m} \sum_{j=1}^m \muh_{Y_0|X_0=\x_j'} =    \sum_{i=1}^n\beta_i\ell(\cdot, \y_i) \quad \text{with} \quad (\beta_1,\dots,\beta_n)^\top = (\kmat + n\varepsilon \id)^{-1}\widetilde{\kmat}\mathbf{1}_m,
  \end{equation}
  \noindent where $\varepsilon > 0$ is a regularization constant, $\mathbf{1}_m=(1/m,\ldots,1/m)^\top \in \mathbb{R}^m$, $\kmat \in \mathbb{R}^{n \times n}$ with $\kmat_{ij}=k(\x_i,\x_j)$, and $\widetilde{\kmat} \in \mathbb{R}^{n \times m}$ with  $\widetilde{\kmat}_{ij} = k(\x_i,\x'_j)$.

The proposed estimator \eqref{eq:empirical-cme} is nonparametric, and can be implemented without knowledge about parametric forms of the conditional $\pp{P}_{Y_0|X_0}$ and marginal $\pp{P}_{X_1}$.
Thus, the estimator is useful when such knowledge is not available.
In Section \ref{sec:theory-main}, we theoretically analyze the asymptotic behavior of the estimator, proving its consistency and deriving convergence rates.
In doing so, we elucidate conditions required for the consistency of the proposed estimator.


The computational complexity of our estimator \eqref{eq:empirical-cme} is $\mathcal{O}(n^3)$ because of the matrix inversion, which may be expensive when the sample size $n$ is huge.
To reduce the complexity, one can adopt existing approximation methods such as Nystr\"om method and random Fourier features \citep{Williams01:Nystrom,Rahimi08:RFF}.

We note that the form of the estimator is identical to the {\em kernel sum rule} \citep[Section 4.1]{Song2013}, a mean embedding approach to computing forward probabilities in Bayesian inference.
The way we use the estimator is different from this previous approach, however.
That is, we use our estimator to estimate the counterfactual distribution and distributional causal effects \eqref{eq:effect_potential_outcomes}, and this requires Assumption \ref{asp:main-asp} to hold for data (or for the population random variables), as shown in Lemma \ref{lem:causal-interpret}.



\subsection{Kernel Treatment Effects}
\label{sec:kte}

We quantify distributional treatment effects by using the RKHS distance between the mean embeddings of potential outcome distributions under consideration. 
We call this approach {\em Kernel Treatment Effects (KTE)}. 
We show below how KTEs can be defined for the different distributional treatment effects discussed in Section \ref{sec:potential-outcome}.

\subsubsection{KTE for Distributional Treatment Effects}
\label{sec:KTE-DATE}

As before, let $\ell$ be a kernel on the output space $\mathcal{Y}$ and $\hbspf$ be its RKHS. 
For the distributional treatment effect \eqref{eq:dist-treatment-effect} discussed in Section \ref{sec:DATE}, the corresponding KTE is defined as 
  \begin{equation}
  \label{eq:kte-po}
    \text{KTE}(Y_0^*,Y_1^*,\hbspf) := \|\mu_{Y_0^*} - \mu_{Y_1^*}\|_{\hbspf}, 
  \end{equation}where $\mu_{Y_0^*}$ and  $\mu_{Y_1^*}$ are the kernel mean embeddings of the distributions of potential outcomes $\pp{P}_{Y_0^*}$ and  $\pp{P}_{Y_1^*}$, respectively, i.e.,
\begin{equation} \label{eq:embeddings-potential-outcomes}
\mu_{Y_0^*} := \int \ell(\cdot, \y) \dd\pp{P}_{Y_0^*}(\y),  \qquad \mu_{Y_1^*} := \int \ell(\cdot, \y) \dd\pp{P}_{Y_1^*}(\y).
\end{equation}

The KTE \eqref{eq:kte-po} may be regarded as a generalization of the ATE \eqref{eq:ate} in the sense that, if $\ell$ is the linear kernel $\ell(\y,\y') = \langle \y,\y' \rangle$ on $\mathcal{Y} = \mathbb{R}^d$, then the KTE only distinguishes the means of the two outcome distributions.
By using a different kernel $\ell$, the KTE may capture the differences between higher-order statistics of the outcome distributions  $\pp{P}_{Y_0^*}$ and  $\pp{P}_{Y_1^*}$.   
For instance, if $\ell$ is a polynomial kernel $\ell(\y, \y') = \left( \left<  \y, \y' \right> + c \right)^m$ of degree $m \in \mathbb{N}$ with $c > 0$, then the KTE \eqref{eq:kte-po} is equal to $0$ if and only if $\pp{P}_{Y_0^*}$ and  $\pp{P}_{Y_1^*}$ have the same moments up to degree $m$  (see, \eg, \citealt[Chapter 3]{Muandet17:KME}).

If $\ell$ is a characteristic kernel, such as Gaussian and Mat\'ern kernels, then the KTE \eqref{eq:kte-po} is equal to $0$ if and only if the two distributions $\pp{P}_{Y_0^*}$ and  $\pp{P}_{Y_1^*}$ are the same. 
In this case, the KTE takes a positive value if and only if there is a difference between $\pp{P}_{Y_0^*}$ and  $\pp{P}_{Y_1^*}$. 
This means that the KTE informs the existence of any difference in the potential outcome distributions, quantifying the distributional treatment effect. 

The question is how to estimate the KTE \eqref{eq:kte-po} from data. 
As in \eqref{eq:observational-data}, let $(\x_i,t_i, \y_i )_{i=1}^N$ be observational data, which are i.i.d.~with the random variables $(X,T,Y)$. 
Recall that $Y$ is the observed outcome and thus given by $Y = \mathbbm{1}(T=0) Y_0^* + \mathbbm{1}(T=1) Y_1^*$ where $Y_0^*$ and $Y_1^*$ are the potential outcomes.  
In observational studies, it is common to use the {\em propensity score} $e(\x) := \mathbb{E}[ T \,|\, X=\x]$, the conditional probability of the treatment assignment $T$ being made given that the covariates are $X = \x$, to define an unbiased estimator of the average treatment effect $\mathbb{E}[Y_1^*] - \mathbb{E}[Y_0^*]$  \citep{Rosenbaum83:Propensity}.  
We show here that the same strategy of {\em inverse propensity weighting} \cite[Section III-C]{Imbens04:ATE} can be straightforwardly used to define unbiased estimators of the mean embeddings $\mu_{Y_1^*}$ and  $\mu_{Y_0^*}$ of potential outcome distributions  $\pp{P}_{Y_1^*}$ and $\pp{P}_{Y_0^*}$, respectively, thus providing a way of estimating the KTE.
That is, assuming that the propensity $e(\x)$ is available, we define
\begin{equation} \label{eq:propensity-est}
\hat{\mu}_{Y_1^*} := \frac{1}{m} \sum_{i = 1}^N \frac{ t_i \ell(\cdot, \y_i)  }{e(\x_i)}, 
\qquad  
\hat{\mu}_{Y_0^*} := \frac{1}{n} \sum_{j = 1}^N \frac{ (1-t_j) \ell(\cdot, \y_j ) }{1- e(\x_j)}, 
\end{equation}
where $m := \sum_{i=1}^N t_i$ and $n := \sum_{j = 1}^N (1-t_j)$ are the populations of treated and control groups, respectively. 

In the special case of a completely randomized experiment where $X$ and $T$ are independent and thus the propensity is $e(\x) = 1/2$ for all $\x \in \mathcal{X}$, the above estimators reduce to the standard empirical estimators of mean embeddings: $\hat{\mu}_{Y_1^*} := \frac{2}{m} \sum_{i = 1}^N t_i \ell(\cdot, \y_i) $ and $\hat{\mu}_{Y_0^*} := \frac{2}{n} \sum_{j = 1}^N (1-t_j) \ell(\cdot, \y_j )$. 
Note that these uniformly-weighted empirical estimators are {\em biased} if the experiment is {\em not} completely randomized, \ie, in observational studies. 
This is because, for instance, the sample $\y_i$ contributing to $\hat{\mu}_{Y_1^*}$ follows the distribution of $Y_1^* | T = 1$, which is different from the unconditional $Y_1^*$.  
Thus, we need the inverse propensity weighting to obtain unbiased estimators in the case of observational studies. 
 
 The following result shows that the estimators \eqref{eq:propensity-est} are indeed unbiased estimators of the corresponding mean embeddings  $\mu_{Y_1^*}$ and  $\mu_{Y_0^*}$  of potential outcome distributions. 
 The proof is presented in Appendix \ref{sec:proof-propensity-unbiased}.

\begin{theorem} \label{theo:propensity-unbiased}
Suppose that $0 < e(\x) < 1$ for all $\x \in \mathcal{X}$ and that the conditional exogeneity in Assumption \ref{asp:main-asp} is satisfied. 
Let $(\x_i,t_i, \y_i )_{i=1}^N$ be i.i.d.~with $(X,T,Y)$, and let $\hat{\mu}_{Y_1^*}$ and $\hat{\mu}_{Y_0^*}$ be the estimators \eqref{eq:propensity-est} of the mean embeddings  $\mu_{Y_1^*}$ and  $\mu_{Y_0^*}$ of the potential outcome distributions $\pp{P}_{Y_1^*}$ and $\pp{P}_{Y_0^*}$ in \eqref{eq:embeddings-potential-outcomes}.
Then, we have
$$
\mathbb{E} [  \hat{\mu}_{Y_1^*} ]  =   \mu_{Y_1^*}, \qquad 
 \mathbb{E} [\hat{\mu}_{Y_0^*}] =  \mu_{Y_0^*} . 
$$
\end{theorem}

Theorem \ref{theo:propensity-unbiased} shows that the estimators \eqref{eq:propensity-est}  are unbiased, but does not say anything about their convergence rates as the sample size goes to infinity. 
The following result provides this; it essentially shows that the estimators \eqref{eq:propensity-est} converge to the mean embeddings  $\mu_{Y_1^*}$ and  $\mu_{Y_0^*}$ at the same rates as the standard kernel mean estimators, which are minimax optimal \citep{TolstikhinSM17:Minimax}.
The key assumption here is that the propensity $e(\x)$ is uniformly lower- and upper-bounded away from $0$ and $1$. 
The proof is presented in Appendix \ref{sec:proof-conv-rate-propensity}.

\begin{theorem} \label{theo:conv-rate-propensity}
Suppose the propensity score $e(\x)$ satisfies $\inf_{\x \in \mathcal{X}} e(\x) > 0$ and $\sup_{\x \in \mathcal{X}} e(\x) < 1$, that $\sup_{\y \in \mathcal{Y}} \ell(\y,\y) < \infty$, and that the conditional exogeneity in Assumption \ref{asp:main-asp} is satisfied.
Let $(\x_i,t_i, \y_i )_{i=1}^N$ be i.i.d.~with $(X,T,Y)$, and let $\hat{\mu}_{Y_1^*}$ and $\hat{\mu}_{Y_0^*}$ be the estimators \eqref{eq:propensity-est} of the mean embeddings  $\mu_{Y_1^*}$ and  $\mu_{Y_0^*}$ of the potential outcome distributions $\pp{P}_{Y_1^*}$ and $\pp{P}_{Y_0^*}$ in \eqref{eq:embeddings-potential-outcomes}.
Then, we have
$$
\mathbb{E}\left[ \left\| \hat{\mu}_{Y_1^*} - \mu_{Y_1^*} \right\|_{\hbspf}^2\right] = O(m^{-1}), \qquad \mathbb{E}\left[ \left\| \hat{\mu}_{Y_0^*} - \mu_{Y_0^*} \right\|_{\hbspf}^2\right] = O(n^{-1}), \qquad (N \to \infty),
$$
where $m := \sum_{i=1}^N t_i$ and $n := \sum_{i=1}^N (1-t_i)$.
\end{theorem}

Based on the estimators \eqref{eq:propensity-est}, we can define a consistent estimator of \eqref{eq:kte-po} as 
\begin{align}\label{eq:emp-kte-po}
& \widehat{\text{KTE}}_b^2 (Y_0^* ,Y_1^*,\hbspf)
:= \left\| \hat{\mu}_{Y_1^*} -  \hat{\mu}_{Y_0^*} \right\|_{\hbspf}^2 =   \left\| \hat{\mu}_{Y_1^*}   \right\|_{\hbspf}^2  - 2 \left< \hat{\mu}_{Y_1^*}, \hat{\mu}_{Y_0^*}  \right>_{\hbspf} + \left\| \hat{\mu}_{Y_0^*}   \right\|_{\hbspf}^2   \\
&= \frac{1}{m^2}\sum_{i,j=1}^N \frac{ t_i t_j \ell (\y_i, \y_j) }{ e(\x_i) e(\x_j) } - \frac{2}{mn} \sum_{i,j=1}^N \frac{ t_i (1-t_j) \ell (\y_i, \y_j) }{ e(\x_i) (1-e(\x_j)) }    + \frac{1}{n^2}\sum_{i,j=1}^N \frac{ (1-t_i)(1-t_j)\ell (\y_i, \y_j) }{ (1-e(\x_i))(1-e(\x_j)) }, \nonumber
\end{align}
where the last equality follows from the reproducing property of the kernel $\ell$.
By the triangle inequality, we  have $|\widehat{\text{KTE}}_b (Y_0^* ,Y_1^*,\hbspf) -  \text{KTE} (Y_0^* ,Y_1^*,\hbspf)|  = | \left\| \hat{\mu}_{Y_1^*} -  \hat{\mu}_{Y_0^*} \right\|_{\hbspf}  - \|\mu_{Y_0^*} - \mu_{Y_1^*}\|_{\hbspf} |
\leq |  \left\| \hat{\mu}_{Y_1^*} - \mu_{Y_1^*} \right\|_{\hbspf} +  \left\| \hat{\mu}_{Y_0^*} - \mu_{Y_0^*} \right\|_{\hbspf} | = O_p(m^{-1/2} + n^{-1/2})$ as $n, m \to \infty$, which shows that the estimator \eqref{eq:emp-kte-po} is asymptotically unbiased.

Note that \eqref{eq:emp-kte-po} is a biased estimator, while being asymptotically unbiased. This bias is caused by the terms with identical indices ($i = j$) in the first and third summations of \eqref{eq:emp-kte-po}. 
Thus, by subtracting these terms, an unbiased estimator of the KTE can be defined as 
\begin{align} \label{eq:unbiased-KTE-DATE}
& \widehat{\text{KTE}}_u^2 (Y_0^* ,Y_1^*,\hbspf) := 
  \frac{1}{m(m-1)}  \sum_{i \not= j} \frac{ t_i t_j \ell (\y_i, \y_j) }{ e(\x_i) e(\x_j) } \\ 
  & - \frac{2}{mn} \sum_{i,j=1}^N \frac{ t_i (1-t_j) \ell (\y_i, \y_j) }{ e(\x_i) (1-e(\x_j)) } + \frac{1}{n (n-1) }  \sum_{i \not = j}   \frac{ (1-t_i)(1-t_j)\ell (\y_i, \y_j) }{ (1-e(\x_i))(1-e(\x_j)) }. \nonumber
\end{align}
By similar arguments as in the proof of Theorem \ref{theo:conv-rate-propensity}, it can be shown that this is indeed an unbiased estimator of (the square of) KTE \eqref{eq:kte-po}.
Moreover, since it can be shown that
$$
\left| \widehat{\text{KTE}}_u^2 (Y_0^* ,Y_1^*,\hbspf) - \widehat{\text{KTE}}_b^2 (Y_0^* ,Y_1^*,\hbspf) \right| = O_p(m^{-1} + n^{-1}) \quad (n, m \to \infty)  
$$ 
given that the assumptions in Theorem \ref{theo:conv-rate-propensity} hold, this unbiased estimator \eqref{eq:unbiased-KTE-DATE} enjoys the same convergence rate as the biased one \eqref{eq:emp-kte-po}: $|\widehat{\text{KTE}}_u (Y_0^* ,Y_1^*,\hbspf) -  \text{KTE} (Y_0^* ,Y_1^*,\hbspf)|  = O_p(m^{-1/2} + n^{-1/2})$ as $n, m \to \infty$.

 \subsubsection{KTE for Distributional Treatment Effects on the Treated}
 
 We define KTE for the distributional effect  $\pp{P}_{Y_1^* | T } (\cdot\,|\,t) - \pp{P}_{Y_0^* | T } (\cdot\,|\,t)$ introduced in Section \ref{sec:effects_by_treatment}, where  $t \in \{0,1\}$.
We only consider the case $t = 1$ here, which is interpreted as the distributional treatment effect for the treated; the case $t = 0$ can be defined similarly. 
The definition is
  \begin{equation}
    \label{eq:kte-treated}
        \text{KTE}( Y_1^* | (T = 1),  Y_0^* | (T = 1) ,\hbspf)  :=  \left\| \mu_{Y_1^* | T = 1} -  \mu_{Y_0^* | T = 1} \right\|_{\hbspf},
  \end{equation}
 where $\mu_{Y_1^* | T = 1} $ and $\mu_{Y_0^* | T = 1}$ are the kernel mean embeddings of $\pp{P}_{Y_1^* | T } (\cdot\,|\,1) $ and $\pp{P}_{Y_0^* | T } (\cdot\,|\,1) $, respectively:
 $$
   \mu_{Y_1^* | T = 1}   := \int \ell(\cdot, \y) \dd \pp{P}_{Y_1^* | T } (\y \,|\,1), 
   \qquad 
   \mu_{Y_0^* | T = 1}   := \int \ell(\cdot, \y) \dd \pp{P}_{Y_0^* | T } (\y \,|\,1).
 $$
 
Lemma \ref{lemma:observable_outcome_dist} shows that  $\pp{P}_{Y\left<1|1\right>}(\y)  = \pp{P}_{Y_1^* | T} (\y \,|\,1)$, while Lemma \ref{lem:causal-interpret} implies that $\pp{P}_{Y\langle 0|1 \rangle}=\pp{P}_{Y_0^* | T = 1}$  under Assumption \ref{asp:main-asp}.
Thus, we can define an estimator of the above KTE \eqref{eq:kte-treated} as follows.
Let $\hat{\mu}_{Y\langle 0|1 \rangle} = \sum_{i=j}^n\beta_i\ell(\cdot, \y_j)$ be the estimator \eqref{eq:empirical-cme} of the CME, and let $\hat{\mu}_{Y\langle 1|1 \rangle} := \frac{1}{m} \sum_{i=1}^m \ell(\check{\y}_i,\cdot)$ where $\check{\y}_1,\ldots,\check{\y}_m$ is a sample from $\pp{P}_{Y\langle 1|1 \rangle}$. 
Note that such  $\check{\y}_1,\ldots,\check{\y}_m$ can be obtained in practice, as $\pp{P}_{Y\langle 1|1 \rangle}$ is the distribution of the observed outcome $Y$ given that the treatment assignment is $T = 1$.
Then, we can define an empirical estimator of the KTE in \eqref{eq:kte-treated} as follows:
\begin{align}\label{eq:emp-kte-treated}
\MoveEqLeft \widehat{\text{KTE}}^2 (Y_1^* | (T = 1),  Y_0^* | (T = 1) ,\hbspf) \nonumber \\
&:= \left\| \hat{\mu}_{Y\langle 1|1 \rangle} -  \hat{\mu}_{Y\langle 0|1 \rangle} \right\|_{\hbspf}^2 \nonumber \\ 
&= \frac{1}{m^2}\sum_{i,j=1}^m \ell (\check{\y}_i, \check{\y}_j) - \frac{2}{m} \sum_{i=1}^m \sum_{j=1}^n \beta_j \ell (\check{\y}_i, \y_j) + \sum_{i,j=1}^n \beta_i \beta_j \ell (\y_i, \y_j).
\end{align}

\subsubsection{KTE for Distributional Effects of the Covariate Distributions}
\label{sec:KTE-dist-effect-cov}

Similarly, we can define a KTE for the distributional effect $\pp{P}_{Y_0^* | T } (\cdot\,|\,0) - \pp{P}_{Y_0^* | T } (\cdot \,|\,1)$ defined in \eqref{eq:effect_potential_outcomes} by
  \begin{equation}
    \label{eq:kte}
        \text{KTE}( Y_0^* | (T = 0),  Y_0^* | (T = 1) ,\hbspf)  :=  \left\| \mu_{Y_0^* | T = 0} -  \mu_{Y_0^* | T = 1} \right\|_{\hbspf},
  \end{equation}
  where $\mu_{Y_0^* | T = 0} $ and $\mu_{Y_0^* | T = 1}$ are the kernel mean embeddings of $\pp{P}_{Y_0^* | T } (\cdot \,|\,0) $ and $\pp{P}_{Y_0^* | T } (\cdot \,|\,1) $, respectively, i.e.,
  $$
\mu_{Y_0^* | T = 0}   := \int \ell(\cdot, \y) \dd \pp{P}_{Y_0^* | T } (\y \,|\,0), 
\qquad 
\mu_{Y_0^* | T = 1}   := \int \ell(\cdot, \y) \dd \pp{P}_{Y_0^* | T } (\y \,|\,1).
  $$

Lemma \ref{lemma:observable_outcome_dist} shows that  $\pp{P}_{Y\left<0|0\right>}(\y)  = \pp{P}_{Y_0^* | T} (\y \,|\,0)$, and Lemma \ref{lem:causal-interpret}, under Assumption \ref{asp:main-asp}, implies that $\pp{P}_{Y\langle 0|1 \rangle}=\pp{P}_{Y_0^* | T = 1}$.
Therefore, we define an estimator of \eqref{eq:kte} in the following way.
Let $\hat{\mu}_{Y\langle 0|1 \rangle}$ be the estimator \eqref{eq:empirical-cme} of the CME, and let $\hat{\mu}_{Y\langle 0|0 \rangle} := \frac{1}{n} \sum_{i=1}^n \ell(\cdot, \tilde{\y}_i)$ where $\tilde{\y}_1,\ldots,\tilde{\y}_n$ is a sample from $\pp{P}_{Y\langle 0|0 \rangle}$. 
Note that such  $\tilde{\y}_1,\ldots,\tilde{\y}_n$ can be obtained in practice, as $\pp{P}_{Y\langle 0|0 \rangle}$ is the distribution of the observed outcome $Y$ given that the treatment assignment is $T = 0$. 
Then, we can define an empirical estimator of the KTE in \eqref{eq:kte} as follows:
\begin{align}\label{eq:emp-kte}
\MoveEqLeft \widehat{\text{KTE}}^2 (Y_0^* | (T = 0),  Y_0^* | (T = 1) ,\hbspf) \nonumber \\
&:= \left\| \hat{\mu}_{Y\langle 0|0 \rangle} -  \hat{\mu}_{Y\langle 0|1 \rangle} \right\|_{\hbspf}^2 \nonumber \\ 
&= \frac{1}{n^2}\sum_{i,j=1}^n \ell (\tilde{\y}_i, \tilde{\y}_j) - \frac{2}{n} \sum_{i,j=1}^n \beta_j \ell (\tilde{\y}_i, \y_j) + \sum_{i,j=1}^n \beta_i \beta_j \ell (\y_i, \y_j).
\end{align}


In the next section, we analyze the convergence behavior of the proposed CME estimator in \eqref{eq:empirical-cme} as the sample size $n$ goes to infinity. 
Readers who are interested in applications may skip the next section and jump to Section \ref{sec:applications} and Section \ref{sec:policy-evaluation} directly.

\section{Convergence Analysis of the CME Estimator}
\label{sec:theory-main}

We provide here a convergence analysis of the CME estimator introduced in Section \ref{sec:CME-estimators}. In Section \ref{sec:consistency}, we first establish its consistency under mild assumptions. In Section \ref{sec:convergence-rate}, we derive its convergence rates by making a quantitative assumption on the smoothness of certain functions involved.


We first introduce necessary notation and definitions.
We assume that the covariate space $\mathcal{X}$ and the outcome space $\mathcal{Y}$ are measurable spaces, and that the kernels $k$ and $\ell$ are measurable on $\mathcal{X}$ and $\mathcal{Y}$, respectively, with $\hbspace$ and $\hbspf$  being their respective RKHSs.
Let $\pp{P}_{X_0}$ and $\pp{P}_{X_1}$ be the probability distributions of the random variables $X_0 \in \mathcal{X}$ and $X_1 \in \mathcal{X}$, respectively (see Definition \ref{def:random_variables} for the definition of these random variables).

Let $L_2(\pp{P}_{X_0})$ be the Hilbert space of square-integral functions\footnote{More precisely, each element in $L_2(\pp{P}_{X_0})$ is a $\pp{P}_{X_0}$-equivalent class of functions; see Appendix \ref{sec:app-preliminary}.} with respect to $\pp{P}_{X_0}$:
$$
L_2(\pp{P}_{X_0}) := \left\{ f: \mathcal{X} \to \mathbb{R} \mid  \int f^2(x) \dd\pp{P}_{X_0}(x) < \infty \right\},
$$
which is equipped with the inner product $\left< f, g \right>_{L_2(\pp{P}_{X_0}) } := \int f(x)g(x) \dd\pp{P}_{X_0}(x)$ and the resulting norm $ \| f \|_{L_2(\pp{P}_{X_0})} := \sqrt{ \left< f, f \right>_{L_2(\pp{P}_{X_0}) }  }$.
Let $\pp{P}_{X_0} \otimes \pp{P}_{X_0}$ be the product measure of $\pp{P}_{X_0}$ and $\pp{P}_{X_0}$ on the product space $\mathcal{X} \times \mathcal{X}$.

\subsection{Consistency} 
\label{sec:consistency}
To establish the consistency of the CME estimator, we require the following conditions.
\begin{assumption} \label{as:g-and-theta}
Assume that the following conditions are satisfied:
\begin{enumerate}[label={\bf (\roman*)}]
\item \label{as:measruable}
The kernels $k$ and $\ell$ are bounded on $\mathcal{X}$ and $\mathcal{Y}$, respectively, that is, $\sup_{\x \in \mathcal{X}}k(\x,\x) < \infty$ and $\sup_{\y \in \mathcal{Y}} \ell(\y,\y) < \infty$.



\item \label{as:RKHS-dense}
The RKHS $\hbspace$ of $k$ is dense in $L_2(\pp{P}_{X_0})$.

\item \label{as:radon-nikodym}
The distribution $\pp{P}_{X_1}$ is absolutely continuous with respect to $\pp{P}_{X_0}$ with the Radon-Nikodym derivative $g := \dd\pp{P}_{X_1}/\dd\pp{P}_{X_0}$ satisfying $g \in L_2(\pp{P}_{X_0})$.

\item \label{as:cov-op-embed} 
 $(\x_1,\y_1),\ldots,(\x_n,\y_n)$ are i.i.d.~observations of the random variables $(X_0, Y_0)$, and $\x'_1,\ldots,\x'_m$ are i.i.d.~observations of the random variable $X_1$, with $n = m$. 
\end{enumerate}
\end{assumption}

\begin{remark} \rm
We make the following comments on  Assumption \ref{as:g-and-theta}.
\begin{itemize}
\item  The boundedness condition \ref{as:measruable} is satisfied, for instance, if $k$ and $\ell$ are shift-invariant kernels, such as Gaussian and Mat\'ern kernels.

\item 
The condition \ref{as:RKHS-dense} requires that the RKHS be rich enough to approximate square-integrable functions with respect to $\pp{P}_{X_0}$. 
For instance, this is satisfied by the Gaussian kernel \citep[Theorem 4.63]{SteChr2008}, and therefore by any kernel whose RKHS is larger than that of the Gaussian kernel, such as Laplace and Mat\`ern kernels \citep[Theorem 4.48]{SteChr2008}.

\item The condition \ref{as:radon-nikodym} requires the support of $\pp{P}_{X_1}$ be included in that of $\pp{P}_{X_0}$, and thus is related to the common support assumption in Assumption \ref{asp:main-asp}.
If both $\pp{P}_{X_1}$ and $\pp{P}_{X_0}$ have density functions $p_{X_1}$ and $p_{X_0}$, respectively, with respect to a common reference measure (e.g., the Lebesgue measure in the case of $\mathcal{X} \subset \mathbb{R}^d$), then the Radon-Nikodym derivative becomes the density ratio or the importance weight function $g(x) = p_{X_1}(x)/p_{X_0}(x)$. 
Thus, the square-integrability of $g$ requires, intuitively, that $p_{X_1}$ should not be very different from $p_{X_0}$.

\item 
In the condition \ref{as:cov-op-embed}, we assume $n = m$ for simplicity of presentation.
\end{itemize}

\end{remark}

Before presenting the result, we introduce a function $\theta: \mathcal{X} \times \mathcal{X} \to \mathbb{R}$ defined by 
\begin{equation} \label{eq:theta-def}
\theta(\x,\tilde{\x}) := \iint \ell(\y,\tilde{\y}) \dd\pp{P}_{Y_0|X_0}(\y|\x) \dd\pp{P}_{Y_0|X_0}(\tilde{\y}|\tilde{\x}) .
\end{equation} 
This function appears in the proof of consistency, and also is needed to derive convergence rates in Section \ref{sec:convergence-rate}.
Note that the assumption in \ref{as:measruable} that $\ell$ being bounded implies that $\theta \in L_2(\pp{P}_{X_0} \otimes \pp{P}_{X_0})$; this property is used in the proof of consistency. 

Theorem \ref{theo:uinf_conv} below shows the consistency of the CME estimator \eqref{eq:empirical-cme}.
The proof can be found in Appendix \ref{sec:proof-consistency}.

\begin{theorem}[Consistency] \label{theo:uinf_conv}
Suppose that Assumption \ref{as:g-and-theta} is satisfied.
Let $\hat{\mu}_{Y\langle 0|1 \rangle}$ be the estimator defined in \eqref{eq:empirical-cme} with a regularization constant $\varepsilon_n > 0$.
Then if $\varepsilon_n \to 0$ and $n^{1/2}\varepsilon_n \to \infty$ as $n \to \infty$, we have 
$$\left\|  \hat{\mu}_{Y\langle 0|1 \rangle}  - \mu_{Y\langle 0|1 \rangle} \right\|_{\hbspf} \to 0$$ 
in probability as $n \to \infty$.
\end{theorem}


\begin{remark} \rm
As discussed, the form of the CME estimator \eqref{eq:empirical-cme} is the same as that of the kernel sum rule, and \citet[Theorem 8]{Fukumizu13:KBR} proves its consistency. 
Unlike ours, however, \citet{Fukumizu13:KBR} assume that the function $\theta$ in \eqref{eq:theta-def} belongs to the tensor-product RKHS $\hbspace \otimes \hbspace$, which is a rather strong assumption for proving just the consistency.
For instance, if $\hbspace$ is the RKHS of the Gaussian kernel, then this assumption requires that $\theta$ be infinitely differentiable. 
The theoretical contribution of our analysis is in removing this condition.
\end{remark}

Recall that by Lemma \ref{lem:causal-interpret} we have $\mu_{Y\langle 0|1 \rangle} = \mu_{Y_0^* | T = 1}$ under the conditional exogeneity condition (Assumption \ref{asp:main-asp}). 
Thus, Theorem \ref{theo:uinf_conv} implies that the CME estimator is consistent in estimating $\mu_{Y_0^* | T = 1}$, as summarized in the following corollary. 
This justifies the use of the CME estimator in dealing with counterfactual questions, as will be described in Section \ref{sec:applications}.

\begin{corollary}
\label{coro:consistency}
Suppose that Assumptions  \ref{asp:main-asp} and \ref{as:g-and-theta} are satisfied.
Let $\hat{\mu}_{Y\langle 0|1 \rangle}$ be the estimator defined in \eqref{eq:empirical-cme} with a regularization constant $\varepsilon_n > 0$.
Then, if $\varepsilon_n \to 0$ and $n^{1/2}\varepsilon_n \to \infty$ as $n \to \infty$, we have 
$$\left\|  \hat{\mu}_{Y\langle 0|1 \rangle}  - \mu_{Y_0^* | T = 1} \right\|_{\hbspf} \to 0$$ 
in probability as $n \to \infty$.
\end{corollary}

\subsection{Convergence Rates} 
\label{sec:convergence-rate}

Next, we present a result on the convergence rate of the CME estimator \eqref{eq:empirical-cme}.
This result is obtained based on certain smoothness assumptions on the Radon-Nikodym derivative $g = \dd\pp{P}_{X_1}/\dd\pp{P}_{X_0}$ and the function $\theta$ defined in \eqref{eq:theta-def}.
To state these assumptions,  we need to introduce the following concepts, details of which can be found in Appendix \ref{sec:app-preliminary}.

In the sequel, $I \subset \mathbb{N}$ denotes a set of indices, which is a finite set or an infinite set depending on whether the RKHS $\hbspace$ is finite dimensional (e.g., if $k$ is a linear or polynomial kernel) or infinite dimensional (e.g., if $k$ is a Gaussian or Mat\'ern kernel). 
We define an integral operator $T : L_2(\pp{P}_{X_0}) \to L_2(\pp{P}_{X_0})$ by 
\begin{equation*} \label{eq:kernel-integral-op}
    T f := \int k(\cdot, \x)f(\x) \dd\pp{P}_{X_0}(\x), \quad f \in L_2(\pp{P}_{X_0}).
\end{equation*}
Intuitively, the output function $Tf$ is a smoother version of the input function $f$, as $Tf$ can be seen as a convolution between $f$ and the kernel $k$.

Under Assumption \ref{as:g-and-theta} \ref{as:measruable} and \ref{as:RKHS-dense}, there exist at most countable families of functions $(e_i)_{i \in I} \subset \hbspace$ and the associated positive constants $(\mu_i)_{i \in I} \subset (0,\infty)$ such that i) $\mu_1 \geq \mu_2 \geq \cdots > 0$, that ii) $(\mu_i^{1/2} e_i)_{i \in I}$ is an orthonormal basis (ONB) in $\hbspace$, that iii) $( e_i )_{i \in I}$ is an ONB in $L_2(\pp{P}_{X_0})$, and that iv) the integral operator can be written as
$$
T f = \sum_{i \in I} \mu_i \left< f, e_i \right>_{L_2(\pp{P}_{X_0})} e_i,
$$ 
with convergence in $L_2(\pp{P}_{X_0})$; see Lemmas \ref{lemma:op-eigen-decomp} and \ref{lemma:ONB-mercer-iso} in Appendix \ref{sec:app-preliminary}.
In other words, the pairs $(\mu_i, e_i)_{i \in I}$ are eigenvalues and eigenfunctions of the integral operator: $T e_i = \mu_i e_i$ for $i \in I$.
Based on this eigendecomposition, one can define a {\em power} of the integral operator $T$: for a constant $\alpha \geq 0$, the $\alpha$-th power of $T$ is defined as 
$$
T^\alpha f := \sum_{i \in I} \mu_i^\alpha \left< f, e_i \right>_{L_2(\pp{P}_{X_0})} e_i, \quad f \in L_2(\pp{P}_{X_0}).
$$ 

We now make the following assumption about the smoothness of the Radon-Nikodym derivative $g = \dd\pp{P}_{X_1}/\dd\pp{P}_{X_0}$, where ${\rm Range}(T^\alpha)$ denotes the range or image of $T^\alpha$.
This way of stating a smoothness condition  is common in learning theory for kernel methods, e.g., \citet{CapDev07,SmaZho07,Fukumizu13:KBR}.
\begin{assumption} \label{as:range-assumption-g}
There exists a constant $0 \leq \alpha \leq 1$ such that the Radon-Nikodym derivative $g = \dd\pp{P}_{X_1}/\dd\pp{P}_{X_0}$ satisfies $g \in {\rm Range}(T^\alpha)$.
\end{assumption}

\begin{remark} \rm
\begin{itemize}
\item
Assumption \ref{as:range-assumption-g} quantifies the smoothness of the Radon-Nikodym derivative $g = \dd\pp{P}_{X_1}/\dd\pp{P}_{X_0}$ by the constant $0 \leq \alpha \leq 1$.
That is, $g$ is smoother if $\alpha$ is close to $1$ and less smooth if $\alpha$ is close to $0$.
Since we have $\dd\pp{P}_{X_1}(x) = g(x) \dd\pp{P}_{X_0}(x)$,  a larger $\alpha$ may therefore be understood as that $\pp{P}_{X_0}$ and $\pp{P}_{X_1}$ are more similar.   
This interpretation can be obtained as follows.
 
\item
The assumption implies that there exists a square-integrable function  $f \in L_2(\pp{P}_{X_0})$ that $g = T^\alpha f$. 
As mentioned, $T$ acts as a smoother, outputting a smoothed version $Tf$ of an input function $f$.
Similarly, its power $T^\alpha$ acts as a smoother, but now $\alpha$ determines the degree of smoothness of the output function. 
For $\alpha = 0$, $T^\alpha$ is just the identity map, and there is no effect of smoothing.
As $\alpha$ increases, the degree of smoothness increases. 
In fact, \citet[Theorem 4.6]{SteSco12} shows that $\mathrm{Range}(T^\alpha)$  for $0 < \alpha \leq 1/2$ is equal to an interpolation space between $L_2(\pp{P}_{X_0})$ and the RKHS $\hbspace$ as a set of functions.
In particular, we have ${\rm Range}(T^{\alpha}) = \hbspace$ for $\alpha = 1/2$, and thus the assumption implies $g \in \hbspace$.
Thus, the case $\alpha > 1/2$ is that $g$ is smoother than the least smooth functions in $\hbspace$.
\end{itemize}
\end{remark}

We next state a smoothness assumption about the function $\theta(\x,\x')$ defined in \eqref{eq:theta-def}.
To simplify the presentation, let $\mathcal{X}^2 := \mathcal{X}\times\mathcal{X}$.
Define a kernel on $\mathcal{X}^2$ as the product kernel $k_{\rm prod}: \mathcal{X}^2 \times \mathcal{X}^2 \to \mathbb{R}$ such that
\begin{equation*} \label{eq:prod-kernel}
k_{\rm prod} \left( (\x_1,\x_2), (\tilde{\x}_1, \tilde{\x}_2) \right) := k(\x_1,\tilde{\x}_1)k(\x_2,\tilde{\x}_2), \quad (\x_1,\x_2), (\tilde{\x}_1,\tilde{\x}_2) \in \mathcal{X}^2.
\end{equation*}
We then define an integral operator $T_{\rm prod} : L_2(\pp{P}_{X_0} \otimes \pp{P}_{X_0}) \to  L_2(\pp{P}_{X_0} \otimes \pp{P}_{X_0})$ by \begin{equation*} \label{eq:int-op-joint}
T_{\rm prod} \eta := \int k_{\rm prod} \left(\cdot, (\tilde{\x}_1, \tilde{\x}_2) \right) \eta\left((\tilde{\x}_1, \tilde{\x}_2)\right) \dd \left(\pp{P}_{X_0} \otimes \pp{P}_{X_0} \right)\left((\tilde{\x}_1, \tilde{\x}_2) \right), \quad \eta \in L_2(\pp{P}_{X_0} \otimes \pp{P}_{X_0}).
\end{equation*}
By Assumption  \ref{as:g-and-theta} \ref{as:measruable} and \ref{as:RKHS-dense}, this can be written in terms of the eigensystem $(\mu_i, e_i)_{i \in I}$ as 
$$
T_{\rm prod} \eta = \sum_{i,j \in I} \mu_i \mu_j \left< \eta, e_i \otimes e_j  \right>_{L_2(\pp{P}_{X_0} \otimes \pp{P}_{X_0})} e_i  \otimes e_j,
$$
where $e_i \otimes e_j: \mathcal{X}^2 \to \mathbb{R}$ denotes the tensor product of $e_i$ and $e_j$, and the convergence is in $L_2(\Pz \otimes \Pz)$; see Lemma \ref{lemma:eig-decomp-joint-int-op} in Appendix \ref{sec:theory-preliminary}.
That is, each $e_i \otimes e_j$ is an eigenfunction of $T_{\rm prod}$ with the corresponding eigenvalue $\mu_i \mu_j$. 
The $\beta$-th power of $T_{\rm prod}$ for $0 \leq \beta \leq 1$ is then defined as
\begin{equation} \label{eq:power-joint-int-op}
T_{\rm prod}^\beta \eta = \sum_{i,j \in I} (\mu_i \mu_j)^\beta \left< \eta, e_i \otimes e_j \right>_{L_2(\pp{P}_{X_0} \otimes \pp{P}_{X_0})} e_i \otimes e_j.    
\end{equation}
Similar to Assumption \ref{as:range-assumption-g}, we make the following smoothness assumption for the function $\theta: \mathcal{X} \times \mathcal{X} \to \mathbb{R}$ defined in \eqref{eq:theta-def}, based on the range of the power $T_{\rm prod}^\beta$.
\begin{assumption} \label{as:range-assumption-theta}
There exists a constant $0 \leq \beta \leq 1$ such that the function $\theta$ defined in \eqref{eq:theta-def} satisfies $\theta \in {\rm Range}(T_{\rm prod}^\beta)$.
\end{assumption}

\begin{remark} \rm
As for Assumption \ref{as:range-assumption-g}, we can interpret Assumption \ref{as:range-assumption-theta} as quantifying the smoothness of $\theta$ by the constant $\beta$.
That is, larger $\beta$ implies that $\theta$ is smoother.
Note that $\theta$ can be written as $\theta(\x, \x') = \left<\mu_{Y_0|X_0 = \x}, \mu_{Y_0|X_0 = \x'} \right>_{\hbspf}$, where $\mu_{Y_0|X_0 = \x} :=  \int \ell(\cdot,\y) \dd\pp{P}_{Y_0|X_0}(\y|\x)$ is the kernel mean of $\pp{P}_{Y_0|X_0}(\cdot|\x)$. 
Therefore, $\theta$ is smooth if the mapping $\x \to  \mu_{Y_0|X_0 = \x}$ is smooth.
Thus, $\beta$ may be interpreted as quantifying the smoothness of this mapping.
\end{remark}

We are now ready to state Theorem \ref{theo:convergence-rate} below, which establishes the convergence rate of the CME estimator. 
The rate is given in terms of the constants $\alpha$ and $\beta$ introduced in the above assumptions. 
The proof is given in Appendix \ref{sec:proof-convergence-rate}.

\begin{theorem}[Convergence rates] \label{theo:convergence-rate}
Suppose that Assumptions \ref{as:g-and-theta}, \ref{as:range-assumption-g} and \ref{as:range-assumption-theta} hold with $ \alpha + \beta \leq 1$.
Let $\hat{\mu}_{Y\langle 0|1 \rangle}$ be the estimator defined in \eqref{eq:empirical-cme} with a regularization constant $\varepsilon_n > 0$.
Let $c > 0$ be an arbitrary constant, and set $\varepsilon_n = c n^{- 1 / \left(1 + \beta + \max (1 - \alpha, \alpha) \right)  }$.
Then we have
\begin{equation*}
\left\| \hat{\mu}_{Y\langle 0|1 \rangle} - \mu_{Y\langle 0|1 \rangle} \right\|_{\hbspf} = 
O_p\left( n^{ - (\alpha + \beta) /  2 (1 + \beta + \max(1-\alpha, \alpha))   } \right) \quad (n \to \infty).
\end{equation*}
\end{theorem}


\begin{remark} \rm
Let us interpret the rate of Theorem \ref{theo:convergence-rate}.
\begin{itemize}
    \item 
The exponent $ (\alpha + \beta) /  2 (1 + \beta + \max(1-\alpha, \alpha))$ in the rate is smaller than $1/2$ for any $\alpha$ and $\beta$, and thus the rate is always slower than the parametric rate $n^{-1/2}$.
For instance, the rate becomes $n^{-1/4}$ if $\alpha = \beta = 1/2$.
This is due to the CME estimator being nonparametric, as for other nonparametric statistical estimators in general \citep{Tsybakov08:INE}. 
We are not aware of, however, whether the obtained rate is minimax optimal. 
We leave this question for future research.

\item  An important interpretation of the rate is as follows:
if {\em either} the Radon-Nikodym  derivative $g = \dd\pp{P}_{X_1}/\dd\pp{P}_{X_0}$ {\em or} the function $\theta$ is smooth, then the CME estimator converges reasonably fast.
For instance, the rate becomes $n^{-1/6}$ if $\alpha = 0$ and $\beta = 1$, and $n^{-1/4}$ if $\alpha=1$ and $\beta = 0$ (recall that $\alpha$ and $\beta$ quantify the smoothness of $g$ and $\theta$, respectively).
Therefore, even in the situation where the change from $\pp{P}_{X_1}$ to $\pp{P}_{X_0}$ is large, we may still expect a good performance for the CME estimator if the relationship between $X_0$ and $Y_0$ is smooth (and vice versa).   
\end{itemize}

\end{remark}

As for Corollary \ref{coro:consistency}, we obtain the following corollary from Theorem \ref{theo:convergence-rate}.
\begin{corollary}
 \label{coro:convergence-rate}
Suppose that Assumptions \ref{asp:main-asp}, \ref{as:g-and-theta}, \ref{as:range-assumption-g} and \ref{as:range-assumption-theta} hold with $ \alpha + \beta \leq 1$. 
Let $\hat{\mu}_{Y\langle 0|1 \rangle}$ be the estimator defined in \eqref{eq:empirical-cme} with a regularization constant $\varepsilon_n > 0$.
Let $c > 0$ be an arbitrary constant, and set $\varepsilon_n = c n^{- 1 / \left(1 + \beta + \max (1 - \alpha, \alpha) \right)  }$.
Then we have
\begin{equation*}
\left\| \hat{\mu}_{Y\langle 0|1 \rangle} -\mu_{Y_0^* | T = 1}  \right\|_{\hbspf} = 
O_p\left( n^{ - (\alpha + \beta) /  2 (1 + \beta + \max(1-\alpha, \alpha))   } \right) \quad (n \to \infty).
\end{equation*}
\end{corollary}

\section{Applications to Sampling and Testing}
\label{sec:applications}

In this section, we discuss important applications of the proposed framework.

\subsection{Sampling from Counterfactual Distributions}

\begin{algorithm}[t]
\caption{Sampling from a counterfactual mean embedding estimate}
\label{alg:sampling-CME}
\begin{algorithmic}[1]
\STATE \textbf{Input}: A CME estimate $\hat{\mu}_{Y\langle 0|1 \rangle} = \sum_{i=1}^n\beta_i\ell(\y_i,\cdot)$ with $(\beta_i,\y_i)_{i=1}^n \subset \mathbb{R} \times \mathcal{Y}$ and kernel $\ell: \mathcal{Y} \times \mathcal{Y} \to \mathbb{R}$; the number $m \in \mathbb{N}$ of sample points to generate.
\STATE Compute $\tilde{\y}_{1} := \arg\max_{\y \in \mathcal{Y}}
\sum_{i=1}^{n} \beta_i \ell(\y_i,\y)$.
\FOR{$t = 2$ to $m$}
\STATE Compute $\tilde{\y}_{t} := \arg\max_{\y \in \mathcal{Y}} \sum_{i=1}^{n} \beta_i \ell(\y_i,\y)   -  \dfrac {1}{t}\sum^{t-1}_{i=1}\ell(\tilde{\y}_{i}, \y)$. 
\ENDFOR
\STATE \textbf{Output}: $\tilde{\y}_1,\dots,\tilde{\y}_m$.
\end{algorithmic}
\end{algorithm}
   
While a CME estimate can be seen as a weighted sample $(\beta_i,\y_i)_{i=1}^n$, the coefficients $\beta_1,\dots,\beta_n$ may in general include negative values; thus it is not straightforward to interpret them as importance weights.
If one can generate sample points $\tilde{\y}_1,\dots,\tilde{\y}_m$ from the CME estimate, then these unweighted points may be more useful for an analyst. 
For instance, we might use them for the purpose of visualization (\eg, scatter plot and histogram).
Moreover, as will be described below, such unweighted points can be straightforwardly used for testing hypotheses regarding distributional treatment effects.

We propose a method for sampling from the counterfactual distribution based on the CME estimator and the kernel herding algorithm \citep{CheWelSmo10,KanNisGreFuk16,Kajihara_ICML2018}.
The method is summarized in Algorithm \ref{alg:sampling-CME}, which generates sample points $\tilde{\y}_1,\dots,\tilde{\y}_m$ from $\hat{\mu}_{Y\langle 0|1 \rangle}$ in \eqref{eq:empirical-cme}.
When the kernel $\ell$ is shift-invariant  (\eg, Gaussian), the procedure in Algorithm \ref{alg:sampling-CME} to generate $\tilde{\y}_1,\dots,\tilde{\y}_t$ for $t = 1,\dots,m \in \mathbb{N}$ is equivalent to the greedy minimization of the RKHS distance between the CME estimate $\hat{\mu}_{Y\langle 0|1 \rangle}$ and the empirical kernel mean $ \frac{1}{t} \sum_{i=1}^t \ell(\tilde{\y}_i, \cdot)$:
\begin{equation}\label{eq:kherding}
 \left\| \hat{\mu}_{Y\langle 0|1 \rangle}  - \frac{1}{t} \sum_{i=1}^t \ell(\tilde{\y}_i, \cdot)  \right\|_\hbspf = \sup_{ \|f\|_\hbspf \leq 1 } \left| \sum_{i=1}^n \beta_i f(\y_i) - \frac{1}{t} \sum_{j=1}^t f(\tilde{\y}_j) \right|.
\end{equation}
See \citet{CheWelSmo10} for details.
In other words, the points $\tilde{\y}_1,\dots,\tilde{\y}_m$ are those greedily minimizing the worst case error to the weighted points $(\beta_i, \y_i)_{i=1}^n$ in the unit ball of the RKHS $\hbspf$; thus, this algorithm is a greedy variant of Quasi Monte Carlo methods  \citep{DicKuoSlo13}.
Notice that therefore these points are, of course, not independent to each other.
The convergence rate $O(n^{-1/2})$ is guaranteed for $\frac{1}{t} \sum_{i=1}^t \ell(\cdot,\tilde{\y}_i)$ \citep{BacJulObo12}, which may hold even when the optimization problem \eqref{eq:kherding} is solved approximately \citep{LacLinBac15,KanNisGreFuk16}.

Lastly, to obtain high dimensional samples, \eg, images, from the counterfactual distribution, one can train deep generative models using MMD-GAN \citep{li2015generative,dziugaite2015training,sutherland2017generative,li2017mmd} based the CME estimate. We defer this promising application to future work.


\subsection{Counterfactual Inference as Two-sample Testing}
\label{sec:causal-two-sample}
One can also identify distributional treatment effects by formulating the problem as that of hypothesis testing, or more specifically, two-sample testing. 
To describe this, we assume that we are given data $(\x_i,t_i,\y_i)_{i=1}^N$, which are i.i.d.~with random variables $(X,T,Y)$ with $Y = Y_0^* \mathbbm{1}(T=0) + Y_1^* \mathbbm{1}(T=1)$ being the observed outcome and $Y_0^*, Y_1^*$ being the potential outcomes.  Let $n := \sum_{i=1}^N (1-t_i)$ and $m := \sum_{i=1}^N t_i$.
See Section \ref{sec:kte} for the notation.

\paragraph{Distributional treatment effects.}
Here we are interested in testing whether $ \pp{P}_{Y_0^*}$ and $\pp{P}_{Y_1^*}$ are equal or not (see Section \ref{sec:DATE}).
The null hypothesis $H_0$ and the alternative hypothesis $H_1$ are thus defined as
\begin{equation*}
H_0: \pp{P}_{Y_0^*}  = \pp{P}_{Y_1^*}, \quad
H_1: \pp{P}_{Y_0^*}  \neq \pp{P}_{Y_1^* },
\end{equation*}
As a test statistic, we propose to use an estimate $\widehat{\text{KTE}} (Y_0^* ,Y_1^*,\hbspf)$ of the kernel treatment effect,  $\text{KTE}(Y_0^*,Y_1^*,\hbspf) = \| \mu_{Y_0^*} - \mu_{Y_1^*} \|^2_\hbspf$, introduced in Section \ref{sec:KTE-DATE}. 
This estimate $\widehat{\text{KTE}} (Y_0^* ,Y_1^*,\hbspf)$, computed from the data $(\x_i,t_i,\y_i)_{i=1}^N$,  can either be the biased one, $\widehat{\text{KTE}}_b (Y_0^* ,Y_1^*,\hbspf)$, defined in \eqref{eq:emp-kte-po}, or the unbiased one, $\widehat{\text{KTE}}_u (Y_0^* ,Y_1^*,\hbspf)$, in \eqref{eq:unbiased-KTE-DATE}.



To decide a critical region, we need the distribution of the test statistic under the null hypothesis $H_0$. One way to approximate this distribution is to use a bootstrap procedure \citep{EfroTibs93:BS}, as follows. 
Let $B \in \mathbb{N}$ be the number of bootstrap samples.
For each $b = 1,\dots, B$, we randomly permute the indices $1,\dots, N$ to, say, $\pi_b(1), \dots, \pi_b(N) \subset \{1,\dots,N\}$. 
Then we compute the test statistic $\eta_b :=  \widehat{\text{KTE}} (Y_0^* ,Y_1^*,\hbspf)$ based on the permuted data $(\x_i, t_{ \pi_b(i) }, \y_i)_{i=1}^N$.
We then approximate the null distribution by the histogram of $\eta_1,\dots,\eta_B$, and determine a critical tail region for rejecting the null hypothesis (\eg, with significance level $\alpha = 0.05$).


\paragraph{Distributional effects of the covariate distributions.}  
Here, we are interested in whether the two distributions $\pp{P}_{Y_0^* | T } (\cdot\,|\,0)$ and $\pp{P}_{Y_0^* | T } (\cdot\,|\,1)$ are equal or not (see Section \ref{sec:effects_by_treatment_assignment}).
If they are different, there is a distributional effect on the outcomes arising from the difference in covariate distributions.
The identification of such an effect can be phrased as a hypothesis test with the null and alternative hypotheses being
\begin{equation} \label{eq:hypothesis-effects-treatment-assignment}
H_0: \pp{P}_{Y_0^* | T } (\cdot \,|\, 0) = \pp{P}_{Y_0^* | T } (\cdot \,|\, 1), \quad
H_1: \pp{P}_{Y_0^* | T } (\cdot \,|\, 0) \neq \pp{P}_{Y_0^* | T } (\cdot \,|\, 1), 
\end{equation}
To describe the approach, we rearrange the data $(\x_i, t_i, \y_i)_{i=1}^N$ so that $t_i = 0$ for $i = 1,\dots,n$ and $t_i = 1$ for $i = n + 1, \dots, n + m = N$.   
Note that $\y_1,\dots,\y_n$ are a sample from $\pp{P}_{Y_0^* | T } (\cdot \,|\,0)$, while the distribution $\pp{P}_{Y_0^* | T } (\cdot \,|\,1)$ is counterfactual and we do not have a sample from it.
 However, we can estimate the kernel mean of  $\pp{P}_{Y_0^* | T } (\cdot \,|\, 1)$ by the CME estimator \eqref{eq:empirical-cme} using the data $(\x_i,\y_i)_{i=1}^n$ and  $(\tilde{\x}_j )_{j=1}^m := (\x_{ n+j })_{j= 1}^m$ under the conditional exogeneity assumption, and let $\hat{\mu}_{\left< 0 | 1\right>}$ be  the resulting estimate. 
 We then apply kernel herding to $\hat{\mu}_{\left< 0 | 1\right>}$ for obtaining sample points $\tilde{\y}_1,\dots,\tilde{\y}_m$ approximating $\pp{P}_{Y_0^* | T } (\cdot \,|\,1)$,  as described in Algorithm \ref{alg:sampling-CME}.
 

We can then apply any method for two-sample test, \eg, those in \citet{Gretton12:KTT}, to the two samples $\y_1,\dots,\y_n$ and $\tilde{\y}_1,\dots,\tilde{\y}_m$ to test the hypotheses \eqref{eq:hypothesis-effects-treatment-assignment}.
We note that this is a rather heuristic approach, since $\tilde{\y}_1,\dots,\tilde{\y}_m$ are not drawn from $\pp{P}_{Y_0^* | T } (\cdot \,|\, 1)$, but generated deterministically so as to approximate $\pp{P}_{Y_0^* | T } (\cdot \,|\, 1)$.
We leave a further theoretical study regarding the validity of this approach for future research. 

Finally, one can develop a testing procedure for identifying distributional treatment effects on the treated (see Section \ref{sec:effects_by_treatment}) in the same way as described here, and thus we omit the explanation.

\subsection{Discussion}

Here we discuss how the above sampling and testing procedures may be used in practice.
Suppose that an analyst is interested in the distributional effects of the difference in covariate distributions, \ie, the hypotheses in \eqref{eq:hypothesis-effects-treatment-assignment}, and that the null hypothesis has been rejected as a result of applying the above testing procedure. This suggests the existence of a difference in the two distributions, $\pp{P}_{Y_0^* | T } (\cdot\,|\,0)$ and $\pp{P}_{Y_0^* | T } (\cdot\,|\,1)$. 
This is usually {\em not} the end of the analysis, but is the {\em starting point} of a further exploratory analysis. There are several ways to proceed, to understand how the two outcome distributions differ.

One way is to use the samples $\y_1,\dots,\y_n$ and $\tilde{\y}_1,\dots,\tilde{\y}_m$ (the latter being a counterfactual sample generated from Algorithm \ref{alg:sampling-CME}) approximating the distributions $\pp{P}_{Y_0^* | T } (\cdot\,|\,0)$ and $\pp{P}_{Y_0^* | T } (\cdot\,|\,1)$, respectively. The analyst can use any available statistical method for finding the source of the difference in the two distributions. 
For instance, she may compute summary statistics of both samples (e.g., mean, variance, etc.) and compare them. It is also possible to just plot both samples, or to estimate the densities, to visualize the difference (as we demonstrate in Section \ref{sec:simulations}). 
Another useful method in this context is the approach of \citet{Jitkrittum_NIPS2016} and their follow-up works, which returns interpretable features for explaining the difference in the two samples, such as the sample locations on which (smoothed version) of the two density values differ substantially.

Another important point for discussion is the use of non-characteristic kernels.
If the kernel $\ell$ is not characteristic, such as polynomial kernels, rejecting the null hypothesis implies that there exist a {\em certain kind} of difference in the two distributions. For instance, if the kernel $\ell$ is a polynomial kernel of order $2$, then rejecting the null hypothesis implies that there exists a difference in the mean or in the variance of the two outcome distributions. In this sense, if one is interested in the existence of a specific difference in the outcome distributions (such as the mean and variance), non-characteristic kernels may be more useful.

\section{Application to Off-Policy Evaluation (OPE)}
\label{sec:policy-evaluation}

We describe here how our approach can be applied to the  \emph{off-policy evaluation} task (OPE), \eg, \citet{Dudik11:DoublyRobust}, which aims at evaluating the performance of a given target policy of deciding a certain action given a context. 
The performance is measured in terms of the resulting rewards. 
For instance, consider a recommendation system, where an action is a list of items to be recommended to a user, and a policy determines which action to take, given the features of the user. There will be a positive reward if the user clicks or buys one of the recommended items, and no reward otherwise. 
The goal of OPE is to estimate how a given policy would work, without actually implementing the policy. 
Instead, the evaluation is to be done relying only on logged (or historical) data obtained from a possibly unknown initial policy, which is different from the target policy. 
This task is important when actually implementing a new policy is expensive or difficult, with wide applications including ad placement, recommendation systems, and health care.

We first describe the OPE problem more formally in Section \ref{sec:OPE-problem-description}. We then interpret it with the potential outcome framework and formulate the OPE problem as an estimation of a counterfactual distribution in Section \ref{sec:OPE-counterfactual}.
Finally, we present a concrete algorithm in Section \ref{sec:OPE-algorithm}.

\subsection{Problem Description} \label{sec:OPE-problem-description}
Formally, the OPE task may be defined as follows.
Let $\mathcal{U}$ be a space of context features, $\mathcal{A}$ be a space of actions and $\mathcal{R}$ be a space of rewards. 
For instance, in the case of recommendation systems, each $\mathbf{u} \in\mathcal{U}$ represents a user's features,  $\mathbf{a} \in \mathcal{A}$ a recommendation (\eg, a list of items), and $r \in \mathcal{R}$ the number of clicks on the recommendation. 
A policy $\pi(\mathbf{a} | \mathbf{u})$ is a conditional distribution on the action space $\mathcal{A}$ given context features $\mathbf{u} \in\mathcal{U}$. 
In a recommendation system, $\pi(\mathbf{a} | \mathbf{u})$ determines the probability of providing a recommendation $\mathbf{a}$ to a user having features $\mathbf{u}$.

Assume that tuples $\{(\mathbf{u}_i,\bfa_i,r_i)\}_{i=1}^n \subset \mathcal{U} \times \mathcal{A} \times \mathcal{R}$ of context features ${\bf u}_i \in \mathcal{U}$, action ${\bf a}_i \in \mathcal{A}$ and reward $r_i \in \mathcal{R}$ are available as {\em logged (or historical) data}. We assume that they were independently generated from a joint distribution $\pp{P}_0(\mathbf{u}, \mathbf{a}, r) := q_0(\mathbf{u}) \pi_0(\mathbf{a}|\mathbf{u}) \pp{P}_{0}(r|\mathbf{u},\mathbf{a})$ in the data collection phase, where $q_0({\bf u})$ is a marginal distribution on $\mathcal{U}$, $\pi_0(\mathbf{a}|\mathbf{u})$ is an {\em initial (or logging/behavior) policy}, and  $\pp{P}_0(r\,|\,\mathbf{u},\mathbf{a})$ is a conditional distribution of a reward $r \in \mathcal{R}$ given $(\mathbf{u}, \mathbf{a}) \in \mathcal{U} \times \mathcal{A}$.  
In a recommendation system, for instance, $\pp{P}_0(r\,|\,\mathbf{u},\mathbf{a})$ describes whether a user with features $\mathbf{u}$ who has been recommended a list of items $\mathbf{a}$ would choose one of the items. 
As such, it is typically unknown a priori. Similarly, $q_0(\mathbf{u})$ and $\pi_0(\mathbf{a}|\mathbf{u})$ may be unknown in practice, if $\{(\mathbf{u}_i,\bfa_i,r_i)\}_{i=1}^n$ are given as historical data. 

Let $\pi_*(\mathbf{a}|\mathbf{u})$ be another conditional distribution of actions ${\bf a} \in \mathcal{A}$ given context features ${\bf u} \in \mathcal{U}$, which represents the {\em target policy} that one wants to evaluate. By design, the target policy is known and sampling from it is possible. Let $q_*({\bf u})$ be a probability distribution on $\mathcal{U}$, which represents the distribution of context features under the target environment (\eg, the distribution of user features when a recommendation system is deployed). In the standard OPE setting, it is typically assumed that $q_0({\bf u}) = q_*({\bf u})$, i.e., the historical and target environments are the same; but in general these can be different, $q_0({\bf u}) \not= q_*({\bf u})$. The latter situation is not uncommon in practice and has been recently studied by \citet{uehara2020off}.
Finally, let $\pp{P}_*(r\,|\,\mathbf{u},\mathbf{a})$ be the conditional distribution of a reward $r$ given context features ${\bf u}$ and action ${\bf a}$ under the target environment. We assume that this remains the same as in the data collection phase, \ie, 
\begin{equation*}  
\pp{P}_*(r\,|\,\mathbf{u},\mathbf{a}) =  \pp{P}_0 (r\,|\,\mathbf{u},\mathbf{a}).    
\end{equation*}
This assumption may be understood as the {\em policy invariance} assumption commonly made in the econometric policy evaluation literature, \eg, \citet{heckman2007econometric}.\footnote{More precisely, this assumption may be identified with the policy invariance assumptions PI-1 and PI-2 in Section 2.2 of \citet{heckman2007econometric}, where $s$ and $\omega$ there correspond to ${\bf a}$ and ${\bf u}$ in our setting, respectively, and tuple $(a, b, \tau)$ there essentially corresponds to a policy in our setting.}

The task of off-policy evaluation is then to estimate the expected reward under the target environment: 
\begin{align}
R_* :=\int_{\mathcal{U} \times \mathcal{A}}  \int_{\mathcal{R}} r \dd \pp{P}_*(r|{\bf u}, {\bf a}) \dd\pi_*( {\bf u}, {\bf a} ) = \int_{\mathcal{U} \times \mathcal{A}}  \int_{\mathcal{R}} r \dd \pp{P}_0(r|{\bf u}, {\bf a}) \dd\pi_*( {\bf u}, {\bf a} ), \label{eq:OPE-expected-reward}
\end{align}
where the identity follows from the assumption $ \pp{P}_*(r\,|\,\mathbf{u},\mathbf{a}) =  \pp{P}_0 (r\,|\,\mathbf{u},\mathbf{a})$, and $\pi_*({\bf u}, {\bf a}) := \pi_*({\bf a}| {\bf u}) q_*({\bf u})$ is the joint distribution on $\mathcal{U} \times \mathcal{A}$ given by $ \pi_*({\bf a}| {\bf u})$ and $q_*({\bf u})$.
This estimation is to be done using logged data $\{(\mathbf{u}_i,\bfa_i,r_i)\}_{i=1}^n$ and  the target policy $\pi_*({\bf u}|{\bf a})$.

\subsection{OPE as Counterfactual Inference} 
\label{sec:OPE-counterfactual}

We explain below how our CME estimator \eqref{eq:empirical-cme} can be applied to the OPE task. To this end, we consider the marginal distribution of a reward under the target environment:
\begin{equation} \label{eq:marginal-rewards}
\pp{P}_*(r) := \int_{\mathcal{U} \times \mathcal{A}} \pp{P}_*(r\,|\,\mathbf{u},\mathbf{a})  \dd\pi_*(\mathbf{u},\mathbf{a}) = \int_{\mathcal{U} \times \mathcal{A}} \pp{P}_0(r\,|\, \mathbf{u},\mathbf{a}) \dd\pi_*(\mathbf{u},\mathbf{a}).
\end{equation} 
Note that the expected reward \eqref{eq:OPE-expected-reward} is the mean of this distribution.
We first show that the distribution $\pp{P}_*(r)$ can be interpreted as a counterfactual distribution, by formulating the OPE task using the potential outcome framework\footnote{Our formulation of the OPE task using the potential outcome framework is different from the existing formulation, \eg, \citet{KallusZ18:Continuous}, where each action ${\bf a} \in \mathcal{A}$ is defined as a treatment,  and for each action ${\bf a}$ there is a corresponding potential outcome $Y_{\bf a}^*$. In our formulation, on the other hand, an action ${\bf a}$ taken for the subject is defined as a part of covariates $\x = ({\bf u}, {\bf a})$, and binary treatments, $0$ and $1$, are considered, each of which is defined as exposing the subject to a certain environment. Our formulation enables us to interpret the OPE task as counterfactual inference of changing the covariate distribution, so that our CME estimator can be naturally applied. Thus, our motivation of introducing this formulation is rather pragmatic, and we do not argue whether it is more reasonable than the existing one.  One benefit may exist, however: In our formulation, we explicitly model the assumption on the conditional distributions of rewards being the same for the data collection and evaluation phases via the potential outcome notation \eqref{eq:OPE-conditional-assumption}, while this assumption is implicitly made in the existing formulation. This explicit statement of the assumption helps a researcher to understand when the OPE may be justified. } in Section \ref{sec:cme}.

\paragraph{Random variables.}
Consider a hypothetical subject in population. This subject is associated with covariates $X := (U, A)$, where $U \in \mathcal{U}$ is context features and $A \in \mathcal{A}$ is an action taken. As such, we define the covariate space as $\mathcal{X} := \mathcal{U} \times \mathcal{A}$, the product of the context feature space $\mathcal{U}$ and action space $\mathcal{A}$. In a recommendation system, for instance, $U$ is the user features and $A$ is a recommended list of items. 
We define two treatments $0$ and $1$ as exposing the subject to {\em the environment during the data collection phase} and {\em that during the evaluation phase}, respectively; and the associated potential outcomes $Y_0^*$ and $Y_1^*$ as the {\em rewards} under the respective treatments $0$ and $1$. 
Let $T \in \{0, 1\}$ be a treatment indicator. 

In a recommendation system, for example, an environment may refer to the situation where a user is about to choose an item, such as the calendar year when this takes place. For instance, treatment $0$ may refer to the environment in the year 2000, and treatment 1 the environment in year 2020. Consider a user with the features $U$ and the recommended items $A$, and suppose that $A$ consists of items which were popular in 2000 but are outdated in 2020. Then this user may have chosen an item from $A$ if it were in 2000, but may not choose any item in 2020, \ie, we may have $Y_0^* \not= Y_1^*$.




For ease of understanding, consider a finite population of $N$ subjects with 
\begin{equation} \label{eq:OPE-finite-sample}
( \y_{i0}^*,  \y_{i1}^*, \x_i, t_i  )_{i=1}^N
\end{equation}
being i.i.d.~realizations of the random variables $(Y_0^*, Y_1^*, X, T)$.
That is, the $i$-th subject is associated with covariates $\x_i  := ({\bf u}_i, {\bf a}_i)$ consisting of context features ${\bf u}_i$ and action ${\bf a}_i$. The treatment assignment $t_i \in \{0,1\}$ indicates which environment the $i$-th subject is exposed to. 
Thus, the potential outcomes $\y_{i0}^*$ and $\y_{i1}^*$ are the rewards from the $i$-th subject (associated with $\x_i  := ({\bf u}_i, {\bf a}_i)$) that would have been observed if she was exposed to the environment during the data collection phase (treatment $0$) and that during the evaluation phase (treatment $1$), respectively.

\paragraph{Distributions of the potential outcomes.}
The distributions of the potential outcomes $Y_0^*, Y_1^*$ are defined via their conditional distributions given $X = (U,A)$, \ie, $\pp{P}_{Y_0^* | X}(\y|\x)$ and $\pp{P}_{Y_1^* | X}(\y|\x)$ where $\y = r$ and $\x = ({\bf u}, {\bf a})$.
These are the conditional distributions of rewards $r \in \mathcal{R}$ given context features ${\bf u} \in \mathcal{U}$ and action ${\bf a} \in \mathcal{A}$, under the environment during the data collection phase (treatment $0$) and that during the evaluation phase (treatment $1$), respectively:
$$
\pp{P}_{Y_0^* | X} (\y | \x) = \pp{P}_{0}(r | {\bf u}, {\bf a} ), \quad \pp{P}_{Y_1^* | X} (\y | \x) = \pp{P}_{*}(r | {\bf u}, {\bf a} ) \quad (\y = r,\ \ \x = ({\bf u}, {\bf a})).
$$
Note that our assumption $\pp{P}_{0}(r | {\bf u}, {\bf a} ) = \pp{P}_{*}(r | {\bf u}, {\bf a} )$ implies that 
\begin{equation} \label{eq:OPE-conditional-assumption}
    \pp{P}_{Y_0^* | X} (\y | \x) = \pp{P}_{Y_1^* | X} (\y | \x),
\end{equation}
\ie,  the conditional distributions of the potential outcomes (rewards) $Y_0^*$ and $Y_1^*$ are the same, given the covariates $X = \x := ({\bf u}, {\bf a})$.

For instance, consider a recommendation system with a finite population \eqref{eq:OPE-finite-sample} with the $i$-th user equipped with covariates $\x_i = ({\bf u}_i, {\bf a}_i)$ consisting of features ${\bf u}_i$ and recommended items ${\bf a}_i$.
The identity \eqref{eq:OPE-conditional-assumption} then implies that, for the $i$-th user, the distributions of the potential outcomes (rewards) $\y_{i0}^*$ and $\y_{i1}^*$ are the same. This means that this user should have the same stochastic behavior in choosing (or not choosing) an item from the recommended ones during the data collection (treatment $0$) and evaluation (treatment $1$) phases. In other words, the environmental factors that affect the user behavior should be the same for the data collection and evaluation phases. 
This excludes, for instance, the above example where the data collection phase is in year 2000 and the evaluation phase is in 2020, in which case user preferences are different.

\paragraph{Distributions of covariates.}
As in Section \ref{sec:policy-spec-cov-dist}, we interpret here a policy as specifying a distribution on the covariate space $\mathcal{X} = \mathcal{U} \times \mathcal{R}$. More specifically, consider the conditioning $T = 0$ or $T = 1$, which imply that the hypothetical subject is exposed to the environment of the data collection phase ($T = 0$) or to that of the evaluation phase ($T=1$). Then the corresponding distributions of the covariates $X_0 = X|(T=0)$ and $X_1 = X|T=1$ are given by the logging policy $\pi_0({\bf u}| {\bf a})$ and the target policy $\pi_*({\bf u} | {\bf a})$, respectively:
\begin{equation} \label{eq:OPE-identities-marginal}
\pp{P}_{X_0}(\x) = \pi_{0} (\mathbf{a} | \mathbf{u} ) q_0 (\mathbf{u}), \quad \pp{P}_{X_1}(\x) = \pi_* (\mathbf{a} | \mathbf{u} ) q_* (\mathbf{u}) \quad (\x = ({\bf u}, {\bf a}) \in \mathcal{U} \times \mathcal{A}).
\end{equation}

To describe this, consider the finite population \eqref{eq:OPE-finite-sample}, and assume that for the $i$-th user the treatment indicator is $t_i = 0$, which implies that her data are given in the data collection phase. In this case, her covariates $\x_i = ({\bf u}_i, {\bf a}_i)$ are generated according to the joint distribution $\pi_0({\bf a} | {\bf u}) q_0({\bf u})$ involving the logging policy $\pi_0({\bf a} | {\bf u})$. On the other hand, if $t_i = 1$, her covariates $\x_i = ({\bf u}_i, {\bf a}_i)$ are generated from the joint distribution $\pi_*({\bf a} | {\bf u}) q_*({\bf u})$ given by the target policy  $\pi_*({\bf a} | {\bf u})$. 
Note that in the OPE setting, if $t_i = 1$ we have access to neither $\y_{i0}^*$ nor $\y_{1i}^*$; note also that in this case this ``user'' may be imaginary, with covariates $\x_i = ({\bf u}_i, {\bf a}_i)$ generated artificially.

\paragraph{Distributions of observed outcomes.}

The observed outcome (reward) from the hypothetical subject is defined as $Y = \mathbbm{1}(T=0)Y_0^* + \mathbbm{1}(T=1) Y_1^*$.
Let $Y_0 = Y|(T=0) = Y_0^*| (T=0)$ and $Y_1 = Y|(T=1) = Y_1^*| (T=1)$ be the observed outcome $Y$ conditioned on $T=0$ or $T=1$, respectively. Then $Y_0$ conditioned on $X_0 = X|(T=0)$ and $Y_1$ conditioned on $X_1 = X|(T=1)$ can be written in terms of the potential outcomes $Y_0^*$ and $Y_1^*$ as
\begin{equation*}
    Y_0|X_0 = Y_0^* | X, (T =0), \quad  Y_1|X_1 = Y_1^* | X, (T =1). 
\end{equation*}
Note that the potential outcomes $Y_0^*$ and $Y_1^*$ are independent to the treatment indicator $T$ given the covariates $X$  under the conditional exogeneity in Assumption \ref{asp:main-asp}. Therefore, $ Y_0^* | X, (T =0) = Y_0^* | X $ and $Y_1^* | X, (T=1) = Y_1^* | X$.

Thus, we can identify the conditional distributions $\pp{P}_{Y_0|X_0}$ and  $\pp{P}_{Y_1|X_1}$ as  $\pp{P}_0 (r | {\bf u}, {\bf a}) $ and $\pp{P}_* (r | {\bf u}, {\bf a}) $, respectively:
\begin{equation} \label{eq:OPE-identities-conditional}
    \pp{P}_{Y_0|X_0}(\y | \x) = \pp{P}_{Y_0^* | X} (\y | \x) = \pp{P}_0 (r | {\bf u}, {\bf a}) , \quad \pp{P}_{Y_1|X_1}(\y | \x) = \pp{P}_{Y_1^* | X} (\y | \x) = \pp{P}_* (r | {\bf u}, {\bf a}),
\end{equation}
where $\y = r$ and $\x = ({\bf u}, {\bf a})$.
Therefore, the assumption $\pp{P}_0 (r | {\bf u}, {\bf a}) = \pp{P}_* (r | {\bf u}, {\bf a}) $ (or \eqref{eq:OPE-conditional-assumption}) implies that 
$$
\pp{P}_{Y_0|X_0}(\y | \x)  = \pp{P}_{Y_1|X_1}(\y | \x) .
$$

\paragraph{Reward distribution as a counterfactual distribution.} 
Finally, the reward distribution  \eqref{eq:marginal-rewards} under the target environment can be written as a counterfactual distribution using the identities \eqref{eq:OPE-identities-marginal} and \eqref{eq:OPE-identities-conditional} (assuming the conditional exogeneity) as
$$
\pp{P}_*(r) = \int  \pp{P}_{Y_0 |X_0}(\y|\x) \dd \pp{P}_{X_1}(\x) = \pp{P}_{Y\left< 0 | 1 \right>} (\y).
$$
with $\y = r$ and $\x = ({\bf u}, {\bf a})$.
Thus, we can use the CME estimator \eqref{eq:empirical-cme} to estimate the kernel mean of this reward distribution, which is described below.

\subsection{Off-Policy Evaluation by the CME Estimator}
\label{sec:OPE-algorithm}

Let $\ell$ be a kernel on the outcome (reward) space $\mathcal{Y} = \mathcal{R}$ with $\hbspf$ its RKHS.
Then the mean embedding of the reward distribution under the target environment is defined as 
\begin{equation} \label{eq:OPE-embedding}
    \mu_{\pp{P}_*} := \int \ell(\cdot, r) \dd\pp{P}_*(r) \in \hbspf.
\end{equation}
Note that this becomes the expected reward \eqref{eq:OPE-expected-reward} if we define $\ell$ as a linear kernel: $\ell(r, r') := r r'$, which we use in our experiments in Section \ref{sec:OPE-experiments}. In principle, however, the use of other nonlinear kernels (in particular characteristic kernels) makes the mean embedding more informative, and this may be beneficial in assessing the effectiveness of the target policy. 
\paragraph{Kernel on covariates.}
To use the CME estimator, we also need to define a kernel $k$ on the covariate space $\mathcal{X} = \mathcal{U}\times \mathcal{A}$.  To this end, we first define kernels
$k_{\mathcal{U}}$ and $k_{\mathcal{A}}$ on the context feature space $\mathcal{U}$ and the action space $\mathcal{A}$, respectively. 
Then we can define $k$ as the {\em product kernel} of $k_{\mathcal{U}}$ and $k_{\mathcal{A}}$: $k( ( {\bf u}, {\bf a} ), ({\bf u}', {\bf a}') ) := k_{\mathcal{U}} ({\bf u}, {\bf u}') k_{\mathcal{A}} ({\bf a}, {\bf a}')$ for $( {\bf u}, {\bf a} ), ({\bf u}', {\bf a}') \in \mathcal{U} \times \mathcal{A}$.

\paragraph{Joint sample.}
Recall that logged data $\{(\mathbf{u}_i,\bfa_i,r_i)\}_{i=1}^n$ are i.i.d.~with the joint distribution $\pp{P}_0(\mathbf{u}, \mathbf{a}, r) = \pp{P}_{0}(r|\mathbf{u},\mathbf{a}) \pi_0(\mathbf{a}|\mathbf{u}) q_0(\mathbf{u}) $, which is identified as $\pp{P}_{X_0 Y_0} (\x, \y) = \pp{P}_{Y_0|X_0}(\y|\x) \pp{P}_{X_0}(\x)$ for $\x = ({\bf u}, {\bf a})$ and $\y = r$  because of \eqref{eq:OPE-identities-marginal} and  \eqref{eq:OPE-identities-conditional}.
Thus, by defining  $\x_i := ( \mathbf{u}_i, \mathbf{a}_i )$ and  $\y_i := r_i$, we have an i.i.d.~sample $(\x_i, \y_i)_{i=1}^n$ from the joint distribution $\pp{P}_{X_0 Y_0} (\x, \y)$.

\paragraph{Covariate sample.}
We also need to express $\pp{P}_{X_1}$ in terms of a sample in the form $(\x'_j)_{j=1}^m$ . As in \eqref{eq:OPE-identities-marginal}, the covariate distribution $\pp{P}_{X_1}$ is the joint distribution given by the target policy $\pi_*({\bf a}| {\bf  u})$ and the marginal distribution $q_*({\bf u} )$ of context features. Thus, if we can sample from both $\pi_*({\bf a}| {\bf  u})$ and $q_*({\bf u} )$ (the former is typically possible because it is defined by the designer of the target policy, while the latter depends on the problem), then $(\x'_j)_{j=1}^m$ may be given by 
\begin{equation} \label{eq:OPE-sampling}
    \x'_j := ({\bf u}_j^*, {\bf a}_j^*), \quad \text{where}\quad {\bf u}_j^* \sim q_*({\bf u}), \quad {\bf a}_j^* \sim \pi_*( {\bf a} | {\bf u}_j^* ), \quad j = 1,\dots,m.
\end{equation}
In the particular case where $q_*({\bf u} ) = q_0({\bf u} )$, we can use the sample $({\bf u}_i)_{i=1}^n$ in the logged data $\{(\mathbf{u}_i,\bfa_i,r_i)\}_{i=1}^n$, which are from $q_0({\bf u})$, as a sample from $q_*({\bf u})$: $ {\bf u}_j^*  := {\bf u}_j$ for $j = 1,\dots, m := n$. Note that even if $q_*({\bf u}) \not= q_0({\bf u})$, we can use the CME estimator as long as we have a sample of context features $( {\bf u}_j^* )_{j=1}^m$ from the target environment, \ie, the covariate shift setting of \citet{uehara2020off}.

\paragraph{Algorithm.}
The resulting algorithm is described in Algorithm  \ref{alg:kpe}, which only requires matrix operations and thus is simple to implement. 
We note that the expected reward \eqref{eq:OPE-expected-reward} under the target environment can be estimated as $\muh_{\pp{P}_*(r)} = \sum_{i=1}^n \beta_i r_i$; this is obtained by setting $\ell$ as a linear kernel on $\mathbb{R}$. 

\paragraph{Extensions.}
Note that the above method of approximating the covariate distribution $\pp{P}_{X_1}$ via the sampling procedure in \eqref{eq:OPE-sampling} does not fully exploit the information of the target policy $\pi_*({\bf a} | {\bf u})$, since for each ${\bf u}_j^*$ we only sample one action ${\bf a}_j^*\sim \pi_*({\bf a} | {\bf u}_j^*)$. 
In Appendix \ref{sec:OPE-extension}, we discuss extensions of Algorithm \ref{alg:kpe} to make use of more information from the target policy.


\begin{algorithm}[t]
\caption{Off-Policy Evaluation using the CME estimator \eqref{eq:empirical-cme}}
\label{alg:kpe}
\begin{algorithmic}[1]
\STATE \textbf{Requirement}: A kernel $k_{\mathcal{U}}$ on the context  space $\mathcal{U}$, a kernel $k_{\mathcal{A}}$ on the action space $\mathcal{A}$, a kernel $\ell$ on the reward space $\mathcal{R}$, and a regularization constant $\varepsilon > 0$. 
\STATE \textbf{Input:} Logged data $({\bf u}_i, {\bf a}_i, r_i)_{i=1}^n$, a target policy $\pi_*({\bf u} | {\bf a})$ and a sample of context features $({\bf u}_j^*)_{j=1}^m$. (If $q_*({\bf u}) = q_0({\bf u})$, set ${\bf u}_j^* := {\bf u}_j,  j =1,\dots,m := n$.) 

\FOR{ $j=1$ \TO $n$ } 
\STATE  $\mathbf{a}_j^* \sim \pi_*(\mathbf{a}\,|\,\mathbf{u}_j^*)$ 
\ENDFOR
\STATE Compute ${\bf K} \in \mathbb{R}^{n \times n}$ with
 $\mathbf{K}_{ij}:= k_{\mathcal{U}}({\bf u}_i, {\bf u}_j)  k_{\mathcal{A}}({\bf a}_i, {\bf a}_j) $, $i,j = 1,\dots,n$.
\STATE Compute $\tilde{\bf K} \in \mathbb{R}^{n \times m}$ with $\widetilde{\mathbf{K}}_{ij} := k_{\mathcal{U}}({\bf u}_i, {\bf u}_j^*)   k_{\mathcal{A}}({\bf a}_i, {\bf a}_j^*)$, $i = 1,\dots,n$, $j = 1,\dots,m$.
 \STATE Compute $\bm{\beta} := (\beta_1,\dots,\beta_n)^\top = (\mathbf{K} + n\epsilon {\bf I})^{-1}\widetilde{\mathbf{K}}\mathbf{1}_m \in \mathbb{R}^n$, where  ${\bf 1}_m := \frac{1}{m}(1,\dots,1)^\top \in \mathbb{R}^m$.
\STATE \textbf{Output}: An estimate   $\hat{\mu}_{\pp{P}_*}  = \sum_{i=1}^n \beta_i \ell(\cdot, r_i)$ of the mean embedding \eqref{eq:OPE-embedding} or an estimate $\hat{R}_* := \sum_{i=1}^n \beta_i r_i$ of the expected reward \eqref{eq:OPE-expected-reward}.
\end{algorithmic}
\end{algorithm}

\section{Experiments}
\label{sec:experiments}
This section provides empirical results that demonstrate the advantages of the proposed framework.
The codes to reproduce the experiments are available at \url{https://github.com/sorawitj/counterfactual-mean-embedding}.

\subsection{Simulations: Distributional Treatment Effects}
\label{sec:simulations}

We first conduct simulation experiments on distributional treatment effects in Section \ref{sec:date-exp} and on distributional effects of covariate distributions in Section \ref{sec:exp-dist-eff-cov}. 

\subsubsection{Distributional Treatment Effects (DTE)}
\label{sec:date-exp}

We first deal with the identification of DTE (Sections \ref{sec:DATE}), defined as the difference between the distributions $\pp{P}_{Y_0^*}$ and $\pp{P}_{Y_1^*}$ of two potential outcomes $Y_0^*, Y_1^* \in \mathbb{R}$. 
As discussed in Section \ref{sec:causal-two-sample}, this identification problem can be formulated as hypothesis testing of the null hypothesis $H_0: \pp{P}_{Y_{1}^*}=\pp{P}_{Y_{0}^*}$ against the alternative $H_1: \pp{P}_{Y_{1}^*}\neq\pp{P}_{Y_{0}^*}$. For this purpose, we assume that i.i.d.~observations $\{ (\x_i, t_i, \y_i) \}_{i=1}^N$ of random variables $(X,T,Y)$ are available, where $X \in \mathbb{R}^5$ is covariates, $T \in \{ 0,1 \}$ is a treatment indicator, and $Y = \mathbbm{1}(T=0) Y_0^* + \mathbbm{1}(T=1) Y_1^* \in \mathbb{R}$ is the observed outcome.


The purpose here is to demonstrate the validity of our approach to identifying DTE, described in Sections \ref{sec:KTE-DATE} and Section \ref{sec:causal-two-sample}. To this end, we compare it with a baseline approach that uses an estimate of ATE \eqref{eq:ate} as a test statistic. 
For simplicity, we call here our approach ``DTE'' and the baseline ``ATE''.
We consider the following three scenarios:


\vspace{-3pt}
\paragraph{Scenario I.}
There exists no treatment effect so that the distributions of the potential outcomes $Y_0^*, Y_1^*$ are the same: $\pp{P}_{Y_0^*} = \pp{P}_{Y_1^*}$. Hence, we expect that both ATE and DTE do not detect any treatment effect.


\vspace{-3pt}
\paragraph{Scenario II.} 
There exists a treatment effect that only makes the means of $\pp{P}_{Y_0^*}$ and $\pp{P}_{Y_1^*}$ different: \ie, the mean-shift scenario. Hence, we expect that both ATE and DTE can detect the treatment effect.


\vspace{-3pt}
\paragraph{Scenario III.} 
There exists a treatment effect that does not change the means of $\pp{P}_{Y_0^*}$ and $\pp{P}_{Y_1^*}$,  but changes their higher order moments.
Hence, we expect that ATE fails to detect any treatment effect, whereas DTE with non-linear kernels can detect the difference.


\begin{figure}[t!]   
  \centering 
  \includegraphics[width=0.32\textwidth]{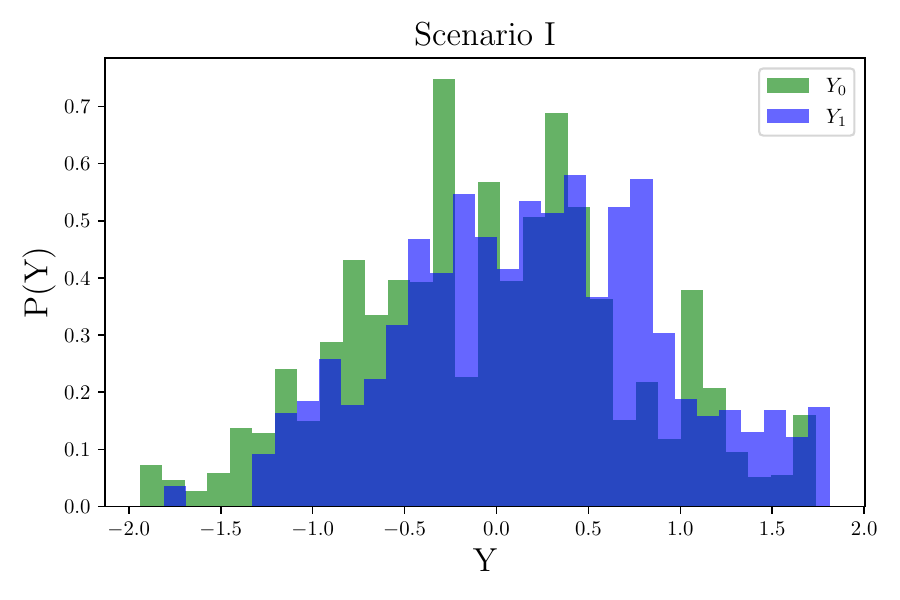}
  \includegraphics[width=0.32\textwidth]{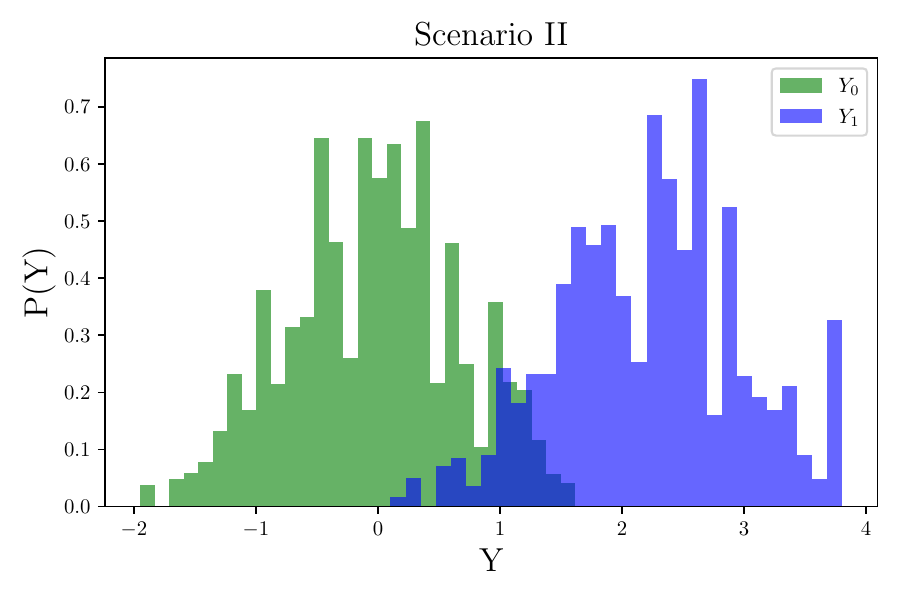}
  \includegraphics[width=0.32\textwidth]{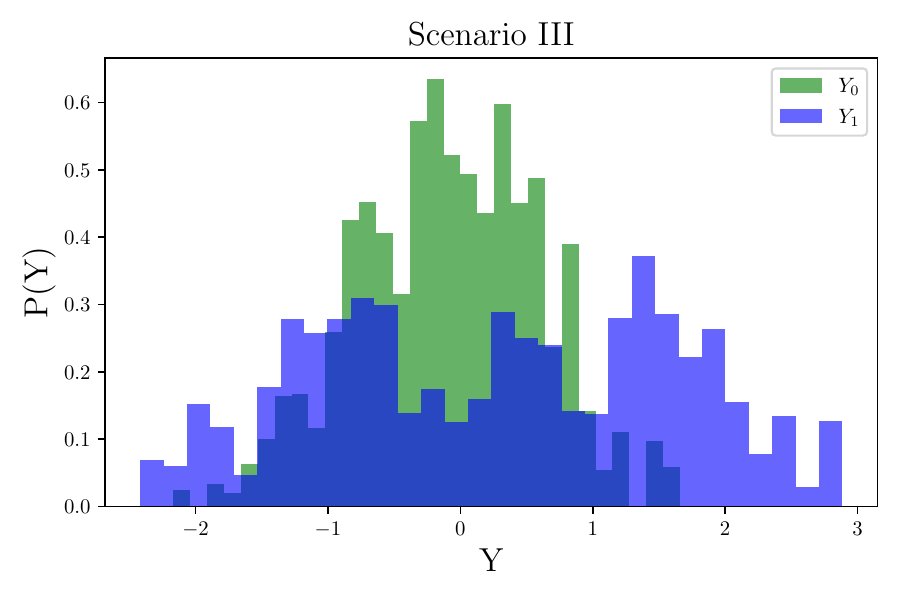}
  \caption{
  Histograms of observed outcomes $(\y_i)_{i=1}^N$ from the data $\{(\x_i, t_i, \y_i) \}_{i=1}^N$ generated under the three scenarios in Section \ref{sec:date-exp}, with $N=500$.
  For each scenario, the green histogram consists of outcomes $\y_i$ with $t_i = 0$, which are i.i.d.~with $Y_0 = Y|(T=0)$), and the blue histogram consists of $\y_i$ with $t_i = 1$, which are i.i.d.~with $Y_1 = Y|(T=1)$.
  Note that $Y_0 =  Y_0^* |(T=0)$ and $Y_1 = Y_1^* | (T=1)$, so the distributions of $Y_0$ and $Y_1$ (described here) are slightly different from those of the potential outcomes $Y_0^*$ and $Y_1^*$.
  }
  \label{fig:synthetic-data}
\end{figure}  

To realize these scenarios, we define the random variables $X$, $T$, $Y_0^*$ and $Y_1^*$ as  
\begin{align*}
&X \sim \mathcal{N}(\mathbf{0},\sigma_{\x}\id_{5}), \quad  T \sim \text{Bernoulli}\left(\frac{1}{1+\exp\left(-\bm{\alpha}^\top X - \alpha_0\right)}\right) \\ 
& Y_0^*  = \beta^\top X  + \varepsilon_0, \quad Y_1^*= \beta^\top X + b + \varepsilon_1, \nonumber 
\end{align*}
where  $\varepsilon_0, \varepsilon_1 \sim \mathcal{N}(0,\sigma^2_{\varepsilon})$ are independent noises.
%
Throughout the experiment, we set $\bm{\beta} =[0.1,0.2,0.3,0.4,0.5]^\top$, $\bm{\alpha} =[0.05,0.04,0.03,0.02,0.01]^\top$, $\alpha_0=0.05$, and $\sigma^2_{\varepsilon}=\sigma^2_{\x}=0.1$. 
We set $b=0$ for the Scenario I and $b=2$ for the Scenario II.
For Scenario III, we set $b = 2z - 1$, where $z \in \{0, 1\}$ is an independent Bernoulli random variable $z\sim\mathrm{Bernoulli}(0.5)$ generated for every observation.
By construction, the conditional exogeneity $Y_0^*, Y_1^*  \ci T | X$ in Assumption \ref{asmp:cond-exogen} is satisfied. 
For each scenario, we generate data $\{(\x_i, t_i, \y_i) \}_{i=1}^N$ as i.i.d.~observations of $X$, $T$ and $Y = \mathbbm{1}(T=0) Y_0^* + \mathbbm{1}(T=1) Y_1^*$, with $N\in\{50,100\}$.
Figure \ref{fig:synthetic-data} describes the empirical distributions of observed outcomes for the three scenarios, where $Y_0 = Y|(T=0)$ and $Y_1 = Y | (T=1)$.


For DTE and ATE, we perform the following tests using data 
$\{(\x_i,t_i,\y_i,e_i)\}_{i=1}^n$ augmented with propensity scores $e_i := e(\x_i) := \mathbb{E}[T \,|\,X=\x_i]$.
For DTE, we use the unbiased KTE estimate in \eqref{eq:unbiased-KTE-DATE} as a test statistic, with the Gaussian kernel $\ell(\y, \y') = \exp(-\|\x-\y\|_2^2/2\sigma^2)$ whose bandwidth parameter $\sigma$ is chosen using the median heuristic \citep{garreau2017large}. 
For ATE, we also use \eqref{eq:unbiased-KTE-DATE} as a test statistic, but with the linear kernel $\ell(\y, \y') = \y^\top\y'$, resulting in a test that distinguishes only the means of two distributions.
We use the bootstrap procedure described in Section  \ref{sec:causal-two-sample} to construct the distribution of the test statistic under the null $H_0: \pp{P}_{Y_0^*} = \pp{P}_{Y_1^*}$, with $B=10,000$ bootstrap samples. The significance level $\alpha$ is set to $0.01$ in all experiments.

Table \ref{tab:dist-results} reports the frequencies of rejecting the null hypothesis $H_0: \pp{P}_{Y_0^*} = \pp{P}_{Y_1^*}$ over 1000 repetitions, for each of the three scenarios.  When the null hypothesis $H_0$ is true (Scenario I), these are the frequencies of Type-I errors, which are well calibrated approximately at the designed level $\alpha = 0.01$ for both ATE and DTE. 
When the alternative hypothesis $H_1: \pp{P}_{Y_0^*} \not= \pp{P}_{Y_1^*}$ is true (Scenarios II and III), these represent test powers (\ie, one minus the probability of Type II error).  
In Scenario II, both ATE and DTE successfully reject the null hypothesis, capable of detecting the mean shift effect in the potential outcome distributions. 
In Scenario III, where the treatment effects do not appear in the mean but in the higher order moments, DTE has significantly higher power than ATE, demonstrating that DTE can identify higher order distributional effects.


\begin{table}[t!]
    \centering
    \begin{tabular}{llcccccc}
        \toprule
         &&& \multicolumn{2}{c}{$N=50$} && \multicolumn{2}{c}{$N=100$} \\
         &&& \textbf{ATE} & \textbf{DTE} && \textbf{ATE} & \textbf{DTE} \\
         \midrule
        \textbf{Scenario I}: & No Treatment Effect && 0.013 & 0.012 && 0.013 & 0.012 \\
        \textbf{Scenario II}: & Mean Shift Effect && 1.000 & 1.000 && 1.000 & 1.000 \\
        \textbf{Scenario III}: & High-order Treatment Effect && \textbf{0.012} & \textbf{0.224} && \textbf{0.012} & \textbf{0.639} \\
        \bottomrule
    \end{tabular}
    \caption{
    The frequencies of rejecting the null hypothesis $H_0: \pp{P}_{Y^*_{1}}=\pp{P}_{Y^*_{0}}$ when the null hypothesis is true (\ie, the probability of the Type-I error in Scenario I) and when the alternative hypothesis $H_1: \pp{P}_{Y^*_{1}}\neq\pp{P}_{Y^*_{0}}$ is true (\ie, the power of the test in Scenario II \& III), computed from 1000 repetitions.
    The significance level $\alpha$ is 0.01.}
    \label{tab:dist-results}
\end{table}


\subsubsection{Distributional Effects of Covariate Distributions} \label{sec:exp-dist-eff-cov}

We next consider the identification of distributional effects of covariate distributions (see Sections \ref{sec:effects_by_treatment_assignment},  \ref{sec:more-on-assignment} and \ref{sec:KTE-dist-effect-cov}).
As before, let $Y_0^*, Y_1^* \in \mathbb{R}$ be potential outcomes,  $T \in \{0,1\}$ be a treatment indicator, $Y = \mathbbm{1}(T=0)Y_0^* +  \mathbbm{1}(T=1)Y_1^*$ be the observed outcome, and $X \in \mathbb{R}^5$ be covariates. Let $X_0 = X | (T=0)$, $X_1 = X | (T=1)$ and $Y_0 = Y|(T=0) = Y_0^* | (T = 0)$.
Here we are interested in the distributional effects defined as
\begin{align*}
\pp{P}_{Y_0^* \,|\, T } (\cdot \,|\, 0) - \pp{P}_{Y_0^* \,|\, T } (\cdot \,|\, 1) &= \pp{P}_{Y\left< 0|0\right>} - \pp{P}_{Y\left< 0|1\right>} \\
&= \int \pp{P}_{Y_0|X_0}(\cdot|\x) \pp{P}_{X_0}(\x) - \int \pp{P}_{Y_0|X_0}(\cdot|\x) \pp{P}_{X_1}(\x)
\end{align*}
where the first identity holds under the conditional exogeneity.
As discussed in Section \ref{sec:causal-two-sample}, the identification of this distributional effect can be cast as testing the null hypothesis  $H_0:  \pp{P}_{Y_0^* | T } (\cdot \,|\, 0) = \pp{P}_{Y_0^* | T } (\cdot \,|\, 1)$ against the alternative $H_1: \pp{P}_{Y_0^* | T } (\cdot \,|\, 0) \not= \pp{P}_{Y_0^* | T } (\cdot \,|\, 1)$. 
 

For this experiment, we define the joint distribution of $T$, $X$ and $Y_0^*$ by first specifying the distribution of $T$, and then specifying the conditional distributions of $X$, $Y_0^*$ and $Y_1^*$ given $T$ (note that $Y_1^*$ is not relevant in this experiment). 
To this end, we define the distribution $\pp{P}_T$ of $T$ as $\pp{P}_T(0) = \pp{P}_T(1) = 1/2$. 
Then, we define the conditional distributions of $X$ and  $Y_0^*$ given $T$ as
\begin{align*}\label{eq:synthetic-data-ta}
    &  Y_0^* \,|\, (T=0) = \bm{\beta}^\top X_0 + \varepsilon_0, \quad 
    X_0 = X|(T=0) \sim \mathcal{N}(\mathbf{0},\sigma_x\id_{5}), \\
    & Y_0^* \,|\, (T=1) = \bm{\beta}^\top X_1 + \varepsilon_1, \quad
    X_1 = X|(T=1) \sim \sum_{j=1}^3\gamma_j\mathcal{N}(\bm{\nu}_j,\sigma_x\id_{5}),
\end{align*}
\noindent where $\varepsilon_0, \varepsilon_1 \sim \mathcal{N}(0,\sigma^2_{\varepsilon})$ are independent, $\bm{\beta} =[0.1,0.2,0.3,0.4,0.5]^\top$, $\sigma_\varepsilon = \sigma_x = 0.1$, and  $\gamma_1 = \gamma_2 = \gamma_3 = 1/3$. 
We set $\bm{\nu}_1 = [-5, 2.5, 0, 0, 2.5]$, $\bm{\nu}_2 = [2.5, 2.5, 0, 0, -5]$, and $\bm{\nu}_3 = [2.5, -5, 0, 0, 2.5]$, so that $X_0$ and $X_1$ have the same zero mean.
By construction, $Y_0^*\,|\,T=0$ and $Y_0^*\,|\,T=1$ have the same mean, which is zero, while their higher-order moments differ. 
In other words, the distributional effects of the covariate distributions appear only in the higher-order moments.
We generate data $(\x_i, \y_i)_{i=1}^n$ as i.i.d.~observations of $(X_0, Y_0)$ (recall that $Y_0 = Y_0^*|(T=0)$) and $(\x'_j)_{j=1}^m$ as i.i.d.~observations of $X_1$, where $n = m$ (which amounts to $\pp{P}_T(0) = \pp{P}_T(1) = 1/2$).


We estimate the embedding $\mu_{Y_0|T=1} = \int \ell(\cdot, \y) \dd \pp{P}_{Y_0|T}(\y|1)$ of the counterfactual distribution $\pp{P}_{Y_0|T}(\cdot|1) = \pp{P}_{Y\left< 0|1\right>}$ with the CME estimator \eqref{eq:empirical-cme} based on $(\x_i, \y_i)_{i=1}^n$ and $(\x'_j)_{j=1}^m$.
We set the kernel $\ell$ on the outcome space as the Gaussian kernel $\ell(\y,\y') = \exp(-\|\y-\y'\|_2^2/2\sigma_Y^2)$ whose bandwidth parameter $\sigma_Y$ is chosen by the median heuristic using $(\y_i)_{i=1}^n$.
We also set the kernel $k$ on the covariate space as the Gaussian kernel  $k(\x,\x') = \exp(-\|\x-\x'\|_2^2/2\sigma_X^2)$, whose parameter $\sigma_X$ as well as the regularization constant $\varepsilon$ in the CME estimator are chosen by 5-fold cross validation from  $\sigma_X\in\{0.01,0.1,1,10\}$ and $\varepsilon\in\{0.01,0.1,1,10\}$. 
This cross validation is done by regarding the joint sample $(\x_i, \y_i)_{i=1}^n$ as training data for regression from $\x_i$ to $\y_i$, and by performing kernel ridge regression with kernel $k$ and regularization parameter $\varepsilon$, motivated by the interpretation of conditional mean embedding as kernel ridge regression \citep{Grunewalder12:LGBPP}.
 

We apply Algorithm \ref{alg:sampling-CME} to the resulting CME estimate $\hat{\mu}_{Y\left<0|1\right>} = \sum_{i=1}^n \beta_i \ell(\cdot, \y_i)$ to generate counterfactual samples $(\y'_j)_{j=1}^n$.
We now have $(\y_i)_{i=1}^n$ as an i.i.d.~sample from $\pp{P}_{Y_0^* | T } (\cdot \,|\, 0)$ and $(\y'_j)_{j=1}^n$ as an approximate sample of the counterfactual distribution $\pp{P}_{Y_0^* | T } (\cdot \,|\, 1)$.
As discussed in Section \ref{sec:causal-two-sample}, we can test the null hypothesis  $H_0:  \pp{P}_{Y_0^* | T } (\cdot \,|\, 0) = \pp{P}_{Y_0^* | T } (\cdot \,|\, 1)$ against the alternative $H_1: \pp{P}_{Y_0^* | T } (\cdot \,|\, 0) \not= \pp{P}_{Y_0^* | T } (\cdot \,|\, 1)$ by performing a two sample test using the samples $(\y_i)_{i=1}^n$ and $(\y'_j)_{j=1}^n$.
For this purpose, we perform the kernel two-sample test with the unbiased MMD statistic \cite[Eq.~3]{Gretton12:KTT}, with permutation-based bootstrapping using $B=10,000$ bootstrap samples and with significance level $\alpha = 0.01$.
For comparison, we also perform the same kernel two-sample test, but with linear kernel $\ell(\y, \y') = \y^\top \y'$, resulting in a test that only uses the means of $(\y_i)_{i=1}^n$ and $(\y'_j)_{j=1}^n$.



\begin{figure}[t!]
    \centering
        \subfigure[]{
        \centering
        \includegraphics[width=0.45\textwidth]{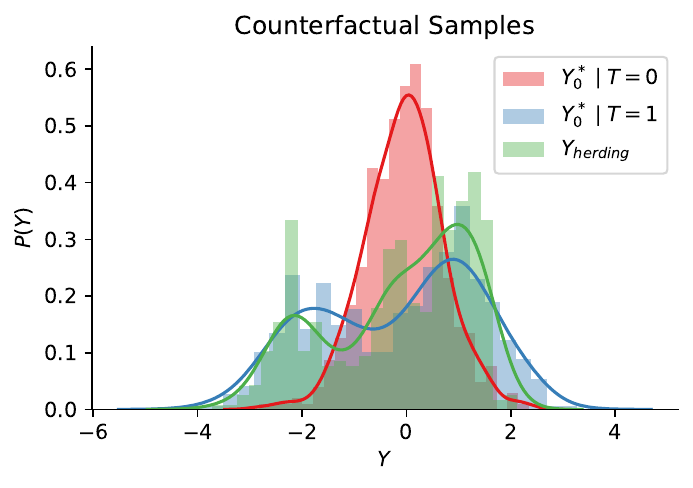}
        \label{subfig:hearding-sample}
        }%
        \subfigure[]{
        \centering
        \includegraphics[width=0.45\textwidth]{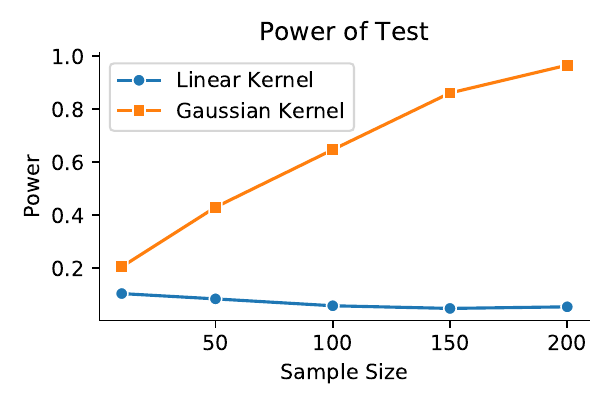}
        \label{subfig:hearding-power}
        }%
        
    \caption{
    Results of the experiments in Section \ref{sec:exp-dist-eff-cov}. 
    \subref{subfig:hearding-sample}
     Histograms of observed outcomes $(\y_i)_{i=1}^n$ of $Y_0^* | (T=0)$ (red), a counterfactual sample of $Y_0^* | (T=1)$ (blue) and the approximate counterfactual sample $(\y'_j)_{j=1}^n$ generated with Algorithm \ref{alg:sampling-CME} applied to our CME estimate (green), obtained from data with size $n = 500$.
     For illustration, we also show density curves obtained from the corresponding samples of the same colors, estimated with kernel density estimation.
    \subref{subfig:hearding-power}
    Powers of the two sample tests based on the generated counterfactual sample $(\y'_j)_{j=1}^n$ and observed outcomes $(\y_i)_{i=1}^n$, using the unbiased MMD statistic with the Gaussian or the linear kernel with significance level $\alpha = 0.01$. The powers are obtained from 1,000 repetitions, for each of different sample sizes.}
    \label{fig:effects-ta}
\end{figure}

Figure \ref{fig:effects-ta} describes the experimental results. 
Figure \ref{subfig:hearding-sample} illustrates the observed outcomes $(\y_i)_{i=1}^n$ of $Y_0^* | (T=0)$ (red), a counterfactual sample of $Y_0^* | (T=1)$ (blue) and the approximate counterfactual sample $(\y'_j)_{j=1}^n$ generated with Algorithm \ref{alg:sampling-CME} applied to our CME estimate (green). Note that the sample of  $Y_0^* | (T=1)$ is shown here for an illustration purpose; in practice we never have access to such a sample, but we can generate it here as we know the ground-truth model.  For illustration, we also show the corresponding density curves obtained from the respective samples using kernel density estimation. 
The approximate counterfactual sample  $(\y'_j)_{j=1}^n$ resembles that from the ground-truth model, supporting the validity of our CME estimator \eqref{eq:empirical-cme} and the sampling method (Algorithm \ref{alg:sampling-CME}). 

Figure \ref{subfig:hearding-power} describes the test powers (i.e., the frequencies of rejecting the null hypothesis $H_0:  \pp{P}_{Y_0^* | T } (\cdot \,|\, 0) = \pp{P}_{Y_0^* | T } (\cdot \,|\, 1)$) over 1,000 repetitions of the above testing procedure for each case of using the Gaussian or the linear kernel for computing the test statistic, for different sample sizes. 
The test with the linear kernel has very low power. 
This implies that the mean of the generated counterfactual sample $(\y'_j)_{j=1}^n$ is close to the mean of the observed sample $(\y_i)_{i=1}^n$ from $Y_0^*| (T=0)$ since the kernel two-sample test with the linear kernel only uses the information of the sample means.  
On the other hand,  the power of the test with the Gaussian kernel increases as the size $n$ of observed data increases, suggesting that the higher-order moments of the generated counterfactual sample $(\y'_j)_{j=1}^n$ differ substantially from those of the observed sample $(\y_i)_{i=1}^n$. These observations suggest that the approximate counterfactual sample $(\y'_j)_{j=1}^n$ has properties consistent with the ground-truth counterfactual distribution $Y_0^*|T=1$. Thus our CME estimator \eqref{eq:empirical-cme} and Algorithm \ref{alg:sampling-CME} are capable of producing an approximate counterfactual sample based on which a test for distributional effects can be constructed.

\subsection{Off-Policy Evaluation} \label{sec:OPE-experiments}

We conduct experiments on the off-policy evaluation (OPE) task for a recommendation system, described in Section \ref{sec:policy-evaluation}. 
Let $\eta: \mathcal{U} \times \mathcal{A} \to \mathbb{R}$ be the regression function that takes a pair $({\bf u}, {\bf a})$ of user features ${\bf u} \in \mathcal{U}$ and recommendation ${\bf a} \in \mathcal{A}$ as an input and outputs the conditional expectation of the reward $r$:
$$
\eta({\bf u}, {\bf a}) := \mathbb{E}[ r \,|\, {\bf u}, {\bf a} ] := \int_{\mathbb{R}} r \dd \pp{P}_*( r \,|\,  {\bf u}, {\bf a} ) =   \int_{\mathbb{R}} r \dd\pp{P}_0( r \,|\,  {\bf u}, {\bf a} ), 
$$
where $\pp{P}_*( r \,|\,  {\bf u}, {\bf a} ) = \pp{P}_0( r \,|\,  {\bf u}, {\bf a} )$ is the conditional distribution of the reward $r$ given the pair $({\bf u}, {\bf a})$, which is assumed to be invariant under the target and logging environments.

For a given target policy $\pi_*({\bf a} | {\bf u})$, the OPE task is to estimate the expected reward under the target environment defined by
$$
R_* := \int_{\mathcal{U} \times \mathcal{A} }  \int_{\mathbb{R}} r \dd\pp{P}_*( r \,|\,  {\bf u}, {\bf a} ) \dd \pi_*(   {\bf u}, {\bf a} ) =   \int_{\mathcal{U} \times \mathcal{A} } \eta({\bf u}, {\bf a}) \dd \pi_*(   {\bf u}, {\bf a} ),
$$
where $\pi_*(  {\bf u}, {\bf a} ) = \pi_*( {\bf a} \,|\, {\bf u} ) q_*({\bf u}) = \pi_*( {\bf a} \,|\, {\bf u} ) q_0({\bf u})$ is the joint distribution of context features ${\bf u} \in \mathcal{U}$ and action ${\bf a} \in \mathcal{A}$.
Here we consider the standard setting where the marginal distributions of the user features ${\bf u}$ are the same under the target and logging environments: $q_*( {\bf u}) = q_0 ( {\bf u})$.
The above estimation is to be done based on the logged data $\mathcal{D}_{\mathrm{init}} := \{(\mathbf{u}_i,\bfa_i,r_i)\}_{i=1}^n$ obtained from the joint distribution $\pp{P}_0({\bf u}, {\bf a}, r) = \pp{P}_0( r \,|\,  {\bf u}, {\bf a} ) \pi_0({\bf a}, {\bf u})$ during the data collection phase, where $\pi_0({\bf u}, {\bf a}) = \pi_0({\bf u} | {\bf a}) q_0({\bf u})$.

We compare our approach in Algorithm \ref{alg:kpe}, which we call \texttt{CME} below, to the following benchmark estimators using both simulated and real-world data.

\begin{paragraph}{Direct method with a parametric regressor (DM).}
The direct method \citep{Dudik11:DoublyRobust} first learns the regression function $\eta$ based on the logged data $\mathcal{D}_{\rm init}$ with a regression model of one's choice.  
Let $\hat{\eta}: \mathcal{U} \times \mathcal{A} \to \mathbb{R}$ be the learned regressor. 
Then the expected reward $R_*$ is estimated as 
$$\widehat{R}_{\text{DM}} = \frac{1}{n}\sum_{i=1}^{n} \mathbb{E}_{\mathbf{a} \sim \pi_*(\mathbf{a}\,|\,\mathbf{u}_i)}[\hat{\eta}(\mathbf{u}_i,\mathbf{a})].$$


The direct method obtains the approximation $\hat{\eta}$ based on the logged data  $\mathcal{D}_{\rm init} = \{(\mathbf{u}_i,\bfa_i,r_i)\}_{i=1}^n$, in which input pairs $({\bf u}_i, {\bf a}_i)$ are generated from the covariate distribution $\pi_0({\bf u}, {\bf a}) = \pi_0({\bf u} | {\bf a}) q_0({\bf u})$ that is different from the target covariate distribution $\pi_*({\bf u}, {\bf a}) = \pi_*({\bf u} | {\bf a})q_0({\bf u})$. Recall that we interpret the paired variables $({\bf u}, {\bf a})$ as ``covariates'' in our discussion.
This situation is known as {\em covariate shift} in the literature. It is well known that under the covariate shift, a {\em parametric} regression model may produce a significant bias  \citep{Shimodaira00:CovShift}. That is, the approximation quality of the learned model $\hat{\eta}$ obtained with the logged data $\mathcal{D}_{\rm init}$ may be good with respect to the covariate distribution  $\pi_0({\bf u}, {\bf a})$ under the data collection environment, but can be poor with respect to the target covariate distribution $\pi_*({\bf u}, {\bf a})$, e.g.,  $\| \eta - \hat{\eta} \|_{L_2(\pi_*) }^2 := \int (\eta( {\bf u}, {\bf a} ) - \hat{\eta} ( {\bf u}, {\bf a} ) )^2 \dd\pi_*({\bf u}, {\bf a})$ may be large. 
This in turn may induce a large bias in the estimation of the expected reward $R_* = \int \eta( {\bf u}, {\bf a}) \dd \pi_*({\bf u}, {\bf a})$.
To demonstrate this, we use a 3-layer feedforward neural network, which is an (overparametrized) parametric model, as a regressor for the direct method.

\end{paragraph}

\begin{paragraph}{Weighted inverse propensity score (wIPS).}
The \texttt{wIPS} estimator obtains an unbiased estimate of the target reward by re-weighting each observation in the logged dataset by the ratio of the {\em propensity scores} under the target and initial policies \citep{Horvitz52:Sampling, Precup00:wIPS}. 
The \texttt{wIPS} estimator is defined by 
$$\widehat{R}_{\text{wIPS}} = \left(\sum_{i=1}^{n} w_i r_i\right)\bigg/\left(\sum_{i=1}^n w_i\right),$$ 
where $w_i :=  \pi_*(\mathbf{a}_i, \mathbf{u}_i) /  \pi_0(\mathbf{a}_i, \mathbf{u}_i) =  \pi_*(\mathbf{a}_i|\mathbf{u}_i) /  \pi_0(\mathbf{a}_i|\mathbf{u}_i)$ are the propensity weights.
\end{paragraph}

\begin{paragraph}{Doubly robust (DR).}
The \texttt{DR} estimator combines the two aforementioned estimators by exploiting both the regression model $\hat{\eta}{(\mathbf{u},\mathbf{a})}$ and the propensity scores \citep{Cassel76:DR,Dudik11:DoublyRobust}. The estimator is given by 
$$
\widehat{R}_{\text{DR}} = \frac{1}{n}\sum_{i=1}^{n} \left( \mathbb{E}_{\mathbf{a} \sim \pi_*(\mathbf{a}\,|\,\mathbf{u}_i)}[\hat{\eta}{(\mathbf{u}_i,\mathbf{a})]} + 
w_i (r_i - \hat{\eta}{(\mathbf{u}_i,\mathbf{a}_i)})\right).
$$ 
It has been proved to be unbiased if at least one of the estimators, $\hat{\eta}$ and $\pi_*/\pi_0$ is correctly specified; see, \eg, \citet{Cassel76:DR,Dudik11:DoublyRobust}.
\end{paragraph} 

\begin{paragraph}{Slate estimator.}
The \texttt{Slate} estimator, proposed for recommendation systems, makes use of the structure within a recommendation (= action) by assuming a certain linearity assumption on the regression function with respect to the recommendation \citep{Swaminathan17:Slate}.
More precisely, \citet{Swaminathan17:Slate} consider a recommendation system in which an action ${\bf a} \in \mathcal{A}$ is an ordered list (called {\em slate}) of $K \in \mathbb{N}$ items chosen from $M \in \mathbb{N}$ possible items. 
Let $\mathbf{1}_{\mathbf{a}} \in \mathbb{R}^{\mathit{KM}}$ be the indicator vector whose $(k, m)$-th element is $1$ if ${\bf a}$ contains the item $m \in \{1,\dots,M\}$ in the slot $k \in \{1,\dots,K\}$,  and 0 otherwise.  
\citet[Assumption 1]{Swaminathan17:Slate} then model the regression function $\eta({\bf u}, {\bf a})$ as a linear function of this indicator vector: $\eta({\bf u}, {\bf a}) = {\bf w}_{\bf u}^\top {\bf 1}_{\bf a}$, where ${\bf w}_{\bf u}$ is an unknown feature vector of the context ${\bf u}$ (Note that ${\bf w}_u$ can be a nonlinear function of ${\bf u} \in \mathcal{U}$).
Under this assumption, the authors derive the slate estimator as
$$
\widehat{R}_{\text{slate}} = \frac{1}{n}\sum_{i=1}^{n} r_i \cdot \mathbf{q}_{\mathbf{u}_i}^\top \Gamma_{ \mathbf{u}_i }^\dagger\mathbf{1}_{\mathbf{a}_i} ,
$$ 
where $\Gamma_{ \mathbf{u}_i }^\dagger$ is the Moore-Penrose pseudoinverse of the matrix $\Gamma_{ \mathbf{u}_i } := \mathbb{E}_{\mathbf{a} \sim \pi_*(\mathbf{a}\,|\,\mathbf{u}_i)}[\mathbf{1}_{\mathbf{a}} \mathbf{1}_{ \mathbf{a} }^\top ] \in \mathbb{R}^{KM \times KM}$, and $\mathbf{q}_{\mathbf{u}_i} := \mathbb{E}_{\mathbf{a} \sim \pi_*(\mathbf{a}\,|\,\mathbf{u}_i)} [\mathbf{1}_{ \mathbf{a} }] \in \mathbb{R}^{KM}$. 
Thanks to the linearity assumption, the slate estimator may enjoy a lower variance than the \texttt{wIPS} estimator, while the assumption may also lead to a non-vanishing bias if it does not hold.
\end{paragraph} \\

For the \texttt{CME}, we use a kernel defined as $k((\mathbf{u},\mathbf{a}),(\mathbf{u}',\mathbf{a}')) := k_{\mathcal{U}}(\mathbf{u},\mathbf{u}') k_{\mathcal{A}}(\mathbf{a},\mathbf{a}')$ where $k_{\mathcal{U}}(\mathbf{u},\mathbf{u}') := \exp\left( -\|\mathbf{u}-\mathbf{u}'\|_2^2/2\sigma_{u}^2\right)$ and $k_{\mathcal{A}}(\mathbf{a},\mathbf{a}') := \exp\left( -\|\mathbf{a} - \mathbf{a}' \|_2^2/2\sigma_{a}^2\right)$.
For this experiment, the linear kernel  $\ell(r,r') := r r'$ is used as a reward kernel since we only compare the estimation of the expected reward. 
The regularization parameter $\varepsilon$ is selected by the cross validation procedure in Appendix \ref{sec:model-selection}, while we determined $\sigma_u$ and $\sigma_a$ by the median heuristic, \ie, $\sigma_{u}^2 = \mathrm{median}\{\|\mathbf{u}_i- \mathbf{u}_j\|_2^2\}_{1 \leq i < j \leq n}$ and $\sigma_{a}^2 = \mathrm{median}\{\|\mathbf{a}_i- \mathbf{a}_j\|_2^2\}_{1 \leq i < j \leq n}$.  

Before proceeding, we point out here a connection between our approach and the direct method. Assume that we use kernel ridge regression to obtain the approximation $\hat{\eta}$ of the regression function $\eta$ using the logged data: $\hat{\eta} ({\bf u}, {\bf a}) =  \mathbf{r}^\top (\mathbf{K} + n\epsilon I)^{-1}\widetilde{\mathbf{k}}(\mathbf{u}, \mathbf{a})$, where $\mathbf{r} := (r_1,\ldots,r_n)^\top \in \mathbb{R}^n$ and $\widetilde{\mathbf{k}}(\mathbf{u}, \mathbf{a}) := ( k((\mathbf{u}_j, \mathbf{a}_j), (\mathbf{u}, \mathbf{a})))_{j=1}^n \in \mathbb{R}^n$
Then, the estimate of the direct method can be related to the CME estimate as
    \begin{eqnarray}
    \widehat{R}_{\text{DM}} &=& \frac{1}{n}\sum_{i=1}^{n} \mathbb{E}_{\mathbf{a} \sim \pi_*(\mathbf{a}\,|\,\mathbf{u}_i)}[\hat{\eta}(\mathbf{u}_i,\mathbf{a})]  \approx \frac{1}{n}\sum_{i=1}^{n}\hat{\eta}(\mathbf{u}_i,\mathbf{a}^{*}_i) \nonumber \\
    &=& \frac{1}{n}\sum_{i=1}^{n} \mathbf{r}^\top (\mathbf{K} + n\epsilon I)^{-1}\widetilde{\mathbf{k}}(\mathbf{u}_i, \mathbf{a}^{*}_i) = \mathbf{r}^\top (\mathbf{K} + n\epsilon I)^{-1} \frac{1}{n}\sum_{i=1}^{n} \widetilde{\mathbf{k}}(\mathbf{u}_i, \mathbf{a}^{*}_i) \label{eq:krr},
    \end{eqnarray}
        where the approximation in the first line is a Monte Carlo approximation based on a single draw $\mathbf{a}_i^*$ from the target policy $\pi_*(\mathbf{a}\,|\,\mathbf{u}_i)$ for each $i$.
    As we can see from \eqref{eq:krr}, the estimate has the same form as the CME estimate given in Algorithm 2  when the output kernel $\ell$ is a linear kernel, i.e., when we are only interested in the expected reward. 

In this sense, the CME estimate can be interpreted as the direct method with kernel ridge regressor. 
Note that the kernel ridge regression is a \emph{nonparametric} method (as long as the RKHS of the kernel on covariates is infinite dimensional such as the RKHS of the Gaussian kernel), and thus less prone to the effects of covariate shift. 
In fact, our convergence results in Section \ref{sec:theory-main} show that the CME is consistent and thus asymptotically unbiased. 
This explains why our method, even if it can be related to the direct method, works well in the off-policy evaluation task compared to the direct method using a parametric model (as we will show shortly).



\subsubsection{Simulated Data}
\label{sec:simulated-data}


We consider the following setting for our simulation experiment. When a user visits a website, the system provides a recommendation as an ordered list of $K \in \mathbb{N}$ items out of $M \in \mathbb{N}$ available items to that user. 
Each item $m \in \{1,\dots,M\}$ is represented by a feature vector ${\mathbf{v}_m} \in \mathbb{R}^d$ generated randomly as  ${\mathbf{v}_m} \sim \mathcal{N}(0, \id_{d})$, where $d \in \mathbb{N}$.
Hence, a recommendation is an ordered list $\mathbf{a} = (\mathbf{v}_{m_1}, \mathbf{v}_{m_2}, ..., \mathbf{v}_{m_K}) \in \mathbb{R}^{d \times K}$, where $m_1, m_2, \dots, m_K \subset \{1,\dots,M\}$.
Likewise, each user $j \in \{1,\ldots,N\}$ has a feature vector ${\mathbf{u}_j} \in \mathbb{R}^d$ generated as ${\mathbf{u}_j} \sim \mathcal{N}(0, \id_{d})$, where $N \in \mathbb{N}$ is the number of users. 
The reward from a user is 1 if the user clicks any of the recommended items and 0 otherwise. Specifically, for each $(\mathbf{a}_i,\mathbf{u}_j)$ pair, let $\theta_{ij} = \mathbb{P}(\mathrm{click}\,|\, \mathbf{a}_i,\mathbf{u}_j) = 1/(1 + \exp(-\overline{\mathbf{a}}_i^{\top}\mathbf{u}_j + \epsilon_{ij}))$ be the probability of a click, where $\overline{\mathbf{a}}_i$ is the mean vector of the item vectors listed in $\mathbf{a}_i$, and $\epsilon_{ij} \sim \mathcal{N}(0,1)$ is an independent noise.  
The reward from user $j$ receiving recommendation $\mathbf{a}_i$ is defined as $r_{ij} \sim \mathrm{Bernoulli}(\theta_{ij})$.

\begin{figure*}[t!]
\vspace{-1.5em}
\centering
\subfigure[Policy shift]{
\centering
\includegraphics[width=0.33\textwidth]{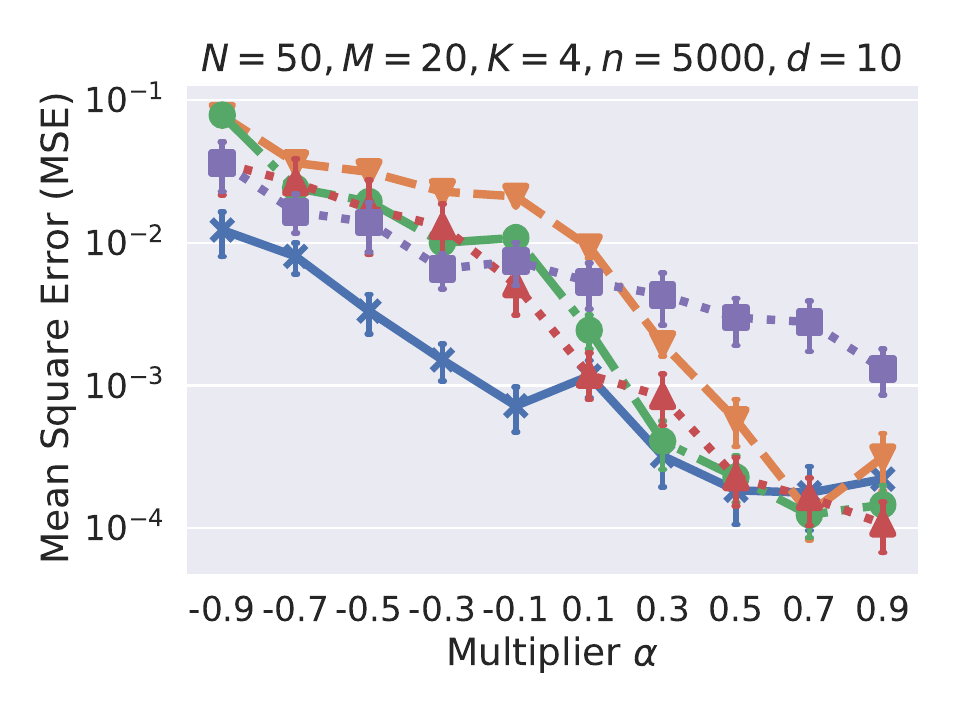}
\label{subfig:policy-shift}
}%
\subfigure[Context dimension]{
\centering
\includegraphics[width=0.33\textwidth]{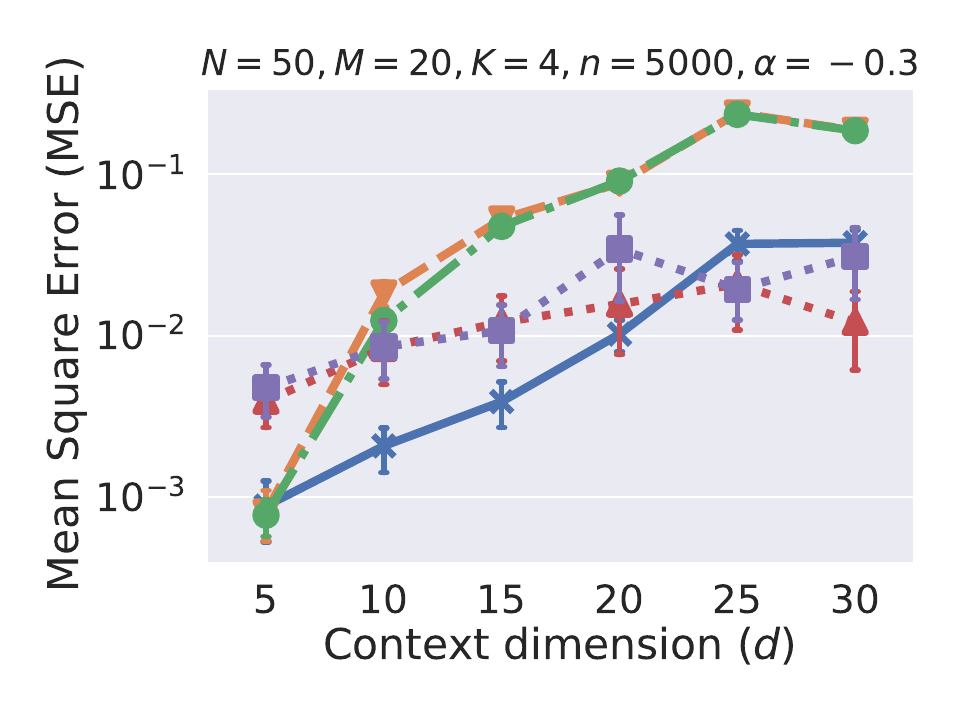}
\label{subfig:context-dim}
}%
\subfigure[Item size]{
\centering
\includegraphics[width=0.33\textwidth]{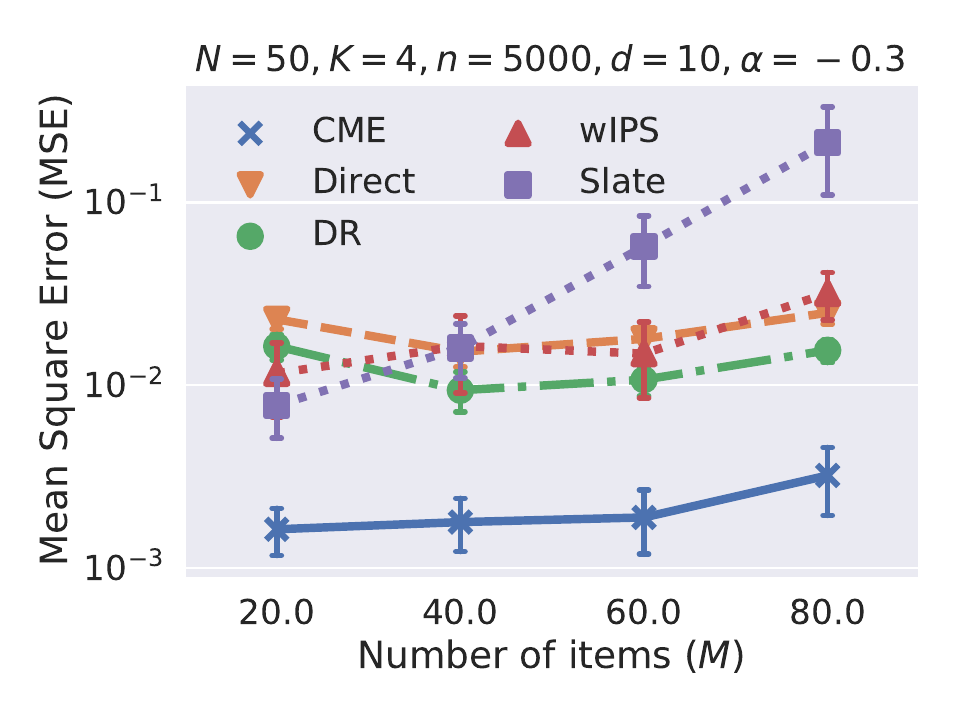}
\label{subfig:reco-size}
}%
\\
\subfigure[Users]{
  \includegraphics[width=0.33\textwidth]{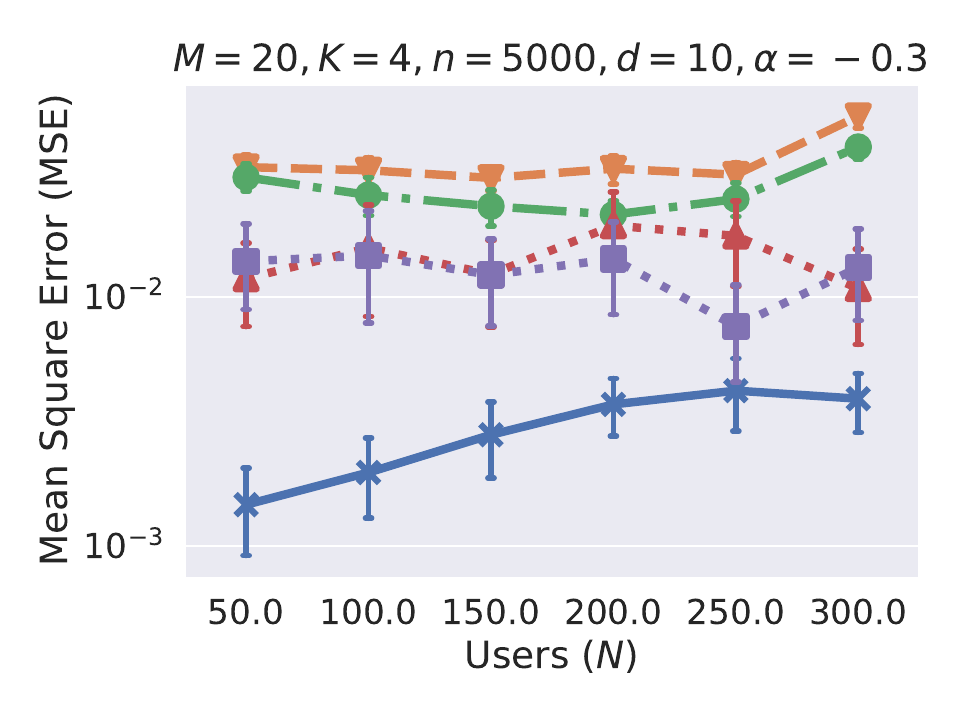}
  \label{subfig:user-size}
}%
\subfigure[Recommended items]{
  \includegraphics[width=0.33\textwidth]{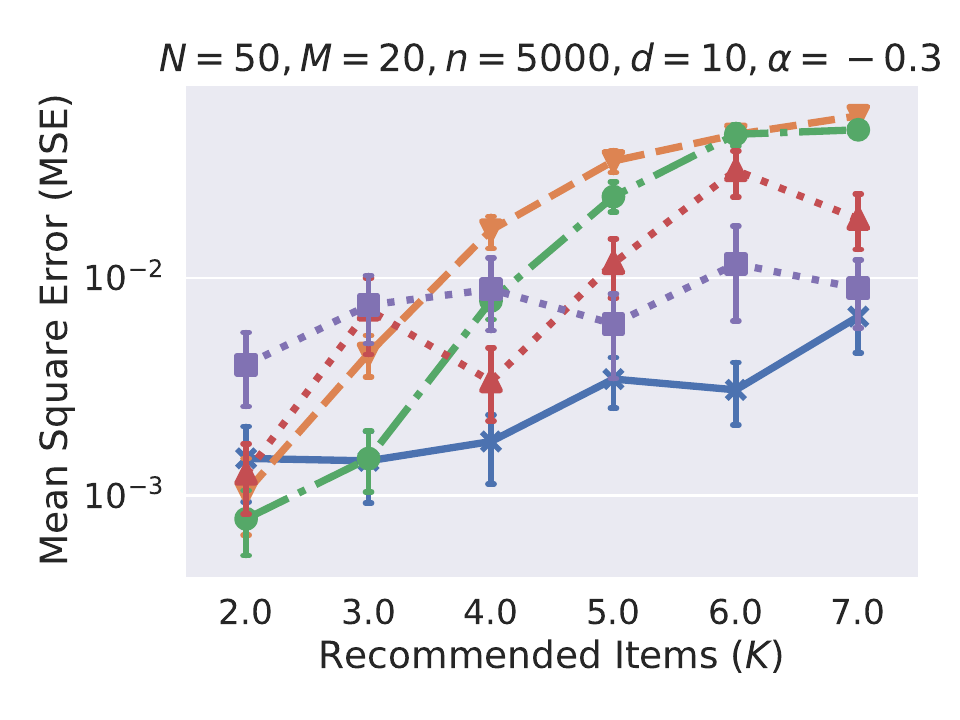}
  \label{subfig:recom-size}
}%
\subfigure[Sample size]{
  \includegraphics[width=0.33\textwidth]{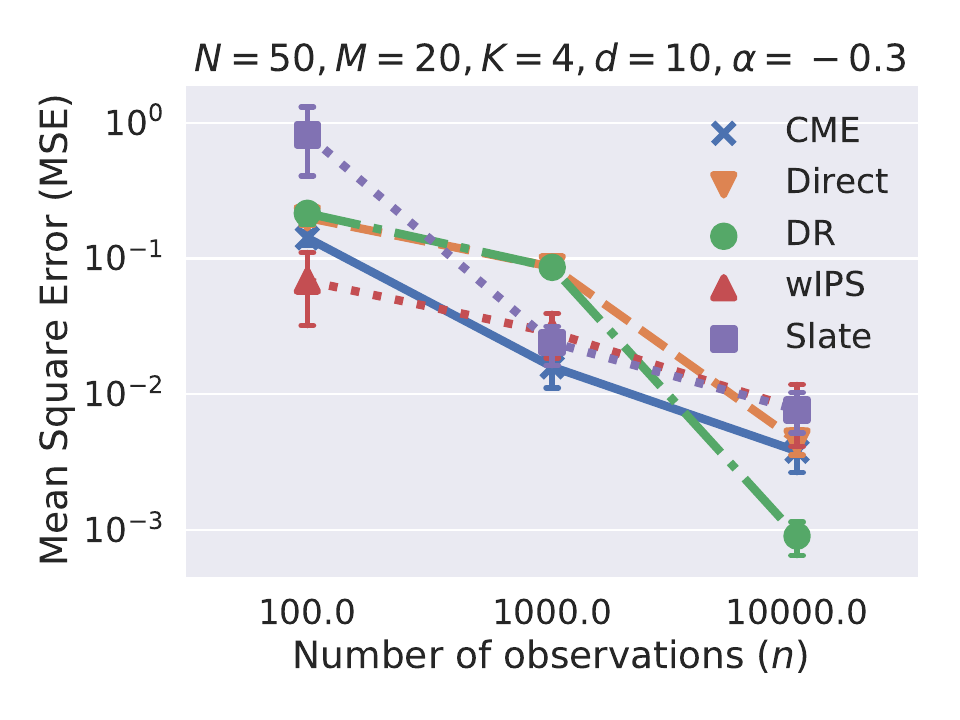}
  \label{subfig:sample-size}
}%
\vspace{-1.0em}
\caption{Mean square error (MSE) of the expected reward estimated by different estimators as we vary the value of \subref{subfig:policy-shift} the multiplier $\alpha$, \subref{subfig:context-dim} the context dimension $d$, \subref{subfig:reco-size} the number of available items $M$, \subref{subfig:user-size} the number of users $N$, \subref{subfig:recom-size} the number of recommended items $K$, and \subref{subfig:sample-size} the number of observations $n$. Each error bar represents a 95\% confidence interval.}
\label{fig:simulated-results}
\end{figure*}


In this experiment, we consider the following policy setup. For each user $j$, a policy $\pi ({\bf a} | {\bf u})$ generates a list ${\bf a} = (\mathbf{v}_{m_1}, \mathbf{v}_{m_2}, ..., \mathbf{v}_{m_K})$ of $K$ recommended items by sampling without replacement with respect to a multinomial distribution over all items. The probability of item $l \in \{1,\dots,M\}$ being selected for user $j$ is $\exp(\mathbf{b}_j^\top\mathbf{v}_l)/\sum_{k=1}^{M}\exp(\mathbf{b}_j^\top\mathbf{v}_k)$, where $\{\mathbf{b}_j\}_{j=1}^N$ are parameter vectors of the policy $\pi ({\bf a} | {\bf u})$.
Note that we obtain an optimal policy if $\mathbf{b}_j = \mathbf{u}_j$ for all $j \in \{1, \dots, N\}$. To construct initial policy $\pi_0  ({\bf a} | {\bf u})$ and target policy $\pi_*  ({\bf a} | {\bf u})$, we first randomly generate user feature vectors $\mathbf{u}_1,\ldots,\mathbf{u}_N$. Then, for the target policy $\pi_*$, we set $\mathbf{b}^*_j = \mathbf{p}_j^\top \mathbf{u}_j$ for $j=1,\ldots,N$ where $\mathbf{p}_j := (p_{jk})_{k=1}^d$ with $p_{jk} \sim \mathrm{Bernoulli}(0.5)$. That is, the parameter vector $\mathbf{b}^*_j$ is equal to the user feature vector with about half of its entries randomly set to zero. 
For the initial policy $\pi_0$, we set $\mathbf{b}_j = \alpha\mathbf{b}_j^*$ where $\alpha \in [-1,1]$.
The parameter $\alpha$ controls how similar the policies are. If $\alpha=1$, we obtain $\pi_0 = \pi_*$, whereas $\pi_0$ and $\pi_*$ differ the most when $\alpha=-1$. 

We generate two datasets $\mathcal{D}_{\mathrm{init}} = \{(\mathbf{u}_i,\mathbf{a}_i,r_i)\}_{i=1}^n$ and $\mathcal{D}_{\mathrm{target}} = \{(\mathbf{u}^*_i,\mathbf{a}^*_i, r^*_i)\}_{i=1}^n$ using $\pi_0({\bf a}|{\bf u})$ and $\pi_*({\bf a}|{\bf u})$, respectively, where $\mathbf{u}_i = \mathbf{u}_i^*$ for $	i=1,\ldots,n$. 
Note that the target rewards $r^*_1,r^*_2,\ldots,r^*_n$ are only used for evaluation. Our task is to estimate the expected reward of the target policy from the remaining information. 
We perform 5-fold CV over parameter grids, \ie, the number of hidden units $n_h \in \{50,100,150,200\}$ for the \texttt{Direct} and \texttt{DR} estimators, and the regularization parameter $\varepsilon\in\{10^{-8},\ldots,10^0\}$ for our \texttt{CME}.
We repeat the experiments 30 times independently to obtain the mean square errors (MSE) and their 95\% confidence intervals in the estimation of the expected reward for each estimator.


We investigate the behavior of different estimators as we vary different
experimental conditions including the degree of difference between initial and target policies ($\alpha$), the context dimensionality ($d$), the number of items ($M$), the number of users ($N$), the number of recommended items ($K$), and the number of observations ($n$). Figure \ref{fig:simulated-results} depicts the experimental results (note that vertical axis is in log scale). 
In brief, we find that 
\begin{enumerate*}[label=\itshape\alph*\upshape)]
\item the performance of all estimators degrade as the difference between $\pi_0$ and $\pi_*$ increases (\ie, as $\alpha$ tends to $-1$), but the \texttt{CME} is least susceptible to this difference,
\item the \texttt{Slate} estimator does not perform well in this setting because its linearity assumption does not hold,
\item all estimators deteriorate as the context dimension increases, but the effect appears to be more pronounced for the \texttt{Direct}, \texttt{DR}, and \texttt{CME} estimators than for the \texttt{wIPS} and \texttt{Slate} estimators as they do not rely directly on the covariates, \item the opposite effect is observed if we increase the number of available items $M$, as illustrated in Figure \ref{subfig:reco-size}, and
\item the \texttt{CME} estimator achieves better performance than other estimators in most experiments.
\end{enumerate*}


\subsubsection{Real Data}

For our real data experiment, we use the data from the Microsoft Learning to Rank Challenge
dataset (MSLR-WEB30K) \citep{Qin13:MSLR} and treat them as an off-policy evaluation problem. We follow the same experiment setting as described in \citet[Section 4.1]{Swaminathan17:Slate}. The data contains a set of queries and the corresponding URLs. 
Each pair of query $q$ and URL $u$ is represented by a feature vector $f_{q,u}$ and accompanied by a relevance judgment $\rho(q, u) \in \{0, ..., 4\}$. 
We consider the expected reciprocal rank (ERR) \citep{Chapelle09:ERR} as our reward function, which is defined as 
${\rm ERR}(q, u) := \sum_{k=1}^{K}\frac{1}{k}\prod_{j=1}^{k-1} (1 - R(q, u_j))R(q,u_k)$,
where $R(q,u) := \frac{2^{{\rho(q,u)}}-1}{2^{\mathrm{maxrel}}}$ with $\mathrm{maxrel} := 4$.
In order to obtain distinct initial and target policies $\pi_0({\bf a} | {\bf u})$ and $\pi_*({\bf a} | {\bf u})$, the feature vector $f_{q,u}$ is split into URL features $f_{q,u}^{\text{url}}$ and body features $f_{q,u}^{\text{body}}$, which are used to train two regression models to predict the relevance score $\rho(q, u)$: a Lasso regression model is trained to predict $\rho(q, u)$ from $f_{q,u}^{\text{url}}$ (denoted by $\mathrm{lasso}_{\text{url}}$), and a regression tree model is trained  to predict $\rho(q, u)$ from $f_{q,u}^{\text{body}}$ (denoted by $\mathrm{tree}_{\text{body}}$). These two regression models are then used in the initial and target policies as will be described below.


\begin{figure}[t!]
\vspace{-1.5em}
\centering
\includegraphics[width=0.95\textwidth]{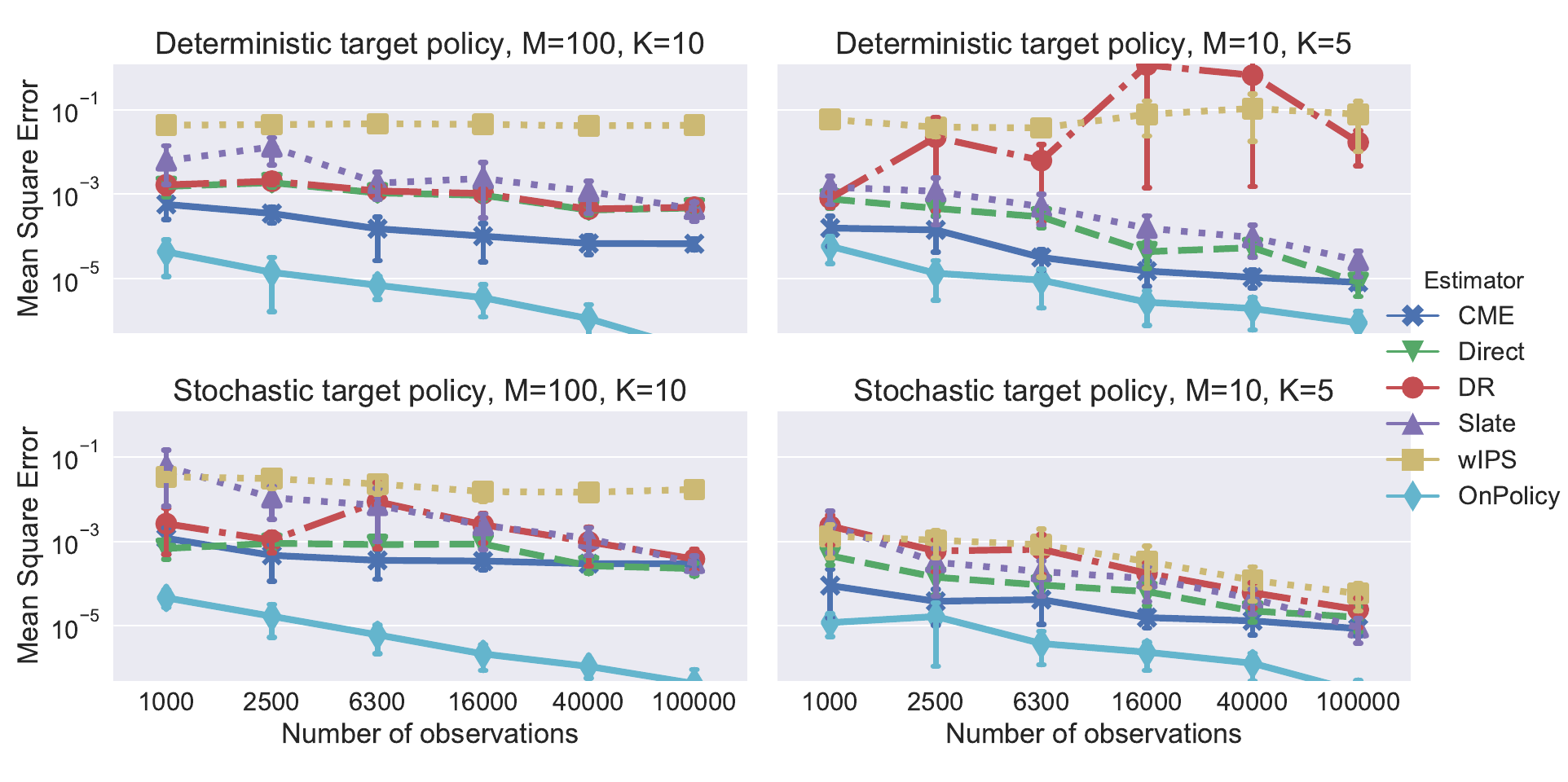}
\vspace{-1.0em}
\caption{The performance of different estimators on the MSLR-WEB30K dataset.}
\label{fig:real-results}
\end{figure}

In order to generate logged data $\mathcal{D}_{\text{init}}$, we first sample a query $q$ uniformly from the dataset, and select top $M$ candidate URLs based on the relevance scores predicted by the $\mathrm{tree}_{\text{body}}$ model. We then generate $K$ recommended URLs out of these top $M$ candidates using an initial policy $\pi_0$ and compute its corresponding reward. The initial policy is a {\em{stochastic}} policy which recommends $K$ URLs by sampling without replacement according to a multinomial distribution parameterized by $p_{\alpha}(u|q) \propto 2^{-\alpha[\log_2 \mathrm{rank}(u, q)]}$, where $\mathrm{rank}(u, q)$ is the ERR of the relevance score $\rho(q,u)$ predicted by the $\mathrm{tree}_{\text{body}}$ model and $\alpha \geq 0$ is an exploration rate parameter. For target data $\mathcal{D}_{\text{target}}$, we consider both {\em stochastic} and {\em deterministic} target policies. 
The {\em stochastic} target policy $\pi_*({\bf a} | {\bf u})$ is similar to the initial policy described earlier except that it employs $\mathrm{lasso}_{\text{url}}$ model for the predicted relevance scores, but this makes the two policies distinct; their top-10 rankings are only overlapping by 2.5 URLs, on average. On the other hand, the {\em deterministic} target policy directly selects top-K URLs ranked by the predicted relevance scores obtained from the $\mathrm{lasso}_{\text{url}}$ model. In this experiment, we set the exploration rate parameter $\alpha = 1$ for the stochastic initial policy, and set $\alpha = 2$ for the stochastic target policy.

We compare our estimator (\texttt{CME}) with the benchmark estimators \texttt{Direct}, \texttt{wIPS}, \texttt{DR} and \texttt{Slate}. In addition, we include the \texttt{OnPolicy} method as a baseline, which estimates rewards directly from the {\em target} policies (and thus, this baseline should always perform the best). 
To accelerate the computation of the \texttt{CME}, we make use of the Nystr\"om approximation method \citep{Williams01:Nystrom}.
We repeat the experiments 10 times independently to obtain the mean square errors (MSEs) and their 95\% confidence intervals in the estimation of the expected reward for each estimator. 


Figure \ref{fig:real-results} depicts the results. 
In short, our \texttt{CME} dominates other estimators in most of experimental conditions (note that the vertical axis is in log scale, so the margins are significantly large). The \texttt{wIPS} clearly suffers from high variance, especially in the deterministic target policy. The reason is that, in the deterministic setting, the propensity score adjustment requires that the treatments picked by logged and target policies match exactly. Otherwise, the propensity ratio $\pi_*({\bf a}|{\bf u})/\pi_0({\bf a}|{\bf u})$ will be zero. The exact match is likely not to happen in practice and leads to higher variance in estimation of rewards. 
The \texttt{Slate}, \texttt{Direct} and \texttt{CME} are relatively robust across different conditions. The \texttt{Direct} method and \texttt{CME} perform particularly well when sample size is small, regardless of the action space, while the \texttt{Slate} estimator is less sample-efficient, especially in the large action space.

\section{Discussion}
\label{sec:discussion}

This paper presents a general-purpose kernel mean representation of counterfactual distributions called the counterfactual mean embedding (CME). 
It draws insights and tools from kernel methods in machine learning and the potential outcome framework in causal inference. 
We show that our estimator of counterfactual distributions exhibits appealing theoretical properties, and also serves as a practical tool for causal inference. 
Ultimately, we hope that our work will be useful not only for researchers in disciplines that rely on the potential outcome framework, such as social and biomedical sciences, but also for researchers in machine learning and statistics to develop novel methodology for counterfactual inference, since several important open questions still remain: \eg, the use of high-order moments of counterfactual distributions, and how to handle a hidden confounder and an instrumental variable.

One promising application of our framework is in generating a sample from the counterfactual  distribution.
For instance, neuroscientists can visualize the fMRI images of subjects under alternative setups without explicitly conducting invasive experiments. In this case, the outcome variable corresponds to an fMRI image. Let $G_{\bm{\theta}}$ be a generative model over the outcome parametrized by a parameter vector $\bm{\theta}$. The choice of $G_{\bm{\theta}}$ can range from a mixture of Gaussians to deep generative models, \eg, generative adversarial networks (GAN). An estimate of the counterfactual distribution, denoted by $G^*_{\bm{\theta}}$, can be obtained via an optimization problem: $G_{\bm{\theta}}^* = \arg\min_{\bm{\theta}} \left\| \muh_{Y\langle 0|1 \rangle} - \muh_{G_{\bm{\theta}}}\right\|_{\hbspf}^2$, where $\muh_{Y\langle 0|1 \rangle}$ is our CME estimate and  $\muh_{G_{\bm{\theta}}}$ denotes the mean embedding of $G_{\bm{\theta}}$ in the RKHS $\hbspf$.  Counterfactual sample generation is ubiquitous for qualitative analysis in many application domains. We leave it as an open problem to future research.

\acks{
We express our sincere gratitude to the Action Editor and the anonymous reviewers, whose comments greatly helped improve the quality of the paper.
We also thank Ricardo Silva, Joris Mooij, Adith Swaminathan, Evan Robin, David Lopez-Paz, Wittawat Jitkrittum, Kenji Fukumizu, and Bernhard Sch\"olkopf for fruitful discussions. Krikamol Muandet would like to acknowledge fundings from the Thailand Research Fund Grant No. MRG6080206 and additional funding from the Faculty of
Science, Mahidol University.
Motonobu Kanagawa has been partially supported by the European Research Council (StG project PANAMA). 
Sorawit Saengkyongam is supported by a research grant (18968) from VILLUM FONDEN.
}

\newpage
\appendix



\section{Possible Extensions for Off-Policy Evaluation} \label{sec:OPE-extension}

Here we describe possible extensions of the proposed approach to the off-policy evaluation task described in Section \ref{sec:OPE-algorithm} (Algorithm \ref{alg:kpe}). 
As mentioned there, Algorithm \ref{alg:kpe} only generates one action ${\bf a}_j^*\sim \pi_*({\bf a} | {\bf u}_j^*)$ for each ${\bf u}_j^*$, which does not fully exploit the information of the target policy $\pi_*({\bf a} | {\bf u}_j^*)$.
We show a possible approach to using more information from the target policy, thereby improving the quality of the algorithm.

First, notice that the weight vector ${\bm \beta} \in \mathbb{R}^n$ in Algorithm \ref{alg:kpe} depends on the sample $({\bf x}_j' )_{j = 1}^m = ( {\bf u}_j^*, {\bf a}_j^* )_{j=1}^m$ only through the vector $\tilde{\bf K} {\bf 1}_m \in \mathbb{R}^m$, where $\tilde{\bf K} = (k(\x_i, \x'_j)) \in \mathbb{R}^{n \times m}$ and ${\bf 1}_m = \frac{1}{m}(1,\dots,1)^\top \in \mathbb{R}^{m}$.
This vector can be written as  
\begin{equation} \label{eq:OPE-vector-expansion}
\tilde{\bf K} {\bf 1}_m = ( \frac{1}{m} \sum_{j = 1}^m k(\x_i, \x'_j) )_{i=1}^n = ( \hat{\mu}_{X_1} (\x_i) )_{i=1}^n \in \mathbb{R}^n,
\end{equation}
where $(\x_i)_{i=1}^n = ({\bf u}_i, {\bf a}_i)_{i=1}^n$ are from the logged data, and $\hat{\mu}_{\pp{P}_{X_1}}$ is an empirical approximation of the mean embedding $\mu_{\pp{P}_{X_1}}$ of the covariate distribution $\pp{P}_{X_1}$, given by
\begin{equation} \label{eq:OPE-emp-kmean}
    \hat{\mu}_{X_1} = \frac{1}{m} \sum_{j=1}^m k(\cdot, {\bf x}_j'), \quad \mu_{X_1} =  \int k(\cdot, \x)d\pp{P}_{X_1}(\x).
\end{equation}
This implies that the role of the sample $({\bf x}_j' )_{j = 1}^m$ is essentially to approximate the kernel mean $\mu_{\pp{P}_{X_1}}$.  
In fact, the quality of the CME estimator (based on which Algorithm \ref{alg:kpe} is constructed) depends on the sample  $({\bf x}_j')_{j = 1}^m$ only through the the approximation error  $\| \hat{\mu}_{ X_1 }  - \mu_{ X_1 }  \|_\hbspf$ (see \eg, the proof of Theorem \ref{theo:estimation-error-rate}).

Thus, Algorithm \ref{alg:kpe} may be improved by constructing a better approximation, say $\check{\mu}_{ X_1 }$, of the kernel mean $\hat{\mu}_{ X_1 }$, and replace \eqref{eq:OPE-vector-expansion} in the computation of the weight vector ${\bm \beta}$ by the evaluations of this new approximation:
$$
{\bm \beta} := ({\bf K} + n \varepsilon {\bf I})^{-1}  {\bm v} \in \mathbb{R}^n, \quad {\bm v} := ( \check{\mu}_{X_1} (\x_i) ) \in \mathbb{R}^n,
$$
where ${\bf K} := (k(\x_i, \x_j))_{i,j = 1}^n \in \mathbb{R}^{n \times n}$.

To construct $\check{\mu}_{X_1}$, recall that the kernel $k$ is given as a product kernel $k(\x, \x') = k_{\mathcal{A}}({\bf a}, {\bf a}')   k_{\mathcal{U}}({\bf u}, {\bf u}')$ for $\x = ({\bf u}, {\bf a})$, $\x = ({\bf u}', {\bf a}')$, and rewrite the kernel mean $\mu_{X_1}$ as
\begin{align*}
 \mu_{X_1} (\x) = \mu_{\pi_*} ({\bf u}, {\bf a}) 
 & = \int_{\mathcal{U}}  \int_{\mathcal{A}}  k_{\mathcal{A}}({\bf a}, {\bf a}')   k_{\mathcal{U}}({\bf u}, {\bf u}')  \dd \pi_*({\bf a}' | {\bf u}')  \dd q_*({\bf u}' ) \\
& = \int_{\mathcal{U}}  \left( \int_{\mathcal{A}} k_{\mathcal{A}}({\bf a}, {\bf a}')  \dd\pi_*({\bf a}' | {\bf u}' ) \right)  k_{\mathcal{U}}({\bf u}, {\bf u}') \dd q_*({\bf u}')   \\
& \approx   \frac{1}{m} \sum_{j = 1}^m \left( \int_{\mathcal{A}} k_{\mathcal{A}}({\bf a}, {\bf a}')  \dd\pi_*({\bf a}' | {\bf u}^*_j ) \right)  k_{\mathcal{U}}({\bf u}, {\bf u}^*_j).
\end{align*}
Thus, if we can approximate the integral in the last expression accurately, we can obtain a good approximation for the kernel mean. 

One approach is to generate $M > 1$ actions  ${\bf a}_{j 1}, \dots, {\bf a}_{jM} \sim \pi_*({\bf a} | {\bf u}_j^*)$ from the target policy for each ${\bf u}_j^*$, and the approximate the integral as 
$$
 \frac{1}{M} \sum_{\nu = 1}^M k_{\mathcal{A}} (  {\bf a},  {\bf a}^*_{j\nu} ) \approx \int_{\mathcal{A}} k_{\mathcal{A}}({\bf a}, {\bf a}')  \dd\pi_*({\bf a}' | {\bf u}^*_j )   
$$
Thus, a new approximation of  $\mu_{X_1}$ may be defined as 
$$
\check{\mu}_{X_1}(\x) :=  \check{\mu}_{\pi_*} ({\bf u}, {\bf a}):=  \frac{1}{m} \sum_{j = 1}^m \left(  \frac{1}{M} \sum_{\nu = 1}^M k_{\mathcal{A}} (  {\bf a},  {\bf a}^*_{j\nu} )   \right)  k_{\mathcal{U}}({\bf u}, {\bf u}^*_j),
$$
where $\x := ({\bf u}, {\bf a})$.
Note that $M = 1$ recovers Algorithm \ref{alg:kpe}.

Another approach is to {\em exactly} compute the integral $\int_{\mathcal{A}} k_{\mathcal{A}}({\bf a}, {\bf a}')  \dd\pi_*({\bf a}' | {\bf u}^*_j )$ when it is possible, and define  a new approximation of  $\mu_{X_1}$ as
$$
\check{\mu}_{X_1}(\x) :=  \check{\mu}_{\pi_*} ({\bf u}, {\bf a}):= \frac{1}{m} \sum_{j = 1}^m \left( \int_{\mathcal{A}} k_{\mathcal{A}}({\bf a}, {\bf a}')  \dd\pi_*({\bf a}' | {\bf u}^*_j ) \right)  k_{\mathcal{U}}({\bf u}, {\bf u}^*_j)  
$$
This essentially is the case of $M = \infty$. For instance, the integral can be computed analytically when the kernel $k_{\mathcal{A}}$ is Gaussian and the target policy $\pi_*({\bf a} | {\bf u})$ is an additive Gaussian noise model of the form $\pi_*({\bf a} | {\bf u}) = \mathcal{N}( {\bf a} | F({\bf u}), \sigma^2 )$ for some function $F: \mathcal{U} \to \mathcal{A}$ and $\sigma^2 > 0$. This way of using an analytic integral for approximating the kernel mean is studied in \cite{nishiyama2020model}; we refer to this paper for details and other examples.

\section{Cross Validation Procedure for Counterfactual Prediction}
\label{sec:model-selection}

Here we describe an approach to cross validation for model selection in counterfactual prediction, focusing on the problem of off-policy evaluation (OPE).  We use below the notation defined in Section \ref{sec:policy-evaluation}. 
Unlike the standard situation in machine learning, performing cross validation directly on the logged data $\mathcal{D} := \{ ({\bf u}_i, {\bf a}_i, r_i) \}_{i=1}^n$ may lead to a biased estimate on the performance measure for counterfactual prediction, resulting in sub-optimal model selection.
This is due to the covariate shift -- the change of the covariate distribution from $\pi_0({\bf u}, {\bf a})$ in the data collection environment to $\pi_*({\bf u}, {\bf a})$ in the target environment. To correct for the bias due to this distributional shift, we propose the following procedure for cross validation. 
A similar cross validation approach has been proposed by \citet{Sugiyama07:CSA}.

Let $\mathcal{M} := \{ 1, \dots, M \}$ be a set of indicators for candidate models. (\eg, each $m \in \mathcal{M}$ may represent a specific choice of the kernel $k$ and regularization constant $\varepsilon$ in our CME estimator.)  
\begin{enumerate}
\item Split $\mathcal{D}$ into $K$ folds: $\mathcal{D}_k = \{({\bf u}_j,{\bf a}_j,r_j)\}_{j=q(k-1)+1}^{qk}$ for $k=1,\ldots,K$ and $q=\lfloor n/K \rfloor$.
\item For each model $m  \in \mathcal{M}$:
\begin{enumerate}
\item For each fold $k=1, 2, \ldots,K$:
\begin{enumerate}
\item Calculate $w_1,\dots,w_q \geq 0$ using propensity scores or covariate matching. 
\item Re-weight the validation reward $\hat{r}_k^* = \sum_{j=1}^{q} w_j r_{q(k-1)+j}$ (\textbf{bias correction}).
\item Use the remaining logged data $\mathcal{D}_{\neg k}$ and validation data $\{(\x^*_j,\s^*_j)\}_{j=q(k-1)+1}^{qk}$ to compute the estimated reward $\hat{r}_k$ and corresponding error $e_k = (\hat{r}_k - \hat{r}_k^*)^2$.
\end{enumerate}
\item Calculate the mean CV error $\bar{e}_p = \frac{1}{K}\sum_{k=1}^K e_k$ (\textbf{variance reduction}).
\end{enumerate}
\item Pick the $m$-th parameter setting whose $\bar{e}_m$ is the smallest.
\end{enumerate}

The algorithm above follows the standard cross validation procedure, except the bias correction step on validation sets. In the bias correction step, we re-weight the sample in the validation set so that the performance estimate computed from this set is unbiased. Nevertheless, the estimate may have high variance, \eg, when the propensity weights are used. This pitfall is alleviated by the variance reduction step.

\section{Proofs for Section \ref{sec:cme}}

\subsection{Proof of Lemma \ref{lemma:observable_outcome_dist}}
\label{sec:proof-lemma-observable-outcome-dist}

\begin{proof}
We only show $\pp{P}_{Y\left<0|0\right>}(\y)  = \pp{P}_{Y_0^* | T} (\y|0)$; the other identity $\pp{P}_{Y\left<1|1\right>}(\y)  = \pp{P}_{Y_1^* | T} (\y|1)$ can be shown similarly. 
First notice that $Y_0 := Y \,|\, (T = 0) = \sum_{t=0}^1 \mathbbm{1}(T=t) Y_t^* \,|\, (T=0) = Y_0^* \,|\, (T = 0)$.
From this and Definition \ref{def:random_variables}, it follows that 
$\pp{P}_{Y\left<0|0\right>}(\y)
= \int  \pp{P}_{Y_0 | X_0} (\y|\x) \dd\pp{P}_{X_0}(\x)
= \int \pp{P}_{Y | T, X} (\y|0,\x) \dd\pp{P}_{X|T}(\x|0) 
= \int \pp{P}_{Y_0^* | T, X} (\y|0,\x) \dd\pp{P}_{X|T}(\x|0) 
= \pp{P}_{Y_0^* | T} (\y|0)$.
\end{proof}

\subsection{Proof of Lemma \ref{lem:causal-interpret}}
\label{sec:proof-lemma-causal-interpret}

\begin{proof}
We only show here $\pp{P}_{Y\langle 0|1 \rangle}=\pp{P}_{Y_0^* | T = 1}$; the other identity $\pp{P}_{Y\langle 1|0 \rangle}=\pp{P}_{Y_1^* | T = 0}$ can be shown similarly. 
We first derive basic identities:
$(a)$ The conditional exogeneity implies that
$\pp{P}_{Y_0^* | T,X} (\y| 1,\x) = \pp{P}_{Y^*_0 | X}(\y|\x) = \pp{P}_{Y_0^* | T,X} (\y| 0,\x)$ for almost every $\x$ with respect to the distribution of $X$;
$(b)$ Recalling that $Y = \mathbbm{1}(T=0) Y^*_0 + \mathbbm{1}(T=1) Y^*_1$, we have $Y\,|\,(T=0) = Y^*_0 \,|\, (T=0)$;
$(c)$ By Definition \ref{def:random_variables}, we have
$ \pp{P}_{Y_0 | X_0}(\y | \x) 
= \pp{P}_{ (Y|T=0) | (X|T=0) }(\y | \x) 
= \pp{P}_{  Y | T,X }(\y |0, \x)$.
Using these, we have
\begin{eqnarray*}
\pp{P}_{Y^*_0 | T}(\y|1) 
&=& \int \pp{P}_{Y^*_0 | T, X}(\y | 1, \x) \dd \pp{P}_{X|T}(\x|1) \stackrel{(a)}{=} \int \pp{P}_{Y^*_0 | T, X}(\y | 0, \x) \dd \pp{P}_{X|T}(\x|1) \\
&\stackrel{(b)}{=}&  \int \pp{P}_{Y | T, X}(\y | 0, \x) \dd \pp{P}_{X|T}(\x|1)  \stackrel{(c)}{=} \int \pp{P}_{Y_0 | X_0}(\y | \x) \dd \pp{P}_{X_1}(\x) = \pp{P}_{Y\langle 0|1 \rangle}(\y),
\end{eqnarray*}
as required.
\end{proof}

\subsection{Proof of Theorem \ref{theo:propensity-unbiased}}

\label{sec:proof-propensity-unbiased}

\begin{proof}
We prove here $\mathbb{E} [  \hat{\mu}_{Y_1^*} ]  =   \mu_{Y_1^*}$; the proof of $\mathbb{E} [\hat{\mu}_{Y_0^*}] =  \mu_{Y_0^*}$ is similar and thus omitted. 
First note that, since $(\x_i,t_i, \y_i )_{i=1}^N$ are i.i.d.~with $(X,T,Y)$ and $m := \sum_{i=1}^N t_i$, we have 
$$
\mathbb{E} [  \hat{\mu}_{Y_1^*} ] =  \mathbb{E} \left[\frac{1}{m} \sum_{i = 1}^N \frac{ t_i \ell(\cdot, \y_i)  }{e(\x_i)} \right] = \mathbb{E} \left[\frac{ T \ell(\cdot, Y)  }{e(X)} \right].
$$
By the definition of $Y$, the conditional exogeneity and the definition of the propensity $e(\x)$, we then have
\begin{eqnarray*}
&&\mathbb{E} \left[\frac{ T \ell(\cdot, Y)  }{e(X)} \right] = \mathbb{E} \left[\frac{ T \ell(\cdot, Y^*_1)  }{e(X)} \right] = \mathbb{E}_X \left[ \mathbb{E} \left[\frac{ T \ell(\cdot, Y^*_1)  }{e(X)} \Big| X \right] \right] \\
&=& \mathbb{E}_X \left[  \frac{ \mathbb{E}[ T | X] \mathbb{E}[ \ell(\cdot, Y^*_1) |X]  }{e(X)}  \right] = \mathbb{E}_X \left[  \frac{ e(X) \mathbb{E}[ \ell(\cdot, Y^*_1) | X]  }{e(X)}  \right] \\
&=& \mathbb{E}_X \left[    \mathbb{E}[ \ell(\cdot, Y^*_1) |X]    \right]  = \mathbb{E}[  \ell(\cdot, Y^*_1)  ] = \mu_{Y_1^*},
\end{eqnarray*}
where $\mathbb{E}_X$ denotes the expectation with respect to $X$.
\end{proof}

\subsection{Proof of Theorem \ref{theo:conv-rate-propensity}}

\label{sec:proof-conv-rate-propensity}.

\begin{proof}
We derive the convergence rate of $\hat{\mu}_{Y_1^*}$; the rate of $\hat{\mu}_{Y_0^*} $ can be derived in a similar way, and thus is omitted.   
Let $\tilde{Y}_1^*$ be an independent copy of $Y_1^*$. 
\begin{eqnarray}
&& \mathbb{E}\left[ \left\| \hat{\mu}_{Y_1^*} - \mu_{Y_1^*} \right\|_{\hbspf}^2 \right] = \mathbb{E}\left[ \left\|  \frac{1}{m} \sum_{i = 1}^N \frac{ t_i \ell(\cdot, \y_i)  }{e(\x_i)}  - \mu_{Y_1^*} \right\|_{\hbspf}^2 \right] \nonumber  \\
&=& \underbrace{\mathbb{E}\left[  \frac{1}{m^2} \sum_{i,j} \frac{ t_i t_j \ell(\y_i, \y_j)  }{e(\x_i) e(\x_j)}   \right]}_{(A)} - 2 \underbrace{ \mathbb{E}\left[  \frac{1}{m} \sum_{i} \frac{ t_i \mathbb{E}_{Y_1^*} [\ell(\y_i, Y_1^*) ]  }{e(\x_i)}   \right] }_{(B)} + \mathbb{E}[ \ell(Y_1^*, \tilde{Y}_1^*) ].  \label{eq:proof-propensity-main}
\end{eqnarray}
We first deal with the term $(A)$. It can be expanded as
\begin{eqnarray*}
(A) =  \frac{1}{m^2} \mathbb{E}\left[   \sum_{i \not= j} \frac{ t_i t_j \ell(\y_i, \y_j)  }{e(\x_i) e(\x_j)}   \right] +   \frac{1}{m^2} \mathbb{E}\left[   \sum_{i} \frac{ t_i^2 \ell(\y_i, \y_i)  }{e^2(\x_i) }   \right]. 
\end{eqnarray*}
Let $(\tilde{X}, \tilde{T}, \tilde{Y}_0^*, \tilde{Y}_1^*)$ be an independent copy of $(X,T,Y_0^*, Y_1^*)$, and write $\tilde{Y} = \mathbbm{1}(\tilde{T}=0) \tilde{Y}_0^* + \mathbbm{1}(\tilde{T}=1) \tilde{Y}_1^*$. 
Since in the first term of $(A)$, $(\x_i,t_i, \y_i)$ and $(\x_j,t_j, \y_j)$  are independent but distributed as $(X,T,Y)$, the first term can be written as 
\begin{eqnarray*}
&& \frac{1}{m^2} \mathbb{E}\left[   \sum_{i \not= j} \frac{ T \tilde{T} \ell(Y, \tilde{Y})  }{e(X) e(\tilde{X})}   \right] =  \frac{m-1}{m}\mathbb{E}\left[   \frac{ T \tilde{T} \ell(Y, \tilde{Y})  }{e(X) e(\tilde{X})}   \right].
\end{eqnarray*}
For the right hand side, we have 
\begin{eqnarray*}
&& \mathbb{E}\left[   \frac{ T \tilde{T} \ell(Y, \tilde{Y})  }{e(X) e(\tilde{X})}   \right] 
\stackrel{(a)}{=} \mathbb{E}\left[   \frac{ T \tilde{T} \ell(Y_1^*, \tilde{Y}_1^*)  }{e(X) e(\tilde{X})}   \right] =  \mathbb{E}_{ \tilde{X}, \tilde{T}, \tilde{Y}_1^* } \left[ \frac{\tilde{T}}{e(\tilde{X})} \mathbb{E}_{X,T,Y_1^*}\left[   \frac{ T \ell(Y_1^*, \tilde{Y}_1^*)  }{e(X) }  \mid  \tilde{Y}_1^*   \right]  \right] \\
&=& \mathbb{E}_{ \tilde{X}, \tilde{T}, \tilde{Y}_1^* } \left[ \frac{\tilde{T}}{e(\tilde{X})} \mathbb{E}_X \left[   \frac{ \mathbb{E}_{T,Y_1^*}[ T\ell(Y_1^*, \tilde{Y}_1^*) \mid X  ]  }{e(X) }  \mid   \tilde{Y}_1^*   \right]  \right] \\  
&\stackrel{(b)}{=} & \mathbb{E}_{ \tilde{X}, \tilde{T}, \tilde{Y}_1^* } \left[ \frac{\tilde{T}}{e(\tilde{X})} \mathbb{E}_X \left[   \frac{ \mathbb{E}_T[ T \mid X  ] \mathbb{E}_{Y_1^*}[ \ell(Y_1^*, \tilde{Y}_1^*) \mid X  ]  }{e(X) }  \mid   \tilde{Y}_1^*   \right]  \right]  \\
&\stackrel{(c)}{=} & \mathbb{E}_{ \tilde{X}, \tilde{T}, \tilde{Y}_1^* } \left[ \frac{\tilde{T}}{e(\tilde{X})} \mathbb{E}_X \left[  \mathbb{E}_{Y_1^*}[ \ell(Y_1^*, \tilde{Y}_1^*) \mid X  ]   \mid  \tilde{Y}_1^*  ] \right]  \right] = \mathbb{E}_{ \tilde{X}, \tilde{T}, \tilde{Y}_1^* } \left[ \frac{\tilde{T}}{e(\tilde{X})}  \mathbb{E}_{Y_1^*}[ \ell(Y_1^*, \tilde{Y}_1^*)  ]   \mid  \tilde{Y}_1^* ]   \right] \\
&=& \mathbb{E}_{ \tilde{X} } \left[ \frac{1}{e(\tilde{X})}  \mathbb{E}_{\tilde{T}, \tilde{Y}_1^* } \left[  \tilde{T}  \mathbb{E}_{Y_1^*}[ \ell(Y_1^*, \tilde{Y}_1^*)  ]   \mid  \tilde{Y}_1^* ] \mid \tilde{X}  \right]  \right] \stackrel{(d)}{=}  \mathbb{E}_{ \tilde{X} } \left[ \frac{1}{e(\tilde{X})}   \mathbb{E}_{\tilde{T}} [ \tilde{T} | \tilde{X}]  \mathbb{E}_{Y_1^*, \tilde{Y}_1^*}[ \ell(Y_1^*, \tilde{Y}_1^*)   \mid \tilde{X} ]    \right] \\
&\stackrel{(e)}{=}& \mathbb{E}_{ \tilde{X} } \left[ \mathbb{E}_{Y_1^*, \tilde{Y}_1^*}[ \ell(Y_1^*, \tilde{Y}_1^*)   \mid \tilde{X} ]    \right] = \mathbb{E}_{Y_1^*, \tilde{Y}_1^*}[ \ell(Y_1^*, \tilde{Y}_1^*) ] .
\end{eqnarray*}
where $(a)$ follows from the definitions of $Y$ and $\tilde{Y}$, $(b)$ from the conditional exogeneity, $(c)$ from $\mathbb{E}[T\mid X] = e(X)$, $(d)$ from the conditional exogeneity, and $(e)$ from $\mathbb{E}[\tilde{T} \mid \tilde{X}] = e(\tilde{X})$.
On the other hand,  the second term of $(A)$ can be written as
\begin{eqnarray*}
&& \frac{1}{m^2} \mathbb{E}\left[   \sum_{i} \frac{ t_i^2 \ell(\y_i \y_i)  }{e^2(\x_i) }   \right] 
 = \frac{1}{m} \mathbb{E}\left[  \frac{ T^2 \ell(Y, Y)  }{e^2(X) }   \right].
\end{eqnarray*}
For the last expression, we have
\begin{eqnarray*}
&& \mathbb{E}\left[  \frac{ T^2 \ell(Y, Y)  }{e^2(X) }   \right] 
\stackrel{(a)}{=}  \mathbb{E}\left[  \frac{ T \ell(Y, Y)  }{e^2(X) }   \right] 
\stackrel{(b)}{=} \mathbb{E}\left[  \frac{ T \ell(Y_1^*, Y_1^*)  }{e^2(X) }   \right] \\
&=& \mathbb{E}_X\left[  \frac{ \mathbb{E}_{T,Y_1^*}[ T \ell(Y_1^*, Y_1^*) \mid X]  }{e^2(X) }   \right] 
\stackrel{(c)}{=}  \mathbb{E}_X\left[  \frac{ \mathbb{E}_{T}[ T |X ]  \mathbb{E}_{Y_1^*}[ \ell(Y_1^*, Y_1^*) \mid X]  }{e^2(X) }   \right] \\
&\stackrel{(d)}{=} & \mathbb{E}_X\left[  \frac{  \mathbb{E}_{Y_1^*}[ \ell(Y_1^*, Y_1^*) \mid X]  }{e(X) }   \right] \leq \frac{\sup_{\y \in \mathcal{Y}} \ell(\y,\y)}{ \inf_{\x \in \mathcal{X}} e(x)  } = : C_{\ell, e} \stackrel{(e)}{<}  \infty,
\end{eqnarray*}
where $(a)$ follows from $T$ taking values in $\{ 0, 1\}$, $(b)$ from $Y$ being $Y_1^*$ if $T = 1$, $(c)$ from the conditional exogeneity, $(d)$ from $\mathbb{E}[T|X] = e(X)$, and $(e)$ from our assumptions that $\sup_{\y \in \mathcal{Y}} \ell(\y,\y) < \infty$ and  $ \inf_{\x \in \mathcal{X}} e(x) > 0$.
Thus, the term $(A)$ is upper-bounded as
\begin{equation*}
(A) \leq  \frac{m-1}{m} \mathbb{E}_{Y_1^*, \tilde{Y}_1^*}[ \ell(Y_1^*, \tilde{Y}_1^*) ] + \frac{1}{m}  C_{\ell, e}  .
\end{equation*}

Next we deal with the term $(B)$, which can be written as
\begin{eqnarray*}
\mathbb{E}\left[  \frac{1}{m} \sum_{i} \frac{ t_i \mathbb{E}_{Y_1^*} [\ell(\y_i, Y_1^*) ]  }{e(\x_i)}   \right]
 = \mathbb{E}\left[  \frac{ \tilde{T} \mathbb{E}_{Y_1^*} [\ell(\tilde{Y}, Y_1^*) ]  }{e(\tilde{X})}   \right],
\end{eqnarray*}
where, as before, $(\tilde{X}, \tilde{T}, \tilde{Y}_0^*, \tilde{Y}_1^*)$ is an independent copy of $(X,T,Y_0^*, Y_1^*)$ and $\tilde{Y} := \mathbbm{1}(\tilde{T}=0) \tilde{Y}_0^* + \mathbbm{1}(\tilde{T}=1) \tilde{Y}_1^*$ . 
The right expression can be expanded as 
\begin{eqnarray*}
&& \mathbb{E}\left[  \frac{ \tilde{T} \mathbb{E}_{Y_1^*} [\ell(\tilde{Y}, Y_1^*) ]  }{e(\tilde{X})}   \right] 
 \stackrel{(a)}{=}    \mathbb{E}\left[  \frac{ \tilde{T} \mathbb{E}_{Y_1^*} [\ell(\tilde{Y}_1^*, Y_1^*) ]  }{e(\tilde{X})}   \right] 
=  \mathbb{E}_{\tilde{X}}\left[  \frac{ \mathbb{E}_{ \tilde{T}, \tilde{Y}_1^* } \left[ \tilde{T} \mathbb{E}_{Y_1^*} [\ell(\tilde{Y}_1^*, Y_1^*) ] \mid \tilde{X}  \right] }{e(\tilde{X})}   \right] \\
 & \stackrel{(b)}{=}   &  \mathbb{E}_{\tilde{X}}\left[  \frac{ \mathbb{E}_{ \tilde{T} } [ \tilde{T} \mid \tilde{X} ]   \mathbb{E}_{Y_1^*, \tilde{Y}_1^* } [\ell(\tilde{Y}_1^*, Y_1^*) ] \mid \tilde{X}  ] }{e(\tilde{X})}   \right] 
  \stackrel{(c)}{=}  \mathbb{E}_{\tilde{X}}\left[   \mathbb{E}_{Y_1^*, \tilde{Y}_1^* } [\ell(\tilde{Y}_1^*, Y_1^*) ] \mid \tilde{X}  ] \right]
 =    \mathbb{E}_{Y_1^*, \tilde{Y}_1^* } [\ell(\tilde{Y}_1^*, Y_1^*) ],
\end{eqnarray*}
where $(a)$ follows from $\tilde{Y}$ being $\tilde{Y}_1^*$ if $\tilde{T} = 1$, $(b)$ from the conditional exogeneity and $(c)$ from $\mathbb{E}_{ \tilde{T} } [ \tilde{T} \mid \tilde{X} ] = e(\tilde{X})$.

Using the obtained results for $(A)$ and $(B)$ in \eqref{eq:proof-propensity-main}, we now have
\begin{eqnarray*}
 \mathbb{E}\left[ \left\| \hat{\mu}_{Y_1^*} - \mu_{Y_1^*} \right\|_{\hbspf}^2 \right] 
&\leq&  \frac{m-1}{m} \mathbb{E}[ \ell(Y_1^*, \tilde{Y}_1^*) ] + \frac{1}{m} C_{\ell, e} - 2  \mathbb{E} [\ell(Y_1^*, \tilde{Y}_1^*) ] + \mathbb{E}[ \ell(Y_1^*, \tilde{Y}_1^*) ] \\
&=& \frac{1}{m} \left( C_{\ell, e}  - \mathbb{E} [\ell(Y_1^*, \tilde{Y}_1^*) ]   \right),
\end{eqnarray*}
which completes the proof. 
 \end{proof}

\section{Preliminaries to the Proofs for Section \ref{sec:theory-main}}  
\label{sec:theory-preliminary}
We collect here preliminary results required for proving the theoretical results in Section \ref{sec:theory-main}. 
Thus, the interested reader may first look at Appendix \ref{sec:proofs-convergence}, where the proofs for the main theoretical results are presented. 
We use here the notation and basic definitions provided in Section \ref{sec:theory-main} of the main body.
In Appendix \ref{sec:app-preliminary}, we introduce certain integral operators and collect basic facts regarding them.
Based on them, in Appendix \ref{sec:lemmas} we present various lemmas needed for the proofs of the convernce results in Appendix \ref{sec:proofs-convergence}.

\subsection{Integral Operators} \label{sec:app-preliminary}

To be rigorous, we employ the following notation used in \cite{SteSco12}: For a measurable function $f: \mathcal{X} \to \mathbb{R}$, $[f]_\sim$ denotes the class of measurable functions that are $\pp{P}_{X_0}$-equivalent to $f$:
$$
[f]_\sim := \left\{ g :\mathcal{X} \to \mathbb{R} \mid \ \pp{P}_{X_0}(\{ x \in \mathcal{X} \mid f(x) \neq g(x)\}) = 0 \right\}.
$$

\paragraph{Integral operators.}
Define three integral operators $T: L_2(\pp{P}_{X_0}) \to L_2(\pp{P}_{X_0})$, $S: L_2(\pp{P}_{X_0}) \to \hbspace$ and $\covx: \hbspace \to \hbspace$ by 
\begin{align}  
 Tf &:= \int k(\cdot,\x) f(\x) \dd \pp{P}_{X_0}(\x) \in L_2(\pp{P}_{X_0}), & f \in L_2(\pp{P}_{X_0}), & \label{eq:integral} \\
 Sf &:= \int k(\cdot,\x) f(\x) \dd \pp{P}_{X_0}(\x)\in \hbspace, & f \in L_2(\pp{P}_{X_0}), & \label{eq:op-S-def} \\
 \covx f  &:= \int k(\cdot,\x) f(\x) \dd \pp{P}_{X_0}(\x) \in \hbspace, & f \in \hbspace & \label{eq:cov-ope-def-app}.
\end{align}
Note that while these operators look similar, they are different in their domains and ranges.
In particular, $\covx$ is the covariance operator.
Under  Assumption \ref{as:g-and-theta} \ref{as:measruable}, \ie, $\sup_{\x \in \mathcal{X}} k(\x, \x) < \infty$, \citet[Lemma 2.3]{SteSco12} implies that the operator $S^*: \hbspace \to L_2(\pp{P}_{X_0})$ defined by
\begin{equation*}
S^* g = [g]_\sim, \quad g \in \hbspace
\end{equation*}
is compact, and thus continuous.
This operator $S^*$ is the adjoint of the operator $S$ defined in \eqref{eq:op-S-def}.
Since $S^*$ is continuous, by \citet[Lemma 2.3]{SteSco12}, the operators $T$ and $\covx$ can be written as 
\begin{equation*}
T = S^* S, \quad \covx = S S^* .
\end{equation*}

The following lemma summarizes conditions required for eigen-decompositions of \eqref{eq:integral}, \eqref{eq:op-S-def} and \eqref{eq:cov-ope-def-app}.
In the sequel, ``ONS'' and ``ONB'' mean ``orthonormal series'' and ``orthonormal basis,'' respectively.
The set $I \subset \mathbb{N}$ is a set of indices, which is finite if the RKHS $\hbspace$ is finite dimensional, and infinite if $\hbspace$ is infinite dimensional.
\begin{lemma}[Spectral decomposition of integral operators] \label{lemma:op-eigen-decomp}
Let $\mathcal{X}$, $k$ and $\pp{P}_{X_0}$ be such that Assumption \ref{as:g-and-theta}~\ref{as:measruable} is satisfied.
Then there exist at most countable families $(e_i)_{i \in I} \subset \hbspace$ and $(\mu_i)_{i \in I} \subset (0,\infty)$ such that $\mu_1 \geq \mu_2 \geq \cdots > 0$, $(\mu_i^{1/2} e_i)_{i \in I}$ is an ONS in $\hbspace$, $( [e_i]_\sim )_{i \in I}$ is an ONS in $L_2(\pp{P}_{X_0})$, and 
\begin{align} \label{eq:eigen_integral}
 T f &= \sum_{i \in I} \mu_i \left< [e_i]_\sim, f \right>_{L_2(\pp{P}_{X_0})} [e_i]_\sim, & f \in L_2(\pp{P}_{X_0}), \\ 
 S f &= \sum_{i \in I} \mu_i \left< [e_i]_\sim, f  \right>_{L_2(\pp{P}_{X_0})} e_i, \quad & f \in L_2(\pp{P}_{X_0}), \label{eq:eigen-integral-S} \\
 \covx g &= \sum_{i \in I} \mu_i \left< \mu_i^{1/2} e_i, g \right>_{\hbspace} \mu_i^{1/2} e_i,  & g \in \hbspace, \label{eq:cov_op_eig}
\end{align}
where the convergence is in $ L_2(\pp{P}_{X_0})$ for \eqref{eq:eigen_integral}, and in $\hbspace$ for \eqref{eq:eigen-integral-S} and \eqref{eq:cov_op_eig}.
\end{lemma}
\begin{proof}
Since $k$ and $\pp{P}_{X_0}$ satisfy  Assumption \ref{as:g-and-theta} \ref{as:measruable}, it follows from \citet[Lemma 2.3]{SteSco12} that $\hbspace$ is compactly embedded into $L_2(\pp{P}_{X_0})$. 
As a result, \citet[Lemma 2.12]{SteSco12} implies that there exist at most countable families $(e)_{i \in I} \subset \hbspace$ and $(\mu_i)_{i \in I} \subset (0,\infty)$ such that  $\mu_1 \geq \mu_2 \geq \cdots > 0$, $( [e_i]_\sim )_{i \in I}$ is an ONS in $L_2(\pp{P}_{X_0})$, $(\mu_i^{1/2} e_i)_{i \in I}$ is an ONS in $\hbspace$, and \eqref{eq:eigen_integral} holds with convergence in $L_2(\pp{P}_{X_0})$.

To show \eqref{eq:eigen-integral-S}, since $([e_i]_\sim)_{i \in I}$ is an ONS in $L_2(\Pz)$, any $f \in L_2(\Pz)$ can be written as
$$
f = \sum_{i \in I}  \left< [e_i]_\sim, f  \right>_{L_2(\pp{P}_{X_0})} [e_i]_\sim + f^\perp,
$$
with convergence in $L_2(\Pz)$, where $f^\perp \in L_2(\Pz)$ is such that $\left< [e_i]_\sim, f^\perp \right>_{L_2(\Pz)} = 0$ for all $i \in I$.
By \citet[Lemma 2.12, Eq.15]{SteSco12} we have $\mu_i e_i = S [e_i]_{\sim}$ for all $i \in I$. It then holds that
$$
Sf = \sum_{i \in I} \mu_i \left< [e_i]_\sim, f  \right>_{L_2(\pp{P}_{X_0})} e_i + S f^\perp,
$$
where the convergence is in $\hbspace$ since $S$ is continuous.
Note that we have $T f^\perp = 0$, since we have \eqref{eq:eigen_integral} and $\left< [e_i]_\sim, f^\perp \right>_{L_2(\Pz)} = 0$ for all $i \in I$.
That is, $f^\perp$ is in the null space of $T$.
Since the null spaces of $S$ and $T$ are equal \cite[Lemma 2.12, Eq.16]{SteSco12}, it follows that $S f^\perp = 0$, which implies \eqref{eq:eigen-integral-S}.

Finally we show \eqref{eq:cov_op_eig}.
First note that $\covx e_i = SS^* e_i = S[e_i]_\sim = \mu e_i$ for all $i \in I$.
Using this and \eqref{eq:eigen-integral-S}, for any $g \in \hbspace$ we have
\begin{eqnarray*}
\covx g &=& S S^* g =  \sum_{i \in I} \mu_i \left< [e_i]_\sim, S^* g  \right>_{L_2(\pp{P}_{X_0})} e_i =  \sum_{i \in I} \mu_i \left< S S^* e_i,  g  \right>_{L_2(\pp{P}_{X_0})} e_i\\
&=&  \sum_{i \in I} \mu_i \left< \covx e_i,  g  \right>_\hbspace e_i = \sum_{i \in I} \mu_i \left< \mu_i e_i,  g  \right>_\hbspace e_i,
\end{eqnarray*}
where the convergence is in $\hbspace$, which implies \eqref{eq:cov_op_eig}.
\end{proof}

\begin{definition} \label{def:power-operators}
Let $\mathcal{X}$, $k$ and $\pp{P}_{X_0}$ be such that Assumption \ref{as:g-and-theta}~\ref{as:measruable} is satisfied.
Let $(e_i)_{i \in I} \subset \hbspace$ and $(\mu_i)_{i \in I} \subset (0,\infty)$ be as in Lemma \ref{lemma:op-eigen-decomp}.
Then for a constant $\beta > 0$, the $\beta$-th power of $T$, $S$ and $\covx$ are respectively defined by
\begin{align*}
T^\beta f &:= \sum_{i \in I} \mu_i^\beta \left< [e_i]_\sim, f \right>_{L_2(\pp{P}_{X_0})} [e_i]_\sim, \quad  &f \in L_2(\pp{P}_{X_0}), \\
S^\beta f &:= \sum_{i \in I} \mu_i^\beta \left< [e_i]_\sim, f \right>_{L_2(\pp{P}_{X_0})} e_i, \quad & f \in L_2(\pp{P}_{X_0}), \\
\covx^\beta f &:= \sum_{i \in I} \mu_i^\beta \left< \mu_i^{1/2} e_i, f \right>_\hbspace \mu_i^{1/2} e_i, \quad & f \in \hbspace.
\end{align*}
\end{definition}



\begin{lemma} \label{lemma:ONB-mercer-iso}
Let $\mathcal{X}$, $k$ and $\pp{P}_{X_0}$ be such that Assumption \ref{as:g-and-theta}~\ref{as:measruable} and \ref{as:RKHS-dense} hold.
Let $(e_i)_{i \in I} \subset \hbspace$ and $(\mu_i)_{i \in I} \subset (0,\infty)$ be as in Lemma \ref{lemma:op-eigen-decomp}.
Then, $([e]_\sim)_{i \in I}$ is an ONB of $L_2(\Pz)$.
\end{lemma}
\begin{proof}
Since $k$ and $\pp{P}_{X_0}$ satisfy Assumption \ref{as:g-and-theta} \ref{as:measruable}, it follows from \citet[Lemma 2.3]{SteSco12} that $\hbspace$ is compactly embedded into $L_2(\pp{P}_{X_0})$. 
Then one can use \citet[Theorem 3.1]{SteSco12}, which states that the assertion is equivalent to the Assumption \ref{as:g-and-theta}~\ref{as:RKHS-dense} that the embedding $S^*: \hbspace \to L_2(\Pz)$ has a dense image in $L_2(\Pz)$. 
\end{proof}

We provide a condition for the integral operator $T_{\rm prod}$ defined in Section \ref{eq:int-op-joint} to admit an eigen-decomposition, which is needed for its power $T_{\rm prod}^\beta$ in \eqref{eq:power-joint-int-op} to be well-defined.

\begin{lemma} \label{lemma:eig-decomp-joint-int-op}
Let $\mathcal{X}$, $k$ and $\pp{P}_{X_0}$ be such that Assumption \ref{as:g-and-theta}~\ref{as:measruable} and \ref{as:RKHS-dense} are satisfied.
Let $T_{\rm prod} : L_2(\pp{P}_{X_0} \otimes \pp{P}_{X_0}) \to  L_2(\pp{P}_{X_0} \otimes \pp{P}_{X_0})$ be the integral operator defined as in Section \ref{eq:int-op-joint}, and let $(e_i)_{i \in I} \subset \hbspace$ and $(\mu_i)_{i \in I} \subset (0,\infty)$ be as in Lemma \ref{lemma:op-eigen-decomp}. 
Then we have
$$
T_{\rm prod} \eta = \sum_{i,j \in I} \mu_i \mu_j \left< \eta, [e_i]_\sim \otimes [e_j]_\sim \right>_{L_2(\pp{P}_{X_0} \otimes \pp{P}_{X_0})} [e_i]_\sim \otimes [e_j]_\sim, \quad \eta \in L_2(\pp{P}_{X_0} \otimes \pp{P}_{X_0}),
$$
where the convergence is in $L_2(\pp{P}_{X_0} \otimes \pp{P}_{X_0})$. 
\end{lemma}
\begin{proof} 
By Assumption \ref{as:g-and-theta}~\ref{as:RKHS-dense} that $S^*$ has a dense image in $L_2(\Pz)$ and Lemma \ref{lemma:ONB-mercer-iso},  $([e_i]_\sim)_{i \in I}$ is an ONB in $L_2(\Pz)$. 
This implies that $([e_i]_\sim \otimes [e_j]_\sim)_{i,j \in I}$ is an ONB in $L_2(\Pz \otimes \Pz)$; see {\it e.g.},~\citet[Ex.~61, p.178]{Fol99}.
Therefore any $\eta \in L_2(\Pz \otimes \Pz)$ can be written as 
$$
\eta = \sum_{i,j \in I}  \left< \eta, [e_i]_\sim \otimes [e_j]_\sim \right>_{L_2(\pp{P}_{X_0} \otimes \pp{P}_{X_0})} [e_i]_\sim \otimes [e_j]_\sim 
$$
with convergence in $L_2(\pp{P}_{X_0} \otimes \pp{P}_{X_0})$.
Note that $[e_i]_\sim \otimes [e_j]_\sim$ for any $i,j \in I$ is an eigenfunction of $T_{\rm prod}$ with the corresponding eigenvalue being $\mu_i \mu_j$, since
\begin{eqnarray*}
T_{\rm prod} ([e_i]_\sim \otimes [e_j]_\sim) 
&=& \int k (\cdot,\x) [e_i]_\sim (\x) \dd\Pz(\x) \otimes \int k(\cdot, \tilde{\x}) [e_j]_\sim (\tilde{\x}) \dd\Pz(\tilde{\x}) \\
&=& (T[e_i]_\sim) \otimes (T[e_j]_\sim) = \mu_i \mu_i [e_i]_\sim \otimes [e_j]_\sim.
\end{eqnarray*}
The assertion follows from this and the above eigendecomposition of $\eta$ in Assumption \ref{as:range-assumption-theta}.
\end{proof}

As a direct corollary of  Lemma \ref{lemma:eig-decomp-joint-int-op}, we have the following result, which provides an eigenbasis expression of the range assumption $\theta \in {\rm Range}(T_{\rm prod}^\beta)$.
\begin{corollary} \label{cor:range-as-eigen-basis}
Let $\mathcal{X}$, $k$ and $\pp{P}_{X_0}$ be such that Assumption \ref{as:g-and-theta}~\ref{as:measruable} and \ref{as:RKHS-dense} are satisfied, and let $(e_i)_{i \in I} \subset \hbspace$ and $(\mu_i)_{i \in I} \subset (0,\infty)$ be as in Lemma \ref{lemma:op-eigen-decomp}.
Suppose that Assumption \ref{as:range-assumption-theta} is satisfied, \ie, $\theta \in {\rm Range}(T_{\rm prod}^\beta)$ holds for $\theta \in L_2(\Pz \otimes \Pz)$ and $0 \leq \beta \leq 1$, where $T_{\rm prod}^\beta$ is defined in \eqref{eq:power-joint-int-op}.
Then there exist $(a_{i,j})_{i,j \in I} \subset \mathbb{R}$ such that $\sum_{i,j \in I} a_{i,j}^2 < \infty$ and 
$$
\theta = \sum_{i,j \in I} a_{i,j} (\mu_i \mu_j)^\beta [e_i]_\sim \otimes [e_j]_\sim, 
$$
where the convergence is in $L_2(\Pz \otimes \Pz)$.
\end{corollary}
\begin{proof}
The assumption $\theta \in {\rm Range}(T_{\rm prod}^\beta)$ implies that there exists some $\eta \in L_2(\Pz \otimes \Pz)$ such that $\theta = T^\beta \eta$.
Therefore 
$$
\theta =  \sum_{i,j \in I} (\mu_i \mu_j)^\beta \left< \eta, [e_i]_\sim \otimes [e_j]_\sim \right>_{L_2(\pp{P}_{X_0} \otimes \pp{P}_{X_0})} [e_i]_\sim \otimes [e_j]_\sim
$$
where the convergence is in $L_2(\Pz \otimes \Pz)$.
Defining $a_{i,j} := \left< \eta, [e_i]_\sim \otimes [e_j]_\sim \right>_{L_2(\pp{P}_{X_0} \otimes \pp{P}_{X_0})}$, we have $\sum_{i,j \in I} a_{i,j}^2 < \infty$ from $\eta \in L_2(\Pz \otimes \Pz)$, which completes the proof.
\end{proof}

Lastly, let $\mathcal{C}_{YX} : \hbspace \to \hbspf$ be the covariance operator of the random variables $X_0$ and $Y_0$ defined as (see \eg, \cite{Fukumizu13:KBR})
\begin{equation} \label{eq:cov-op-joint}
    \mathcal{C}_{YX} f = \int  \ell(\cdot, \y) f(\x) \dd \pp{P}_{X_0 Y_0} (\x, \y) = \mathbb{E}_{X_0, Y_0}[\ell(\cdot, Y_0) f(X_0)], \quad f \in \hbspace 
\end{equation}
where $\pp{P}_{X_0 Y_0}$ is the joint distribution of $X_0$ and $Y_0$. 
Under Assumption \ref{as:g-and-theta}~\ref{as:measruable}, this covariance operator satisfies
$$
\left< \mathcal{C}_{YX} f, g  \right>_{\hbspf} = \left< \mathbb{E}_{X_0, Y_0}[\ell(\cdot, X_0) f(X_0)], g  \right>_{\hbspf} = \mathbb{E}_{X_0, Y_0}[g(Y_0) f(X_0) ], \quad f \in \hbspace, \quad g \in \hbspf.
$$
The conjugate operator of $\mathcal{C}_{YX}$ is denoted by $\mathcal{C}_{XY}: \hbspf \to \hbspace$ and given by
$$
\mathcal{C}_{XY} g = \int k(\cdot, \x) g(\y) \dd \pp{P}_{X_0 Y_0}(\x, \y) = \mathbb{E}_{X_0, Y_0} [k(\cdot, X_0) g(Y_0)],
$$
since for any $f \in \hbspace$ and $g \in \hbspf$ it holds that
$$
\left< f,  \mathcal{C}_{YX}  g  \right>_{\hbspace} = \left< f,  \mathbb{E}_{X_0, Y_0} [k(\cdot, X_0) g(Y_0)] \right>_{\hbspace} = \mathbb{E}_{X_0, Y_0}[f(X_0) g(Y_0)] = \left< \mathcal{C}_{YX} f, g  \right>_{\hbspf}. 
$$

\subsection{Lemmas} \label{sec:lemmas}

We collect here lemmas used in the proofs for the convergence results in Appendix \ref{sec:proofs-convergence}.

\begin{lemma} \label{lemma:prior_der}
Let $\mathcal{X}$, $k$, $\pp{P}_{X_0}$, $\pp{P}_{X_1}$ and $g := \dd\pp{P}_{X_1}/\dd\pp{P}_{X_0}$ be such that  Assumption \ref{as:g-and-theta}~\ref{as:measruable} and \ref{as:radon-nikodym} are satisfied.
Then we have  $\mu_{X_1} = S g$.
\end{lemma}
\begin{proof}
By definitions of the kernel mean $\mu_{X_1}$ and the Radon-Nikodym derivative $g$, we have
\[
\mu_{X_1} = \int k(\cdot,\x) \dd\pp{P}_{X_1}(\x) = \int k(\cdot,\x)  g(\x) \dd\pp{P}_{X_0}(\x) = S g \in \hbspace,
\]
where the expression $Sg$ is justified from  Assumption \ref{as:g-and-theta}~\ref{as:radon-nikodym} that $g \in L_2(\pp{P}_{X_0})$.
\end{proof}

\begin{lemma} \label{lemma:cov_int_op}
Let $\mathcal{X}$, $k$ and $\pp{P}_{X_0}$ be such that Assumption \ref{as:g-and-theta}~\ref{as:measruable} is satisfied.
Then for any $f \in L_2(\pp{P}_{X_0})$ and $\varepsilon > 0$, we have
$
S^*(\covx+ \varepsilon I)^{-1} S f = (T + \varepsilon I)^{-1} T f.
$
\end{lemma}
\begin{proof}
Let $(e_i)_{i \in I} \subset \hbspace$ and $(\mu_i)_{i \in I} \subset (0, \infty)$ as in Lemma \ref{lemma:op-eigen-decomp}.
Then we have
\begin{eqnarray*}
\lefteqn{S^*(\covx + \varepsilon I)^{-1} S f 
= S^*(\covx + \varepsilon I)^{-1} \sum_{j \in I} \mu_j \left<f, [e_j]_\sim \right>_{L_2(\pp{P}_{X_0})}  e_j}   \\
&=& S^* \sum_{i \in I}   (\mu_i + \varepsilon)^{-1}  \left< \mu_i^{1/2} e_i, \sum_{j \in I} \mu_j \left<f, [e_j]_\sim \right>_{L_2(\pp{P}_{X_0})} e_j  \right>_{\hbspace} \mu_i^{1/2} e_i \\
&=& S^* \sum_{i \in I}  (\mu_i + \varepsilon)^{-1} \mu_i \left<f, [e_i]_\sim \right>_{L_2(\pp{P}_{X_0})}  e_i = \sum_{i \in I}  (\mu_i + \varepsilon)^{-1} \mu_i \left<f, [e_i]_\sim \right>_{L_2(\pp{P}_{X_0})}  [e_i]_\sim \\
&=& (T+ \varepsilon I)^{-1} T f,
\end{eqnarray*}
as required.
%
\end{proof}

\begin{lemma} \label{lemma:S-alpha-decomp}
Let $\mathcal{X}$, $k$ and $\pp{P}_{X_0}$ be such that Assumption \ref{as:g-and-theta}~\ref{as:measruable} is satisfied.
Then for any $f \in L_2(\Pz)$ and $\alpha \geq 0$, we have
$
ST^\alpha f = \covx^{1/2 + \alpha} S^{1/2} f.  
$
\end{lemma}
\begin{proof}
Let $(e_i)_{i \in I} \subset \hbspace$ and $(\mu_i)_{i \in I} \subset (0, \infty)$ as in Lemma \ref{lemma:op-eigen-decomp}.
Then we have
\begin{eqnarray*}
\lefteqn{ \covx^{1/2 + \alpha} S^{1/2} f = \covx^{1/2 + \alpha} \sum_{i \in I} \mu_i^{1/2} \left< [e_i]_\sim, f\right>_{L_2(\Pz)}  e_i}  \\ 
&=& \sum_{\ell \in I} \mu_\ell^{1/2 + \alpha} \left< \mu_\ell^{1/2} e_\ell, \sum_{i \in I} \mu_i^{1/2} \left< [e_i]_\sim, f\right>_{L_2(\Pz)}  e_i   \right>_{\hbspace}  \mu_\ell^{1/2} e_\ell  \\ 
&=& \sum_{\ell \in I} \mu_\ell^{1/2 + \alpha} \left<[e_\ell]_\sim,  f\right>_{L_2(\Pz)} \mu_\ell^{1/2} e_\ell  = \sum_{\ell \in I} \mu_\ell^{\alpha} \left<[e_\ell]_\sim,  f\right>_{L_2(\Pz)} \mu_\ell e_\ell  \\ 
&=& \sum_{\ell \in I} \mu_\ell^{\alpha} \left<[e_\ell]_\sim,  f\right>_{L_2(\Pz)} S [e_\ell]_\sim  = S \sum_{\ell \in I} \mu_\ell^{\alpha} \left<[e_\ell]_\sim,  f\right>_{L_2(\Pz)}  [e_\ell]_\sim  \\ 
&=& S  T^\alpha f,
\end{eqnarray*}
as required.
\end{proof}

\begin{lemma} \label{lemma:cov-inv-kmean}
Let $\mathcal{X}$, $k$, $\pp{P}_{X_0}$, $\pp{P}_{X_1}$ and $g := \dd\Po/\dd\Pz$ be such that Assumption \ref{as:g-and-theta}~\ref{as:measruable} and \ref{as:radon-nikodym} are satisfied.
Assume also that Assumption \ref{as:range-assumption-g} holds, i.e., $g \in {\rm Range}(T^\alpha)$ for a constant $\alpha \geq 0$.
Then for any $\varepsilon > 0$, we have
$$
\left\| (\covx + \varepsilon I  )^{-1} {\mu}_{X_1} \right\|_\hbspace \leq 
\begin{cases}
c_\alpha \varepsilon^{ -1/2 + \alpha}, \quad ({\rm if}\ \alpha \leq 1/2) \\
c_\alpha \left\| \covx^{\alpha - 1/2} \right\|, \quad ({\rm if}\ \alpha > 1/2),
\end{cases}
$$
where $c_\alpha := \| h \|_{ L_2(\pp{P}_{X_0})}$ is a constant defined by $h \in L_2(\pp{P}_{X_0})$ such that $g = T^{\alpha}h.$ 
\end{lemma}
\begin{proof}
As in the assertion, write $g = T^\alpha h$ for $h \in L_2(\Pz)$, which exists from the assumption $g \in {\rm Range}(T^\alpha)$.
By Lemmas \ref{lemma:prior_der} and \ref{lemma:S-alpha-decomp}, we can then write $\mu_{X_1}$ as
$$
\mu_{X_1} = Sg = ST^\alpha h = \covx^{1/2 + \alpha} S^{1/2} h.
$$
Since $(\mu_i^{1/2} e_i)_{i \in I}$ is an ONS of $\hbspace$, $([e_i]_\sim)_{i \in I}$ is an ONS in $L_2(\Pz)$ and  
$$S^{1/2} h = \sum_{i \in I}  \left< [e_i]_\sim, h \right>_{L_2(\pp{P}_{X_0})} \mu_i^{1/2} e_i,
$$
it holds that
$
\| S^{1/2} h \|_\hbspace^2 = \sum_{i \in I} \left< [e_i]_\sim, h \right>_{L_2(\pp{P}_{X_0})}^2 \leq \| h \|_{L_2(\Pz)}^2 
$
and thus $S^{1/2} h \in \hbspace$.
Therefore we have
\begin{eqnarray*}
\lefteqn{\left\| (\covx + \varepsilon I  )^{-1} {\mu}_{X_1} \right\|_\hbspace 
=  \left\| (\covx + \varepsilon I  )^{-1}  \covx^{1/2 + \alpha}  S^{1/2} h \right\|_\hbspace}  \\
&\leq&  \left\| (\covx + \varepsilon I  )^{-1} \covx^{1/2 + \alpha} \right\| \left\| S^{1/2}  h \right\|_\hbspace \leq  \left\| (\covx + \varepsilon I  )^{-1} \covx^{1/2 + \alpha} \right\| \left\| h \right\|_{L_2(\Pz)}.
\end{eqnarray*}
Below we focus on bounding the first term in the above bound. 
If $\alpha \leq 1/2$,
$$
\left\| (\covx + \varepsilon I  )^{-1} \covx^{1/2 + \alpha} \right\| 
\leq \left\| (\covx + \varepsilon I  )^{-1/2+\alpha} \right\| \left\| (\covx + \varepsilon I  )^{-1/2 - \alpha} \covx^{1/2 + \alpha} \right\| \leq \varepsilon^{-1/2 + \alpha}.
$$
On the other hand, if $\alpha > 1/2$,
$$
\left\| (\covx + \varepsilon I  )^{-1} \covx^{1/2 + \alpha} \right\| 
\leq \left\| (\covx + \varepsilon I  )^{-1} \covx \right\| \left\| \covx^{\alpha - 1/2} \right\| \leq \left\| \covx^{\alpha - 1/2} \right\|.
$$
This completes the proof.
\end{proof}

\begin{lemma} \label{lemma:reg_conv}

Let $\mathcal{X}$, $k$ and $\pp{P}_{X_0}$ be such that Assumption \ref{as:g-and-theta}~\ref{as:measruable} and \ref{as:RKHS-dense} are satisfied.
Then, for any $g \in L_2(\pp{P}_{X_0})$, we have
\[ \lim_{\varepsilon \to 0} \;\| (T+ \varepsilon I)^{-1} Tg  - g \|_{L_2(\pp{P}_{X_0})} = 0. \]
\end{lemma} 

\begin{proof}
Let $(e_i)_{i \in I} \subset \hbspace$ and $(\mu_i)_{i \in I} \subset (0,\infty)$ be as in Lemma \ref{lemma:op-eigen-decomp}.
By Lemma \ref{lemma:ONB-mercer-iso}, $([e_i]_\sim)_{i \in I}$ is an ONB of $L_2(\Pz)$, which implies that $g$ can be expanded using $([e_i]_\sim)_{i \in I}$. 
From this and Lemma \ref{lemma:op-eigen-decomp}, we then have
\begin{eqnarray*}
 (T+\varepsilon I)^{-1} T g - g 
&=& \sum_{i \in I} (\mu_i + \varepsilon)^{-1} \mu_i \left<g, [e_i]_\sim \right>_{L_2(\pp{P}_{X_0})} [e_i]_\sim -  \sum_{i \in I} \left<g, [e_i]_\sim \right>_{L_2(\pp{P}_{X_0})} [e_i]_\sim \\
&=& \sum_{i \in I}   - \varepsilon  (\mu_i + \varepsilon)^{-1}  \left<g, [e_i]_\sim \right>_{L_2(\pp{P}_{X_0})} [e_i]_\sim.
\end{eqnarray*}
Thus, by Parseval's identity,
\begin{eqnarray*}
\| (T+ \varepsilon I)^{-1} Tg  - g \|_{L_2(\pp{P}_{X_0})}^2 = \sum_{i \in I} \left| \varepsilon (\mu_i+\varepsilon )^{-1} \right|^2 |\left<g, [e_i]_\sim \right>_{L_2(\pp{P}_{X_0})} |^2.
\end{eqnarray*}
Note that $\left| \varepsilon (\mu_i+\varepsilon )^{-1} \right|^2 \leq 1$ for all $i \in I$, that $\sum_{i \in I} | \left<g, [e_i]_\sim \right>_{L_2(\pp{P}_{X_0})} |^2 = \| g \|_{L_2(\pp{P}_{X_0})}^2 < \infty$, and that $\lim_{\varepsilon \to 0} \left| \varepsilon (\mu_i+\varepsilon )^{-1} \right|^2 = 0$ (which follows from $\mu_i > 0$ for all $i \in I$).
These facts enable the use of the dominated convergence theorem, from which we have
\begin{eqnarray*}
\lim_{\varepsilon \to 0}\; \| (T+ \varepsilon I)^{-1} Tg  - g \|_{L_2(\pp{P}_{X_0})}^2 
&=& \lim_{\varepsilon \to 0}\; \sum_{i \in I}  \left| \varepsilon (\mu_i+\varepsilon )^{-1} \right|^2 |\left<g, [e_i]_\sim \right>_{L_2(\pp{P}_{X_0})} |^2 \\
&=& \sum_{i \in I} \lim_{\varepsilon \to 0}\; \left| \varepsilon (\mu_i+\varepsilon )^{-1} \right|^2 |\left<g, [e_i]_\sim \right>_{L_2(\pp{P}_{X_0})} |^2 = 0.
\end{eqnarray*}
This completes the proof.
\end{proof}

\begin{lemma} \label{lemma:range_bound}
Let $\mathcal{X}$, $k$ and $\pp{P}_{X_0}$ be such that Assumption \ref{as:g-and-theta}~\ref{as:measruable} is satisfied.
Let $g \in L_2(\pp{P}_{X_0})$ be such that $g \in {\rm Range} (T^{\alpha})$ for a constant $0 \leq \alpha \leq 1$.
Then, for all $\ve > 0$, we have
\[
\| (T+\ve I)^{-1} T g - g \|_{  L_2(\pp{P}_{X_0}) } \leq  c_\alpha \ve^{\alpha},
\]
where $c_\alpha := \| h \|_{ L_2(\pp{P}_{X_0})}$ with $h \in L_2(\Pz)$ being such that $g = T^\alpha h$.
\end{lemma}

\begin{proof}
Let $(e_i)_{i \in I} \subset \hbspace$ and $(\mu_i)_{i \in I} \subset (0,\infty)$ be as in Lemma \ref{lemma:op-eigen-decomp}.
From $g \in {\rm Range} (T^{\alpha})$ there exists $h \in L_2(\pp{P}_{X_0})$ such that $g = T^{\alpha} h$.
%
Therefore $g$ can be written as
\begin{equation} \label{eq:g-range-expansion}
g = T^\alpha h = \sum_{i \in I} \mu_i^{\alpha} b_i [e_i]_\sim,
\end{equation}
where the convergence is in $L_2(\pp{P}_{X_0})$, and $b_i := \left< h, [e_i]_\sim \right>_{L_2(\pp{P}_{X_0})} $.
It then follows that
\begin{eqnarray*}
 (T+\ve I)^{-1} T g - g 
&=& \sum_{i \in I} (\mu_i + \ve)^{-1} \mu_i \mu_i^{\alpha} b_i [e_i]_\sim - \sum_{i \in I} \mu_i^{\alpha} b_i [e_i]_\sim \\
&=& \sum_{i \in I} - \ve (\mu_i + \ve)^{-1} \mu_i^{\alpha} b_i [e_i]_\sim.
\end{eqnarray*}
Therefore, by Parseval's identity, we have
\[
\| (T+\ve I)^{-1} T g - g \|_{  L_2(\pp{P}_{X_0}) }^2 =  \sum_{i \in I}  \ve^2 (\mu_i + \ve)^{-2} \mu_i^{2\alpha} b_i^2.
\]
The rhs of the above equation can be bounded from above as
\begin{eqnarray*}
 \ve^2 (\mu_i + \ve)^{-2} \mu_i^{2\alpha} b_i^2 
&=& \ve^2 (\mu_i + \ve)^{-2+2\alpha}  (\mu_i + \ve)^{-2\alpha} \mu_i^{2\alpha} b_i^2 \\
&\leq&  \ve^2 (\mu_i + \ve)^{-2+ 2\alpha}   b_i^2   
= \ve^{2\alpha} \ve^{2- 2\alpha} (\mu_i + \ve)^{-2+2\alpha}   b_i^2 \leq \ve^{2 \alpha} b_i^2,
\end{eqnarray*}
where the above two inequalities follow from $\ve > 0$ and $\mu_i > 0$, and the last inequality uses $\alpha \leq 1$. 
Thus, we have
$$
\| (T+\ve I)^{-1} T g - g \|_{  L_2(\pp{P}_{X_0}) }^2  
\leq   \ve^{2\alpha}  \sum_{i \in I} b_i^2 = \ve^{2\alpha} \sum_{i \in I} \left( \left< h, [e_i]_\sim \right>_{L_2(\pp{P}_{X_0})} \right)^2 \leq \ve^{2\alpha}  \| h \|_{ L_2(\pp{P}_{X_0})}^2,
$$
where the last inequality follows from $([e_i]_\sim)_{i \in I}$ being an ONS in $L_2(\Pz)$.
\end{proof}

\begin{remark}
Different from Lemma \ref{lemma:reg_conv}, Lemma \ref{lemma:range_bound} does not require Assumption \ref{as:g-and-theta}~\ref{as:RKHS-dense} that $S^*$ has a dense image in $L_2(\Pz)$. 
In Lemma \ref{lemma:reg_conv}, this condition is required to guarantee that $([e_i]_\sim)_{i \in I}$ is an ONB in $L_2(\Pz)$, so that $g$ can be expanded by this ONB.
On the other hand, in Lemma \ref{lemma:range_bound}, $g$ can be written as \eqref{eq:g-range-expansion}, thanks to the assumption $g \in {\rm Range}(T^\alpha)$; this is the reason why Lemma \ref{lemma:range_bound} does not need Assumption \ref{as:g-and-theta}~\ref{as:RKHS-dense}.
\end{remark}

The following is a key lemma, based on which we show the consistency and convergence rates of our estimator.
\begin{lemma} \label{lemma:approximation-key}
Let $\mathcal{X}$, $\mathcal{Y}$, $k$, $\ell$, $\pp{P}_{X_0}$, $\pp{P}_{X_1}$, $g := \dd\Po/\dd\Pz$ and $\theta: \mathcal{X} \times \mathcal{X} \to \mathbb{R}$ defined in \eqref{eq:theta-def} be such that Assumption \ref{as:g-and-theta}~\ref{as:measruable} and \ref{as:radon-nikodym} are satisfied.
Then for any $\varepsilon > 0$, we have 
\begin{eqnarray*}
\lefteqn{\| \covyx (\covx + \varepsilon I)^{-1} \mu_{X_1} - \mu_{Y\langle 0|1 \rangle} \|_{\hbspf}^2}  \nonumber \\
&=&   \left< g_{\varepsilon} \otimes g_{\varepsilon}, \theta \right>_{L_2(\pp{P}_{X_0} \otimes \pp{P}_{X_0})} - 2 \left<g, (T + \varepsilon I)^{-1} T \int \theta(\cdot,\tilde{\x}) \dd \pp{P}_{X_1}(\tilde{\x})  \right>_{L_2(\pp{P}_{X_0})}\\
&& + \iint \theta (\x,\tilde{\x}) \dd\pp{P}_{X_1}(\x) \dd\pp{P}_{X_1}(\tilde{\x})
\end{eqnarray*} 
where $g_{\varepsilon} := (T + \varepsilon I)^{-1} T g$. 
In the second term of the right hand side, the inner product is well defined, since we have $(T + \varepsilon_n I)^{-1} T \int \theta(\cdot,\tilde{\x}) \dd \pp{P}_{X_1}(\tilde{\x}) \in L_2(\Pz)$.
\end{lemma}

\begin{proof}
First note that, because $\ell$ is bounded (Assumption \ref{as:g-and-theta}~\ref{as:measruable}), the function $\theta$ in \eqref{eq:theta-def} satisfies $\theta \in L_2(\pp{P}_{X_0} \otimes \pp{P}_{X_0})$.
Therefore, the right hand side of the assertion is well defined.
The left hand side of the assertion can be written as
\begin{eqnarray}
\lefteqn{\| \covyx (\covx + \varepsilon I)^{-1} \mu_{X_1} - \mu_{Y\langle 0|1 \rangle} \|_{\hbspf}^2}  \label{eq:approx_decom} \\
& = & \| \covyx (\covx + \varepsilon I)^{-1} \mu_{X_1} \|_{\hbspf}^2   - 2 \left< \covyx (\covx + \varepsilon I)^{-1} \mu_{X_1}, \mu_{Y\langle 0|1 \rangle} \right>_{\hbspf} 
+ \| \mu_{Y\langle 0|1 \rangle} \|_{\hbspf}^2. \nonumber
\end{eqnarray} 
As in the proof of \citet[Thm.~8]{Fukumizu13:KBR}, the third term in \eqref{eq:approx_decom} can be written as
\begin{equation} \label{eq:third_term_proof}
  \| \mu_{Y\langle 0|1 \rangle} \|_{\hbspf}^2 = \iint \theta (\x,\tilde{\x}) \dd\pp{P}_{X_1}(\x) \dd\pp{P}_{X_1}(\tilde{\x}). 
\end{equation}
We thus derive the expressions for the first two terms in \eqref{eq:approx_decom} in the sequel.

\paragraph{The first term in \eqref{eq:approx_decom}.}
Let $f \in \hbspace$ be arbitrary, and let $(\tilde{X}_0, \tilde{Y}_0)$ denote an independent copy of $(X_0,Y_0)$.
By the property of $\covyx$ that $\left<\covyx f, h \right>_{\hbspf} = \mathbb{E}_{X_0,Y_0}[f(X_0) h(Y_0))]$ for any $h \in \hbspf$ and the expression $\theta(\x,\tilde{\x}) = \mathbb{E}_{Y_0, \tilde{Y_0}}[\ell(Y_0,\tilde{Y_0}) | X_0 = \x, \tilde{X}_0 = \tilde{\x}]$, we have
\begin{eqnarray}
 \| \covyx f \|_{\hbspf}^2 
 &=& \left<\covyx f, \covyx f \right>_{\hbspf} = \mathbb{E}_{X_0,Y_0}[f(X_0) (\covyx f)(Y_0))] \nonumber \\
 &=& \mathbb{E}_{X_0,Y_0}[f(X_0)\mathbb{E}_{\tilde{X}_0,\tilde{Y}_0}[\ell(Y_0,\tilde{Y_0})f(\tilde{X_0})]] \nonumber \\
 &=& \mathbb{E}_{X_0, \tilde{X}_0}[f(X_0) f(\tilde{X_0}) \mathbb{E}_{Y_0, \tilde{Y_0}}[\ell(Y_0,\tilde{Y_0}) | X_0, \tilde{X}_0] ] \quad (\because {\rm Fubini\ theorem}) \nonumber \\
 &=& \mathbb{E}_{X_0, \tilde{X}_0} [f(X_0) f(\tilde{X_0}) \theta(X_0,\tilde{X_0})], \label{eq:covyx_f_expand} 
\end{eqnarray}
where the use of Fubini's theorem is enabled by $\ell$ and $f$ being bounded, the latter implied by $k$ being bounded.
Now define $f := (\covx + \varepsilon I)^{-1} \mu_{X_1} \in \hbspace$.
With this choice of $f$, the quantity $ \| \covyx f \|_{\hbspf}^2$ is equal to the first term in \eqref{eq:approx_decom}.
From (\ref{eq:covyx_f_expand}), it follows that
\begin{eqnarray}
\| \covyx f \|_{\hbspf}^2 
&=& \mathbb{E}_{X_0,\tilde{X}_0}[f(X_0) f(\tilde{X}_0) \theta(X_0,\tilde{X}_0)] \nonumber \\
&=& \iint f(\x) f(\tilde{\x}) \theta(\x,\tilde{\x}) \dd\pp{P}_{X_0}(\x) \dd\pp{P}_{X_0}(\tilde{\x}) = \left< S^*f \otimes S^*f, \theta \right>_{L_2(\pp{P}_{X_0} \otimes \pp{P}_{X_0})} \nonumber \\
&=& \left<  S^* (\covx + \varepsilon I)^{-1} \mu_{X_1} \otimes  S^* (\covx + \varepsilon I)^{-1} \mu_{X_1}, \theta \right>_{L_2(\pp{P}_{X_0} \otimes \pp{P}_{X_0})} \nonumber \\
&=& \left< S^* (\covx + \varepsilon I)^{-1} S g \otimes S^* (\covx + \varepsilon I)^{-1} S g, \theta \right>_{L_2(\pp{P}_{X_0} \otimes \pp{P}_{X_0})}  \quad (\because {\rm Lemma}\ \ref{lemma:prior_der}) \nonumber \\
&=& \left< (T + \varepsilon I)^{-1} T g \otimes (T + \varepsilon I)^{-1} T g, \theta \right>_{L_2(\pp{P}_{X_0} \otimes \pp{P}_{X_0})}  \quad (\because {\rm Lemma}\ \ref{lemma:cov_int_op}) \nonumber \\
&=&  \left< g_{\varepsilon} \otimes g_{\varepsilon}, \theta \right>_{L_2(\pp{P}_{X_0} \otimes \pp{P}_{X_0})}, \label{eq:g_eps_tensor}
\end{eqnarray}
where $g_{\varepsilon} := (T + \varepsilon I)^{-1} T g$.
\paragraph{The second term in \eqref{eq:approx_decom}.}
First we have
\begin{eqnarray}
 \mathbb{E}_{Y_0}[\mu_{Y\langle 0|1 \rangle} (Y_0) | X_0 = \x] 
&=& \mathbb{E}_{Y_0} \left[  \int  \mathbb{E}_{\tilde{Y}_0} [\ell(Y_0, \tilde{Y}_0) | \tilde{X}_0 = \tilde{\x}] \dd\pp{P}_{X_1}(\tilde{\x}) | X_0 = \x \right] \nonumber \\
&=& \int \mathbb{E}_{Y_0, \tilde{Y}_0} \left[   \ell(Y_0, \tilde{Y}_0) | X_0 = \x, \tilde{X}_0 = \tilde{\x}  \right] \dd\pp{P}_{X_1}(\tilde{\x}) \quad (\because {\rm Fubini } ) \nonumber \\
&=& \int \theta(\x,\tilde{\x}) \dd\pp{P}_{X_1}(\tilde{\x}) \label{eq:formula_theta}
\end{eqnarray}
where $(\tilde{X}_0,\tilde{Y}_0)$ is an independent copy of $(X_0,Y_0)$.
Note that for the first expression in \eqref{eq:formula_theta}, we have $
\mathbb{E}_{Y_0}[\mu_{Y\langle 0|1 \rangle} (Y_0) | X_0 = \cdot] \in L_2(\pp{P}_{X_0})$ since $\ell$ is bounded.
Using this and \eqref{eq:formula_theta}, we have
\begin{eqnarray}
\covxy \mu_{Y\langle 0|1 \rangle}
&=& \mathbb{E}_{X_0,Y_0} [ k(\cdot,X_0)  \mu_{Y\langle 0|1 \rangle}(Y_0)] = \mathbb{E}_{X_0} \left[ k(\cdot,X_0) \mathbb{E}_{Y_0} [ \mu_{Y\langle 0|1 \rangle}(Y_0) |X_0 ] \right]  \nonumber  \\ 
&=& S \mathbb{E}_{Y_0} [ \mu_{Y\langle 0|1 \rangle}(Y_0) |X_0 = \cdot ] = S \int \theta(\cdot,\tilde{\x}) \dd\pp{P}_{X_1}(\tilde{\x}). \label{eq:second_comp}
\end{eqnarray}
Now for the second term in \eqref{eq:approx_decom}, we have
\begin{align*}
&\left< \covyx (\covx + \varepsilon_n I)^{-1} \mu_{X_1}, \mu_{Y\langle 0|1 \rangle}\right>_{\hbspf} = \left< \mu_{X_1}, (\covx + \varepsilon_n I)^{-1} \covxy \mu_{Y\langle 0|1 \rangle}\right>_{\hbspace} \\
&= \left< Sg, (\covx + \varepsilon_n I)^{-1}  S \int \theta(\cdot,\tilde{\x}) \dd\pp{P}_{X_1}(\tilde{\x})  \right>_\hbspace \quad (\because\ {\rm Lemma}\ \ref{lemma:prior_der}\ \text{and } (\ref{eq:second_comp}))\\
&= \left< g, S^*(\covx + \varepsilon_n I)^{-1}  S \int \theta(\cdot,\tilde{\x}) \dd\pp{P}_{X_1}(\tilde{\x})  \right>_{L_2(\Pz)}  \\
&= \left<g,  (T + \varepsilon_n I)^{-1} T \int \theta(\cdot,\tilde{\x}) \dd\pp{P}_{X_1}(\tilde{\x})  \right>_{L_2(P_{X_0})} \quad (\because {\rm Lemma}\ \ref{lemma:cov_int_op} ).
\end{align*}
This completes the proof.
\end{proof}

\section{Proofs for Section \ref{sec:theory-main}}
\label{sec:proofs-convergence}
We provide proofs for the convergence results presented in Section \ref{sec:theory-main} of the main paper.
The proofs rely on several lemmas collected and proved in Appendix \ref{sec:theory-preliminary}.
The notation and definitions follow those in these sections.
In the following, for any bounded linear operator $A: V \to W$ between normed vector spaces $V$ and $W$,  we denote by $\| A \|$ its operator norm: $\| A \| := \sup_{ \| v \|_V \leq 1} \| A(v) \|_W$, where $\| \cdot \|_V$ and $\| \cdot \|_W$ denote the norms of $V$ and $W$, respectively.

We first show that the CME estimator $\hat{\mu}_{\left< 0 | 1 \right>}$ in \eqref{eq:empirical-cme}  can be expressed in terms of certain empirical covariance operators.    
Given an i.i.d.~sample $(\x_i, \y_i)_{i=1}^n$ from $\pp{P}_{X_0 Y_0}$, the covariance operators $\mathcal{C}_{XX}: \hbspace \to \hbspace$ in \eqref{eq:cov-ope-def-app} and $\mathcal{C}_{YX}: \hbspace \to \hbspf$ in \eqref{eq:cov-op-joint}  can be respectively approximated by $\ecovx: \hbspace \to \hbspace$ and $\ecovyx: \hbspace \to \hbspf$, defined as
$$
\ecovx f := \frac{1}{n} \sum_{i=1}^n k(\cdot, \x_i) f(\x_i), \quad \ecovyx f  = \frac{1}{n} \sum_{i=1}^n \ell(\cdot, \y_i) f(\x_i), \quad f \in \hbspace.
$$
Under Assumption \ref{as:g-and-theta}~\ref{as:measruable} that the kernels $k$ and $\ell$ are bounded, these satisfy $\| \ecovx - \covx \| = O_p(n^{-1/2})$ and $\| \ecovyx - \covyx \| = O_p(n^{-1/2})$ as $n \to \infty$.
Similarly, given an i.i.d.~sample $(\x'_j)_{j=1}^m$ from $\pp{P}_{X_1}$, the kernel mean $\mu_{X_1} := \int k(\cdot,\x) \dd \pp{P}_{X_1}(\x)$ of $\pp{P}_{X_1}$ can be estimated as $\hat{\mu}_{X_1} := \frac{1}{m} \sum_{i=1}^m k(\cdot, \x_j)$, with the error rate $\| \mu_{X_1} - \hat{\mu}_{X_1} \|_\hbspace = O_p(m^{-1/2})$ as $m \to \infty$ under Assumption \ref{as:g-and-theta}~\ref{as:measruable}.

\begin{proposition}
Let $ \hat{\mu}_{\left< 0 | 1 \right>}$ be the CME estimator in \eqref{eq:empirical-cme}. Then we have
\begin{equation} \label{eq:CME-emp-cov-op}
    \hat{\mu}_{\left< 0 | 1 \right>} = \ecovyx (\ecovx + \varepsilon I)^{-1} \hat{\mu}_{X_1}.
\end{equation}
\end{proposition}
\begin{proof}
Define $g := (\ecovx + \varepsilon I)^{-1} \hat{\mu}_{X_1}$. 
Since $\hat{\mu}_{X_1} = (\ecovx + \varepsilon I)  g = \frac{1}{n} \sum_{j=1}^n k(\cdot, \x_j) g(\x_j)  + \varepsilon g$, we have $\hat{\mu}_{X_1} (\x_\ell) = \frac{1}{n} \sum_{j=1}^n k(\x_\ell, \x_j) g(\x_j)  + \varepsilon g(\x_\ell) = \frac{1}{n}( {\bf K} {\bm g} )_\ell + \varepsilon {\bm g}_\ell$ for all $\ell = 1, \dots, n$, where ${\bf K} \in \mathbb{R}^{n \times n}$ with ${\bf K}_{i,j} = k(\x_i, \x_j)$ and ${\bm g} = (g(\x_1), \dots, g(\x_n) )^\top \in \mathbb{R}^n$.
Therefore  ${\bm \mu} = \frac{1}{n} ( {\bf K} + n \varepsilon {\bf I}) {\bm g}$, where ${\bm \mu} := ( \hat{\mu}_{X_1}(\x_1),\dots,  \hat{\mu}_{X_1}(\x_n)  )^\top = \widetilde{\kmat}\mathbf{1}_m$, where $\mathbf{1}_m=(1/m,\ldots,1/m)^\top$ and $\widetilde{\kmat} \in \mathbb{R}^{n \times m}$ with $\widetilde{\kmat}_{ij} = k(\x_i,\x'_j)$. 
Thus ${\bm g} = n ( {\bf K} + n \varepsilon {\bf I} )^{-1} {\bm \mu}$. 
Lastly, the right hand side of \eqref{eq:CME-emp-cov-op} can be expressed as $\frac{1}{n}\sum_{i=1}^n \ell(\cdot, \y_i) g(\x_i) = \sum_{i=1}^n \beta_i \ell(\cdot, \y_i)$, where $\beta = (\beta_1,\dots,\beta_n)^\top = n^{-1} {\bm g} = ( {\bf K} + n \varepsilon {\bf I} )^{-1} {\bm \mu}$, which is the expression of the CME estimator $\hat{\mu}_{\left< 0 | 1 \right>}$ in \eqref{eq:empirical-cme}.  
\end{proof} 

\subsection{Convergence Rates of the Stochastic Error}

The proofs of Theorems \ref{theo:uinf_conv} and
\ref{theo:convergence-rate} rely on the following result, which characterizes the ``stochastic error'' of the CME estimator. 
As stated in Assumption \ref{as:cov-op-embed}, we assume $m=n$ in the following.

\begin{theorem} \label{theo:estimation-error-rate}
Let $\mathcal{X}$ be a measurable space, $k$ be a measurable kernel on $\mathcal{X}$ and $\pp{P}_{X_0}$ be a probability measure on $\mathcal{X}$ such that Assumption \ref{as:g-and-theta}~\ref{as:measruable}, \ref{as:radon-nikodym} and \ref{as:cov-op-embed}  are satisfied.
Assume that the Radon-Nikodym derivative $g := \dd\Po/\dd\Pz$ satisfies $g \in {\rm Range}(T^\alpha)$ for a constant $\alpha \geq 0$.
Then for any $\varepsilon_n > 0$ such that $\varepsilon_n \to 0$ as $n \to \infty$, we have 
$$
 \| \ecovyx (\ecovx + \varepsilon_n I  )^{-1} \hat{\mu}_{X_1} -  \covyx (\covx + \varepsilon_n I)^{-1} \mu_{X_1} \|_{\hbspf}  
 =  O_p \left( n^{-1/2} \varepsilon_n^{ \min( -1 + \alpha, -1/2 )} \right)  \qquad (n \to \infty)
$$
\end{theorem}

\begin{proof}
As in the proof of \citet[Theorem 11]{Fukumizu13:KBR}, the lhs can be bounded  as
\begin{eqnarray}
 \lefteqn{\| \ecovyx (\ecovx + \varepsilon_n I  )^{-1} \hat{\mu}_{X_1} -  \covyx (\covx + \varepsilon_n I)^{-1} \mu_{X_1} \|_{\hbspf}} \nonumber \\
&\leq& \| \ecovyx (\ecovx + \varepsilon_n I  )^{-1} ( \hat{\mu}_{X_1} - {\mu}_{X_1} ) \|_{\hbspf}
+ \| (\ecovyx - \covyx ) (\covx + \varepsilon_n I  )^{-1} {\mu}_{X_1} \|_{\hbspf} \nonumber \\
&& + \| \ecovyx (\ecovx + \varepsilon_n I  )^{-1} ( \covx - \ecovx ) (\covx + \varepsilon_n I)^{-1} {\mu}_{X_1} \|_{\hbspf} \label{eq:upper-estimation}
\end{eqnarray}
By \citet[Theorem 1]{Baker1973}, $\ecovyx$ can be decomposed as $\ecovyx = \ecovy^{1/2} \ecoryx \ecovx^{1/2}$ for a bounded linear operator $\ecoryx: \hbspace \to \hbspf$ with $\| \ecoryx \| \leq 1$, where $\ecovy^{1/2}: \hbspf \to \hbspf$ and  $\ecovx^{1/2}: \hbspace \to \hbspace$ are such that $\ecovy = \ecovy^{1/2} \ecovy^{1/2}$ and $\ecovx = \ecovx^{1/2} \ecovx^{1/2}$.
Therefore,
\begin{eqnarray}
\|  \ecovyx (\ecovx + \varepsilon_n I  )^{-1} \| 
&=& \|  \ecovy^{1/2} \ecoryx \ecovx^{1/2} (\ecovx + \varepsilon_n I  )^{-1} \| \nonumber \\
&\leq& \| \ecovy^{1/2} \|~\|  (\ecovx + \varepsilon_n I  )^{-1/2}  \|  \leq  \| \ecovy^{1/2} \| \varepsilon_n^{-1/2}. \label{eq:Cyy-Cxx-inv}
\end{eqnarray}
Thus, the rate of the first term in \eqref{eq:upper-estimation} is
\begin{equation*}
 \| \ecovyx (\ecovx + \varepsilon_n I  )^{-1} ( \hat{\mu}_{X_1} - {\mu}_{X_1} ) \|_{\hbspf}  \leq  \| \ecovy^{1/2} \|  \varepsilon^{-1/2} \| \hat{\mu}_{X_1} - {\mu}_{X_1}  \|_{ \hbspace } = O_p(\varepsilon^{-1/2} n^{-1/2} ).
\end{equation*}
Next, the rate of the second term in \eqref{eq:upper-estimation} is given by
\begin{eqnarray*}
\| (\ecovyx - \covyx ) (\covx + \varepsilon_n I  )^{-1} {\mu}_{X_1} \|_{\hbspf} 
&\leq& \| \ecovyx - \covyx \| \| (\covx + \varepsilon_n I  )^{-1} {\mu}_{X_1} \|_{\hbspace} \\
&\leq& \| \ecovyx - \covyx \| c_\alpha \varepsilon_n^{  \min( -1/2 + \alpha, 0 ) } \qquad (\because\ {\rm Lemma}\ \ref{lemma:cov-inv-kmean}) \\
&=& O_p\left(n^{-1/2} \varepsilon_n^{  \min( -1/2 + \alpha, 0 ) } \right),
\end{eqnarray*}
where $c_\alpha$ is a constant depending only on $\alpha$ and $g$.
Finally, for the third term in \eqref{eq:upper-estimation}, the rate is given as
\begin{eqnarray*}
\lefteqn{\| \ecovyx (\ecovx + \varepsilon_n I  )^{-1} ( \covx - \ecovx ) (\covx + \varepsilon_n I)^{-1} {\mu}_{X_1} \|_{\hbspf}} \\
&\leq&  \| \ecovyx (\ecovx + \varepsilon_n I  )^{-1} \| \|  \covx - \ecovx  \| \| (\covx + \varepsilon_n I)^{-1} {\mu}_{X_1} \|_{\hbspace} \\
&\leq&  \| \ecovy^{1/2} \| \varepsilon_n^{-1/2} \|  \covx - \ecovx  \| c_\alpha \varepsilon_n^{  \min( -1/2 + \alpha, 0 ) } \qquad (\because \ \eqref{eq:Cyy-Cxx-inv}\ {\rm and}\ {\rm Lemma}\ \ref{lemma:cov-inv-kmean}) \\
&=& O_p\left( n^{-1/2} \varepsilon_n^{ \min( -1 + \alpha, -1/2 )} \right).
\end{eqnarray*}
Since we will set $\varepsilon_n$ so that $\varepsilon_n \to 0$ as $n \to \infty$, the rate of the third term is the slowest in the three terms in  \eqref{eq:upper-estimation}.
This completes the proof.
\end{proof}

\subsection{Proof of Theorem \ref{theo:uinf_conv}}
\label{sec:proof-consistency}

\begin{proof}
By the triangle inequality, we can bound the error of our estimator as
\begin{eqnarray}
\lefteqn{\left\| \ecovyx (\ecovx + \varepsilon_n I  )^{-1} \hat{\mu}_{X_1} - \mu_{Y\langle 0|1 \rangle} \right\|_{\hbspf}} \nonumber \\
&\leq& \left\| \ecovyx (\ecovx + \varepsilon_n I  )^{-1} \hat{\mu}_{X_1} -  \covyx (\covx + \varepsilon_n I)^{-1} \mu_{X_1} \right\|_{\hbspf} \label{eq:estimation_error} \\
&& + \left\| \covyx (\covx + \varepsilon_n I)^{-1} \mu_{X_1} - \mu_{Y\langle 0|1 \rangle} \right\|_{\hbspf} \label{eq:approx_error},
\end{eqnarray}
where \eqref{eq:estimation_error} can be interpreted as the stochastic error and \eqref{eq:approx_error} as the approximation error.
Note that the assumption $g \in L_2(\pp{P}_{X_0})$ enables the use of Theorem \ref{theo:estimation-error-rate} with $\alpha = 0$, which implies that the estimation error (\ref{eq:estimation_error}) converges to $0$ at rate $O_p(n^{-1/2} \varepsilon_n^{-1})$ as $n \to \infty$, provided that $\varepsilon_n \to 0$ and $n^{1/2}\varepsilon_n \to \infty$ as $n \to \infty$.


Here we aim to prove that the approximation error (\ref{eq:approx_error})
goes to zero as $\varepsilon_n \to 0$. 
Note that to this end, we cannot apply the proof of Theorem 8 in \citet{Fukumizu13:KBR}, since it relies on stronger assumptions than ours.
We do this by using Lemma \ref{lemma:approximation-key}, which shows that the approximation error can be written as 
\begin{eqnarray}
\left\| \covyx (\covx + \varepsilon_n I)^{-1} \mu_{X_1} - \mu_{Y\langle 0|1 \rangle} \right\|_{\hbspf}^2  
&=&   \left< g_{\varepsilon_n} \otimes g_{\varepsilon_n}, \theta \right>_{L_2(\pp{P}_{X_0} \otimes \pp{P}_{X_0})} \label{eq:first-term-consistency} \\
&& - 2 \left<g,  (T + \varepsilon_n I)^{-1} T \int \theta(\cdot,\tilde{x}) \dd\pp{P}_{X_1}(\tilde{\x})  \right>_{L_2(\pp{P}_{X_0})} \label{eq:second-term-consistency} \\
&& + \iint \theta (\x,\tilde{\x}) \dd\pp{P}_{X_1}(\x) \dd\pp{P}_{X_1}(\tilde{\x}), \nonumber
\end{eqnarray} 
where $g_{\varepsilon_n} := (T + \varepsilon_n I)^{-1} T g$ with $g = \dd\pp{P}_{X_1}/\dd\pp{P}_{X_0}$ being the Radon-Nikodym derivative.
Below we show the convergence limits of \eqref{eq:first-term-consistency} and \eqref{eq:second-term-consistency} as $\varepsilon_n \to 0$, which conclude the proof.
\paragraph{Convergence of \eqref{eq:first-term-consistency}.}
We will show that
\begin{equation} \label{eq:first_goal}
\left< g_{\varepsilon_n} \otimes g_{\varepsilon_n}, \theta \right>_{L_2(\pp{P}_{X_0} \otimes \pp{P}_{X_0})}  \to  \iint \theta(\x,\tilde{\x}) \dd\pp{P}_{X_1}(\x) \dd\pp{P}_{X_1}(\tilde{\x}) \quad (\varepsilon_n \to 0).
\end{equation}
Note that we have
\begin{eqnarray*}
 \left< g \otimes g, \theta \right>_{L_2(\pp{P}_{X_0} \otimes \pp{P}_{X_0})} 
= \iint  g(\x) g (\tilde{\x}) \theta(\x,\tilde{\x}) \dd\pp{P}_{X_0}(\x)\dd\pp{P}_{X_0}(\tilde{\x}) = \iint \theta(\x,\tilde{\x}) \dd\pp{P}_{X_1}(\x) \dd\pp{P}_{X_1}(\tilde{\x}). 
\end{eqnarray*} 
Therefore it suffices to show that 
$$
\left< g_{\varepsilon_n} \otimes g_{\varepsilon_n}, \theta \right>_{L_2(\pp{P}_{X_0} \otimes \pp{P}_{X_0})}  \to \left< g \otimes g, \theta \right>_{L_2(\pp{P}_{X_0} \otimes \pp{P}_{X_0})} \quad (\varepsilon_n \to 0).
$$ 
Note that by the Cauchy-Schwartz inequality, we have
\begin{eqnarray*}
&& \left| \left< g_{\varepsilon_n} \otimes g_{\varepsilon_n}, \theta \right>_{L_2(\pp{P}_{X_0} \otimes \pp{P}_{X_0})} - \left< g \otimes g, \theta \right>_{L_2(\pp{P}_{X_0} \otimes \pp{P}_{X_0})}  \right| \\
&=& \left| \left< g_{\varepsilon_n} \otimes g_{\varepsilon_n} - g \otimes g, \theta \right>_{L_2(\pp{P}_{X_0} \otimes \pp{P}_{X_0})}  \right| 
\leq \left\| g_{\varepsilon_n} \otimes g_{\varepsilon_n} - g \otimes g \right\|_{L_2(\pp{P}_{X_0} \otimes \pp{P}_{X_0})}  \left\| \theta \right\|_{L_2(\pp{P}_{X_0} \otimes \pp{P}_{X_0})}  .
\end{eqnarray*}
Thus we focus on showing that
\begin{equation} \label{eq:first_subgoal}
\| g_{\varepsilon_n} \otimes g_{\varepsilon_n} - g \otimes g \|_{L_2(\pp{P}_{X_0} \otimes \pp{P}_{X_0})} \to 0 \quad (\varepsilon_n \to 0).
\end{equation}
By the triangle inequality we have 
\begin{eqnarray}
\lefteqn{\| g_{\varepsilon_n} \otimes g_{\varepsilon_n} - g \otimes g \|_{L_2(\pp{P}_{X_0} \otimes \pp{P}_{X_0})}} \nonumber \\
&\leq& \| g_{\varepsilon_n} \otimes g_{\varepsilon_n} - g \otimes g_{\varepsilon_n} \|_{L_2(\pp{P}_{X_0} \otimes \pp{P}_{X_0})} +  \| g \otimes g_{\varepsilon_n}  - g \otimes g \|_{L_2(\pp{P}_{X_0} \otimes \pp{P}_{X_0})}. \label{eq:decom_subgoal}
\end{eqnarray}
The first term of \eqref{eq:decom_subgoal} can be written as
\begin{eqnarray*}
\lefteqn{\| g_{\varepsilon_n} \otimes g_{\varepsilon_n} - g \otimes g_{\varepsilon_n} \|_{L_2(\pp{P}_{X_0} \otimes \pp{P}_{X_0})}  =  \| ( g_{\varepsilon_n} - g)  \otimes g_{\varepsilon_n}  \|_{L_2(\pp{P}_{X_0} \otimes \pp{P}_{X_0})}} \\
&=&  \|  g_{\varepsilon_n} - g \|_{L_2(\pp{P}_{X_0})}  \| g_{\varepsilon_n}  \|_{L_2(\pp{P}_{X_0})} \to 0 \quad (\varepsilon_n \to 0) \quad (\because {\rm Lemma}\ \ref{lemma:reg_conv}), 
\end{eqnarray*}
Similarly, the second term of \eqref{eq:decom_subgoal} can be written as 
\begin{eqnarray*}
\lefteqn{\| g \otimes g_{\varepsilon_n} - g \otimes g \|_{L_2(\pp{P}_{X_0} \otimes \pp{P}_{X_0})}=  \|  g  \otimes (g_{\varepsilon_n} - g)  \|_{L_2(\pp{P}_{X_0} \otimes \pp{P}_{X_0})}} \\
&=&  \|  g \|_{L_2(\pp{P}_{X_0})}  \| g_{\varepsilon_n} - g \|_{L_2(\pp{P}_{X_0})} \to 0 \quad (\varepsilon_n \to 0)  \quad (\because {\rm Lemma}\ \ref{lemma:reg_conv}).
\end{eqnarray*}
We have shown (\ref{eq:first_subgoal}), which concludes (\ref{eq:first_goal}).

\paragraph{Convergence of \eqref{eq:second-term-consistency}.}
We show that as $\varepsilon_n \to 0$,
\begin{equation} \label{eq:second_goal}
\left<g, (T + \varepsilon_n I)^{-1} T \int \theta(\cdot,\tilde{\x}) \dd\pp{P}_{X_1}(\tilde{\x})  \right>_{L_2(\pp{P}_{X_0})}  \to  \iint \theta(\x,\tilde{\x}) \dd\pp{P}_{X_1}(\x) \dd\pp{P}_{X_1}(\tilde{\x}) .
\end{equation}
From Lemma \ref{lemma:reg_conv}, as $\varepsilon_n \to 0$, the lhs converges to 
\begin{eqnarray*}
 \left<g, \int \theta(\cdot,\tilde{x}) \dd\pp{P}_{X_1}(\tilde{\x})  \right>_{L_2(\pp{P}_{X_0})} 
&=& \iint \theta(\x,\tilde{\x})\dd\pp{P}_{X_1}(\tilde{\x}) g(\x) \dd\pp{P}_{X_0}(\x) \\
&=& \iint \theta(\x,\tilde{\x}) \dd\pp{P}_{X_1}(\x) \dd\pp{P}_{X_1}(\tilde{\x}). 
\end{eqnarray*}
Thus we have shown (\ref{eq:second_goal}).
The proof completes by substituting (\ref{eq:first_goal}) and (\ref{eq:second_goal}) in  \eqref{eq:first-term-consistency} and \eqref{eq:second-term-consistency} respectively.
\end{proof}

\subsection{Proof of Theorem \ref{theo:convergence-rate}}
\label{sec:proof-convergence-rate}

\begin{proof}
By the triangle inequality we can bound the error of our estimator as
\begin{eqnarray}
\lefteqn{\| \ecovyx (\ecovx + \varepsilon_n I  )^{-1} \hat{\mu}_{X_1} - \mu_{Y\langle 0|1 \rangle} \|_{\hbspf}} \nonumber \\
&\leq& \| \ecovyx (\ecovx + \varepsilon_n I  )^{-1} \hat{\mu}_{X_1} -  \covyx (\covx + \varepsilon_n I)^{-1} \mu_{X_1} \|_{\hbspf} \label{eq:rate-estimation_error} \\
&& +  \| \covyx (\covx + \varepsilon_n I)^{-1} \mu_{X_1} - \mu_{Y\langle 0|1 \rangle} \|_{\hbspf} \label{eq:rate-approx_error},
\end{eqnarray}
where \eqref{eq:rate-estimation_error} is the estimation error, and \eqref{eq:rate-approx_error} is the approximation error.
By Theorem \ref{theo:estimation-error-rate}, the estimation error decays at the rate 
\begin{eqnarray}
\| \ecovyx (\ecovx + \varepsilon_n I  )^{-1} \hat{\mu}_{X_1} -  \covyx (\covx + \varepsilon_n I)^{-1} \mu_{X_1} \|_{\hbspf} 
 =  O_p \left( n^{-1/2} \varepsilon_n^{ \min( -1 + \alpha, -1/2 )} \right) \label{eq:rate-estimation-error-in-proof-for-rate}
\end{eqnarray}
as $n \to \infty$.
Hence we focus below on deriving a convergence rate for the approximation error. 
We then determine the optimal schedule for the decay of the regularization constant $\varepsilon_n$ as $n \to \infty$ in order to derive a convergence rate for the overall error.

\paragraph{Rate for the approximation error \eqref{eq:rate-approx_error}.}
We will show that the approximation error decays at the rate
\begin{equation} \label{eq:resulting-approx-err-rate}
\| \covyx (\covx + \varepsilon_n I)^{-1} \mu_{X_1} - \mu_{Y\langle 0|1 \rangle} \|_{\hbspf} = O\left(\varepsilon_n^{(\alpha + \beta) / 2}\right) \quad (\varepsilon_n \to 0).
\end{equation}
First note that, by the definition of $g = \dd\pp{P}_{X_1} / \dd\pp{P}_{X_0}$, we have
\begin{eqnarray} \label{eq:equiv-theta-g}
\iint \theta (\x,\tilde{\x}) \dd\pp{P}_{X_1}(\x) \dd\pp{P}_{X_1}(\tilde{\x}) &=& \iint \theta (\x,\tilde{\x}) g(\x)g(\tilde{\x}) \dd\pp{P}_{X_0}(\x) \dd\pp{P}_{X_0}(\tilde{\x}) \nonumber \\ 
&=&  \left< g \otimes g, \theta \right>_{L_2 (\pp{P}_{X_0}  \otimes \pp{P}_{X_0}) }.
\end{eqnarray}
Therefore, using Lemma \ref{lemma:approximation-key} and the notation $g_{\varepsilon_n} := (T + \varepsilon_n I)^{-1} T g$, we can bound the square of the approximation error \eqref{eq:rate-approx_error} as
\begin{eqnarray}
\lefteqn{\| \covyx (\covx + \varepsilon_n I)^{-1} \mu_{X_1} - \mu_{Y\langle 0|1 \rangle} \|_{\hbspf}^2}  \nonumber  \\
&=&   \left< g_{\varepsilon_n} \otimes g_{\varepsilon_n}, \theta \right>_{L_2(\pp{P}_{X_0} \otimes \pp{P}_{X_0})}  - 2 \left<g, (T + \varepsilon_n I)^{-1} T \int \theta (\cdot, \tilde{\x}) \dd\pp{P}_{X_1}(\tilde{\x})  \right>_{L_2(\pp{P}_{X_0})} \nonumber \\
&& + \iint \theta (\x,\tilde{\x}) \dd\pp{P}_{X_1}(\x) \dd\pp{P}_{X_1}(\tilde{\x}) \nonumber \\
&\leq& \left|  \left< g_{\varepsilon_n} \otimes g_{\varepsilon_n} , \theta \right>_{L_2(\pp{P}_{X_0} \otimes \pp{P}_{X_0})} - \left< g \otimes g, \theta \right>_{L_2(\pp{P}_{X_0} \otimes \pp{P}_{X_0})}  \right| \label{eq:approx-rate-bound}\\
&& + 2 \left|  \left< g \otimes g, \theta \right>_{L_2(\pp{P}_{X_0} \otimes \pp{P}_{X_0})} - \left<g, (T + \varepsilon_n I)^{-1} T \int \theta (\cdot, \tilde{\x}) \dd\pp{P}_{X_1}(\tilde{\x})   \right>_{L_2(\pp{P}_{X_0})} \right|, \nonumber
\end{eqnarray} 


\paragraph{Bound on the first term in \eqref{eq:approx-rate-bound}.}
By Corollary \ref{cor:range-as-eigen-basis} (which follows from Assumption \ref{as:range-assumption-theta}), $\theta = \sum_{i,j \in I} a_{i,j}  (\mu_i^\beta [e_i]_\sim) \otimes (\mu_j^\beta  [e_j]_\sim)$ with $\sum_{i,j \in I} a_{ij}^2 < \infty$.
Using this, we have
\begin{eqnarray*}
\lefteqn{\left< g_{\varepsilon_n} \otimes g_{\varepsilon_n}, \theta \right>_{L_2(\pp{P}_{X_0} \otimes \pp{P}_{X_0})} - \left< g \otimes g, \theta \right>_{L_2(\pp{P}_{X_0} \otimes \pp{P}_{X_0})}} \\
&=& \left< g_{\varepsilon_n} \otimes g_{\varepsilon_n},  \sum_{i,j \in I} a_{i,j}  (\mu_i^\beta [e_i]_\sim) \otimes (\mu_j^\beta  [e_j]_\sim) \right>_{L_2(\pp{P}_{X_0} \otimes \pp{P}_{X_0})}  \\
&&  - \left< g \otimes g,  \sum_{i,j \in I} a_{i,j}  (\mu_i^\beta [e_i]_\sim) \otimes (\mu_j^\beta  [e_j]_\sim) \right>_{L_2(\pp{P}_{X_0} \otimes \pp{P}_{X_0})} \\
&=& \sum_{i,j \in I} a_{i,j} \left< g_{\varepsilon_n}, \mu_i^\beta [e_i]_\sim  \right>_{ L_2(\pp{P}_{X_0})} \left< g_{\varepsilon_n}, \mu_j^\beta  [e_j]_\sim \right>_{L_2(\pp{P}_{X_0})} \\
&& -  \sum_{i,j \in I} a_{i,j} \left< g, \mu_i^\beta [e_i]_\sim  \right>_{ L_2(\pp{P}_{X_0})} \left< g, \mu_j^\beta  [e_j]_\sim \right>_{L_2(\pp{P}_{X_0})} \\
&=& \sum_{i,j \in I} a_{i,j} \left< g_{\varepsilon_n} - g, \mu_i^\beta [e_i]_\sim  \right>_{ L_2(\pp{P}_{X_0})} \left< g_{\varepsilon_n}, \mu_j^\beta  [e_j]_\sim \right>_{L_2(\pp{P}_{X_0})} \\
&& +  \sum_{i,j \in I} a_{i,j} \left< g, \mu_i^\beta [e_i]_\sim  \right>_{ L_2(\pp{P}_{X_0})} \left< g_{\varepsilon_n} - g, \mu_j^\beta  [e_j]_\sim \right>_{L_2(\pp{P}_{X_0})} .
\end{eqnarray*}
Therefore,
\begin{eqnarray}
\lefteqn{\left| \left< g_{\varepsilon_n} \otimes g_{\varepsilon_n}, \theta \right>_{L_2(\pp{P}_{X_0} \otimes \pp{P}_{X_0})} - \left< g \otimes g, \theta \right>_{L_2(\pp{P}_{X_0} \otimes \pp{P}_{X_0})}  \right|} \nonumber \\
&\leq& \left|\sum_{i,j \in I} a_{i,j} \left< g_{\varepsilon_n} - g, \mu_i^\beta [e_i]_\sim  \right>_{ L_2(\pp{P}_{X_0})} \left< g_{\varepsilon_n}, \mu_j^\beta  [e_j]_\sim \right>_{L_2(\pp{P}_{X_0})} \right| \nonumber \\
&& + \left| \sum_{i,j \in I} a_{i,j} \left< g, \mu_i^\beta [e_i]_\sim  \right>_{ L_2(\pp{P}_{X_0})} \left< g_{\varepsilon_n} - g, \mu_j^\beta  [e_j]_\sim \right>_{L_2(\pp{P}_{X_0})} \right| \nonumber \\
&\leq& \sqrt{\sum_{i,j \in I} a_{i,j}^2} \sqrt{ \sum_{i \in I} \left< g_{\varepsilon_n} - g, \mu_i^\beta [e_i]_\sim  \right>_{ L_2(\pp{P}_{X_0})}^2 \sum_{j \in I} \left< g_{\varepsilon_n}, \mu_j^\beta  [e_j]_\sim \right>_{L_2(\pp{P}_{X_0})}^2 } \nonumber \\
&& +  \sqrt{ \sum_{i,j \in I} a_{i,j}^2 } \sqrt{ \sum_{i \in I} \left< g, \mu_i^\beta [e_i]_\sim  \right>_{ L_2(\pp{P}_{X_0})}^2 \sum_{j \in I} \left< g_{\varepsilon_n} - g, \mu_j^\beta  [e_j]_\sim \right>_{L_2(\pp{P}_{X_0})}^2 } \nonumber \\
&=& \sqrt{\sum_{i,j \in I} a_{i,j}^2}   \left\| T^\beta (g_{\varepsilon_n} - g)  \right\|_{L_2(\pp{P}_{X_0})}  \left\| T^\beta g_{\varepsilon_n}  \right\|_{L_2(\pp{P}_{X_0})}  \nonumber  \\
&& +  \sqrt{ \sum_{i,j \in I} a_{i,j}^2 }  \left\| T^\beta g \right\|_{L_2(\pp{P}_{X_0})}   \left\| T^\beta (g_{\varepsilon_n} - g)  \right\|_{L_2(\pp{P}_{X_0})} \nonumber \\
&=&  \sqrt{\sum_{i,j \in I} a_{i,j}^2} \left\| T^\beta (g_{\varepsilon_n} - g)  \right\|_{L_2(\pp{P}_{X_0})}  \left(  \left\| T^\beta g_{\varepsilon_n} \right\|_{L_2(\pp{P}_{X_0})} +  \left\| T^\beta g \right\|_{L_2(\pp{P}_{X_0})} \right) \label{eq:bound-gg}
\end{eqnarray}

Note that $T^\beta g \in {\rm Range}(T^{\alpha + \beta})$ holds because of the assumption $g \in {\rm Range}(T^\alpha)$.
Therefore by Lemma \ref{lemma:range_bound} (which can be used because $\alpha + \beta \leq 1$), we have 
\begin{eqnarray}
 \left\| T^\beta g_{\varepsilon_n} - T^\beta g \right\|_{L_2(\pp{P}_{X_0})} 
 &=&  \left\| (T + \varepsilon_n I)^{-1}T T^\beta g - T^\beta g \right\|_{L_2(\pp{P}_{X_0})} \leq c_{\alpha + \beta} \varepsilon_n^{\alpha + \beta}, \label{eq:T-beta-bound}
\end{eqnarray}
where $c_{\alpha + \beta}$ is a constant depending only on $\alpha$, $\beta$ and $g$.
We also have
\begin{eqnarray*}
\left\| T^\beta g_{\varepsilon_n} \right\|_{L_2(\pp{P}_{X_0})}
&\leq& \left\| T^\beta g_{\varepsilon_n} -  T^\beta g  \right\|_{L_2(\pp{P}_{X_0})} + \left\|  T^\beta g  \right\|_{L_2(\pp{P}_{X_0})} \\
&\leq& c_{\alpha + \beta} \varepsilon_n^{\alpha + \beta} + \left\|  T^\beta g  \right\|_{L_2(\pp{P}_{X_0})} \quad (\because \eqref{eq:T-beta-bound})
\end{eqnarray*}

Therefore \eqref{eq:bound-gg}, and thus the first term in \eqref{eq:approx-rate-bound}, is bounded by
\begin{equation}
 \sqrt{\sum_{i,j \in I} a_{i,j}^2} c_{\alpha+\beta} \ve_n^{\alpha+\beta}  \left( c_{\alpha+\beta} \ve_n^{\alpha+\beta} + 2 \| T^\beta g \|_{L_2(\pp{P}_{X_0})}  \right)  . \label{eq:approx-rate-bound-first}
\end{equation}

\paragraph{Bound on the second term in \eqref{eq:approx-rate-bound}.}
From the equivalence \eqref{eq:equiv-theta-g}, (the half of) the second term in \eqref{eq:approx-rate-bound} can be written as
\begin{eqnarray}
&&  \left|  \left< g, \int \theta(\cdot, \tilde{\x}) \dd\Po(\tilde{\x})  \right>_{L_2(\Pz)} - \left<g, (T + \varepsilon_n I)^{-1} T  \int \theta(\cdot, \tilde{\x}) \dd\Po(\tilde{\x})   \right>_{L_2(\pp{P}_{X_0})} \right|. \label{eq:approx-rate-bound-second} 
\end{eqnarray}
Note that, using 
 $\theta = \sum_{i,j \in I} a_{i,j}  (\mu_i^\beta [e_i]_\sim) \otimes (\mu_j^\beta  [e_j]_\sim)$, we can write
\begin{eqnarray*}
\int \theta(\cdot, \tilde{\x}) \dd\Po(\tilde{\x}) &=& \int  \sum_{i,j \in I} a_{i,j}  (\mu_i^\beta [e_i]_\sim) \otimes (\mu_j^\beta  [e_j]_\sim (\tilde{\x})) \dd\pp{P}_{X_1}(\tilde{\x}) \\
&=&  \sum_{i,j\in I} a_{i,j} (\mu_i^\beta [e_i]_\sim) \int  (\mu_j^\beta  [e_j]_\sim (\tilde{\x})) \dd\pp{P}_{X_1}(\tilde{\x}).
\end{eqnarray*}
Therefore,
\begin{eqnarray*}
\left< g, \int \theta(\cdot, \tilde{\x}) \dd\Po(\tilde{\x})  \right>_{L_2(\Pz)} 
&=&  \left< g,  \sum_{i,j \in I} a_{i,j} (\mu_i^\beta [e_i]_\sim) \int (\mu_j^\beta  [e_j]_\sim (\tilde{\x})) \dd\pp{P}_{X_1}(\tilde{\x})   \right>_{L_2(\Pz)} \\
&=&    \sum_{i,j\in I} a_{i,j} \mu_i^\beta \left< g,  [e_i]_\sim \right>_{L_2(\Pz)} \int (\mu_j^\beta  [e_j]_\sim (\tilde{\x})) \dd\pp{P}_{X_1}(\tilde{\x}).
\end{eqnarray*}
Similarly,
\begin{eqnarray*}
\lefteqn{\left< g, (T + \varepsilon_n I)^{-1}T \int \theta(\cdot, \tilde{\x}) \dd\Po(\tilde{\x})  \right>_{L_2(\Pz)}}\\ 
&=& \left< (T + \varepsilon_n I)^{-1}T  g,   \sum_{i,j\in I} a_{i,j} (\mu_i^\beta [e_i]_\sim) \int (\mu_j^\beta  [e_j]_\sim (\tilde{\x})) \dd\pp{P}_{X_1}(\tilde{\x})  \right>_{L_2(\Pz)} \\
&=& \sum_{i,j\in I} a_{i,j}  \mu_i^\beta \left< (T + \varepsilon_n I)^{-1}T  g,   [e_i]_\sim \right>_{L_2(\Pz)} \int (\mu_j^\beta  [e_j]_\sim (\tilde{\x})) \dd\pp{P}_{X_1}(\tilde{\x}) 
\end{eqnarray*}
Because of the properties that $\mu_1 \geq \mu_2 \geq \cdots > 0$, that $([e_j]_\sim)_{j \in I}$ is an ONS in $L_2(\pp{P}_{X_0})$ and that $g \in L_2(\pp{P}_{X_0})$, we have
\begin{eqnarray*}
&& \sum_{j \in I} \left( \int (\mu_j^\beta  [e_j]_\sim (\tilde{\x})) \dd\pp{P}_{X_1}(\tilde{\x}) \right)^2 
\leq \mu_1^{2\beta} \sum_{j \in I}  \left( \int [e_j]_\sim (\tilde{\x}) \dd\pp{P}_{X_1}(\tilde{\x}) \right)^2  \\
&=& \mu_1^{2\beta} \sum_{j \in I}  \left( \int [e_j]_\sim (\tilde{\x}) g(\tilde{\x}) \dd\pp{P}_{X_0}(\tilde{\x}) \right)^2 
= \mu_1^{2\beta} \sum_{j \in I}  \left< [e_j]_\sim,  g \right>_{L_2(\pp{P}_{X_0})}^2 \leq \mu_1^{2\beta} \| g \|_{L_2(\pp{P}_{X_0})}^2 < \infty.
\end{eqnarray*}
Therefore, using the Cauchy-Schwartz, the above identities and inequality, we have 
\begin{eqnarray*}
&& \eqref{eq:approx-rate-bound-second} 
 = \left| \sum_{i,j\in I} a_{i,j}  \mu_i^\beta \left< g - (T + \varepsilon_n I)^{-1}T  g,   [e_i]_\sim \right>_{L_2(\Pz)} \int (\mu_j^\beta  [e_j]_\sim (\tilde{\x})) \dd\pp{P}_{X_1}(\tilde{\x})   \right| \\
 &\leq& \sqrt{ \sum_{i,j\in I} a_{i,j}^2 } \sqrt{ \sum_{i \in I}  \mu_i^{2\beta} \left< g - (T + \varepsilon_n I)^{-1}T  g,   [e_i]_\sim \right>_{L_2(\Pz)}^2 \sum_{j \in I} \left( \int (\mu_j^\beta  [e_j]_\sim (\tilde{\x})) \dd\pp{P}_{X_1}(\tilde{\x}) \right)^2} \\
 &\leq& \sqrt{ \sum_{i,j\in I} a_{i,j}^2 } \mu_1^{\beta} \| g \|_{L_2(\pp{P}_{X_0})} \sqrt{ \sum_{i \in I}  \mu_i^{2\beta} \left< g - (T + \varepsilon_n I)^{-1}T  g,   [e_i]_\sim \right>_{L_2(\Pz)}^2} \\
 &\leq& \sqrt{ \sum_{i,j\in I} a_{i,j}^2 } \mu_1^{\beta} \| g \|_{L_2(\pp{P}_{X_0})} \left\| T^{\beta} \left(g - (T + \varepsilon_n I)^{-1}T  g\right) \right\|_{L_2(\pp{P}_{X_0})},
\end{eqnarray*}
where the last inequality follows from $([e_j]_\sim)_{j \in I}$ being an ONS in $L_2(\pp{P}_{X_0})$ and the definition of $T^\beta$.
Note that we have $T^\beta g \in {\rm Range}(T^{\alpha + \beta})$ from the assumption $g \in {\rm Range}(T^\alpha)$.
Therefore by Lemma \ref{lemma:range_bound},
\begin{eqnarray*}
\left\| T^{\beta} \left(g - (T + \varepsilon_n I)^{-1}T  g\right) \right\|_{L_2(\pp{P}_{X_0})}  = \| T^\beta g - (T  + \varepsilon_n)^{-1} T T^\beta g \|_{L_2(\Pz)} \leq c_{\alpha+\beta}\ \varepsilon_n^{\alpha + \beta}   
\end{eqnarray*}
where $c_{\alpha + \beta} > 0$ is a constant depending only on $\alpha$, $\beta$ and $g$.
Thus, we finally obtain
\begin{equation} \label{eq:approx-rate-bound-second-result}
    \eqref{eq:approx-rate-bound-second}  \leq  \sqrt{ \sum_{i,j\in I} a_{i,j}^2 } \mu_1^{\beta} \| g \|_{L_2(\pp{P}_{X_0})} c_{\alpha+\beta}\ \varepsilon_n^{\alpha + \beta} . 
\end{equation}

\paragraph{Resulting approximation error rate.}
Using \eqref{eq:approx-rate-bound-first} and \eqref{eq:approx-rate-bound-second-result} in \eqref{eq:approx-rate-bound}, the rate \eqref{eq:resulting-approx-err-rate} is finally obtained as
\begin{eqnarray*}
\lefteqn{\| \covyx (\covx + \varepsilon_n I)^{-1} \mu_{X_1} - \mu_{Y\langle 0|1 \rangle} \|_{\hbspf}^2} \\
&\leq&  \sqrt{\sum_{i,j \in I} a_{i,j}^2} c_{\alpha+\beta} \varepsilon_n^{\alpha + \beta} \left(   c_{\alpha+\beta} \ve_n^{\alpha+\beta} + 2 \| T^\beta g \|_{L_2(\pp{P}_{X_0})}  +  \mu_1^{\beta} \| g \|_{L_2(\pp{P}_{X_0})}  \right) \\ 
&=& O(\varepsilon_n^{\alpha+\beta}) \quad (\varepsilon_n \to 0) .
\end{eqnarray*}

\paragraph{Balancing the estimation and approximation error rates.}
For an arbitrary constant $c > 0$ independent of $n$, let $\varepsilon_n = c n^{-b}$ for some constant $b > 0$. We determine $b$ by balancing the two rates \eqref{eq:rate-estimation-error-in-proof-for-rate} and \eqref{eq:resulting-approx-err-rate}.
This yields $b = 1 / (2 - \alpha + \beta)$ for $\alpha \leq 1/2$, and  $b = 1 / (1 + \alpha + \beta)$ for $\alpha \geq 1/2$; equivalently, $b = 1 / (1 + \beta + \max( 1 - \alpha,  \alpha))$ for $0 \leq \alpha \leq 1$.
The proof completes by substituting the resulting $\varepsilon_n = n^{-b}$ in \eqref{eq:rate-estimation-error-in-proof-for-rate} and \eqref{eq:resulting-approx-err-rate}.
\end{proof}

\bibliography{cme}

\end{document}